\documentclass[12pt]{article}
\usepackage[usenames,dvipsnames]{color}

\usepackage{graphicx}
\usepackage{amsthm}
\usepackage{graphics}
\usepackage{amsmath,,amssymb}
\usepackage{algorithmic}
\usepackage{algorithm}
\usepackage{caption}
\usepackage{subcaption}
\usepackage{pstricks, pst-node}
\usepackage{tikz}
\usepackage[linewidth=1pt]{mdframed}
\usepackage{multirow}
\usepackage[round,colon,authoryear]{natbib}
\usepackage{pgfplots}
\usepackage{hyperref}
\usepackage{bm}
\usepackage{enumitem}
\usepackage[title]{appendix}

\pgfplotsset{compat=1.10}
\usepgfplotslibrary{fillbetween}
\usetikzlibrary{patterns}
\usetikzlibrary{shapes,arrows}
\usetikzlibrary{arrows.meta}

\newtheorem{theorem}{Theorem}

\newtheorem{lemma}{Lemma}

\newtheorem{assumption}{Assumption}

\def\EE{\mathbb{E}}
\def\PP{\mathbb{P}}
\def\RR{\mathbb{R}}
\def\Rcal{\mathcal R}

\def\Bcal{\mathcal B}
\def\Acal{\mathcal A}
\def\Lcal{\mathcal L}

\def\Fcal{\mathcal F}
\def\Mcal{\mathcal M}
\def\Pcal{\mathcal P}

\def\Scal{\mathcal S}
\def\Ncal{\mathcal N}
\def\Tcal{\mathcal T}

\def\Hcal{\mathcal H}

\def\bI{\mathbf I}

\def\bX{\boldsymbol X}

\tikzstyle{block} = [rectangle, draw, fill=white!80!black, line width=2pt,
    text width=15em, text centered, rounded corners, minimum height=3em]
\tikzstyle{line} = [draw, -latex',line width=2pt]

\begin{document}
\title{Augmentation Invariant Manifold Learning
}
\author{Shulei Wang\\ University of Illinois at Urbana-Champaign}
\date{(\today)}

\maketitle

\footnotetext[1]{Address for Correspondence: Department of Statistics, University of Illinois at Urbana-Champaign, 605 E. Springfield Ave., Champaign, IL 61820 (Email: shuleiw@illinois.edu).}

\begin{abstract}
	Data augmentation is a widely used technique and an essential ingredient in the recent advance in self-supervised representation learning. By preserving the similarity between augmented data, the resulting data representation can improve various downstream analyses and achieve state-of-the-art performance in many applications. Despite the empirical effectiveness, most existing methods lack theoretical understanding under a general nonlinear setting. To fill this gap, we develop a statistical framework on a low-dimension product manifold to model the data augmentation transformation. Under this framework, we introduce a new representation learning method called augmentation invariant manifold learning and design a computationally efficient algorithm by reformulating it as a stochastic optimization problem. Compared with existing self-supervised methods, the new method simultaneously exploits the manifold's geometric structure and invariant property of augmented data and has an explicit theoretical guarantee. Our theoretical investigation characterizes the role of data augmentation in the proposed method and reveals why and how the data representation learned from augmented data can improve the $k$-nearest neighbor classifier in the downstream analysis, showing that a more complex data augmentation leads to more improvement in downstream analysis. Finally, numerical experiments on simulated and real data sets are presented to demonstrate the merit of the proposed method.
\end{abstract}



\newpage
\section{Introduction}
\label{sc:intro}

\subsection{Data Augmentation in Representation Learning}

Selecting low dimensional data features/representations is one of the most crucial components in various statistical and machine learning tasks, such as data visualization, clustering, and classifications. In these tasks, the performance of different statistical and machine learning methods usually relies largely on the choices of data representations \citep{guyon2008feature,bengio2013representation}. Although domain knowledge sometimes provides useful options to transform the raw data into features, it is still unclear which features can help the statistical methods achieve the best performance on a given data set. One promising solution to address this problem is to learn the explanatory data features/representations from the data itself. Numerous statistical and machine learning techniques are proposed to learn data-driven representations in the literature, including principal component analysis (PCA) \citep{jolliffe2002principal}, manifold learning \citep{belkin2003laplacian,coifman2006diffusion}, and autoencoders \citep{hinton1993autoencoders}. See a comprehensive review in \cite{bengio2013representation}. The idea of learning data-driven representations has been widely used and successful in different applications, including genomics \citep{rhee2018hybrid,chereda2019utilizing}, natural language processing \citep{collobert2011natural,devlin2019bert}, biomedical imaging analysis \citep{chung2017learning,kim2020deep}, and computer vision \citep{caron2020unsupervised,jing2020self}.

In recent years, self-supervised representation learning, a popular and successful method, has been introduced to learn low-dimensional representations from unlabeled data \citep{hjelm2018learning,chen2020simple,he2020momentum,grill2020bootstrap,tian2020contrastive,chen2021exploring,zbontar2021barlow}. Unlike unsupervised techniques, the data representations are trained by the pseudo labels automatically generated from the unlabeled data set. Data augmentation is perhaps one of the most commonly-used ways to generate pseudo labels in self-supervised learning, such as image flipping, rotation, colorization, and cropping \citep{gidaris2018unsupervised,shorten2019survey}. With augmented data, the self-supervised learning methods aim to preserve similarity between augmented data of the same sample. The resulting data representations can improve existing statistical and machine learning methods to achieve state-of-art performance in many applications \citep{chen2020simple,grill2020bootstrap,tian2020contrastive,zbontar2021barlow}.

Despite the exciting empirical performance, there is still little understanding of how self-supervised representation learning and data augmentation works. Why can the data representations learned from the augmented data improve downstream analysis on labeled data? What structure information does self-supervised learning exploit from augmented data? Therefore, there is a clear need to develop a rigorous statistical framework and method to characterize the role of data augmentation and study how augmented data can lead to useful data representations for downstream analysis. In addition, most existing self-supervised representation learning methods mainly focus on capturing the information of augmented data alone, that is, the data representations of augmented data should be similar. A rich source of information often neglected by current self-supervised learning is the data's low dimensional structure, which has been explicitly exploited by many classical representation learning methods, like principal component analysis and manifold learning. One may wonder if we could combine the information in augmented data and low dimensional structure to learn data representations better. We show this is feasible in this paper.

Since the introduction of self-supervised learning, recent works have tried to theoretically understand why the learned data representation from augmented data can help improve the downstream analysis \citep{arora2019theoretical,tsai2020self,wei2020theoretical,tian2020understanding,tosh2021contrastive,wen2021toward,haochen2021provable,wang2022self,wen2022mechanism,balestriero2022contrastive}. Most of these theoretical works focus on the setting where the special conditional independence structure for the augmented data is assumed. For example, \cite{arora2019theoretical} assumes the augmented data of the same sample is drawn from a conditional distribution of a discrete latent variable. \cite{wang2022self} and \cite{wen2021toward} assume the linear latent factor model for the augmented data. Beyond conditional independence structure, \cite{haochen2021provable} and \cite{balestriero2022contrastive} also study the contrastive methods from a spectral graph angle. Although the existing results can provide insights into how self-supervised representation learning methods work, some assumptions, like the linear latent factor model, could be very restrictive in many applications, and there is a lack of understanding of the role of low dimensional structure in self-supervised learning and data augmentation. Motivated by these challenges, we introduce a new framework for self-supervised learning on a low-dimensional product manifold. 

\subsection{A Manifold Model for Data Augmentation}
\label{sc:model}
The high-dimensional data naturally arise in a wide range of applications, including computer vision and genetics. Despite the high dimensionality of conventional representation, there is strong empirical evidence that the data is highly concentrated in a low-dimensional manifold in these applications. For example, a recent study in \citep{pope2021intrinsic} suggests that the intrinsic dimensions of images in data set ImageNet and CIFAR-10 are less than 45 and 30. Motivated by this observation, we assume the observed data (before and after data augmentation) lie in a $d$-dimensional Riemannian manifold $\Mcal\subset \RR^D$ in this paper.

\begin{figure}[h!]
	\begin{center}
		\begin{tikzpicture}[scale=0.9]
			\draw[thick,black](0,0)node{\includegraphics[width=0.333\textwidth]{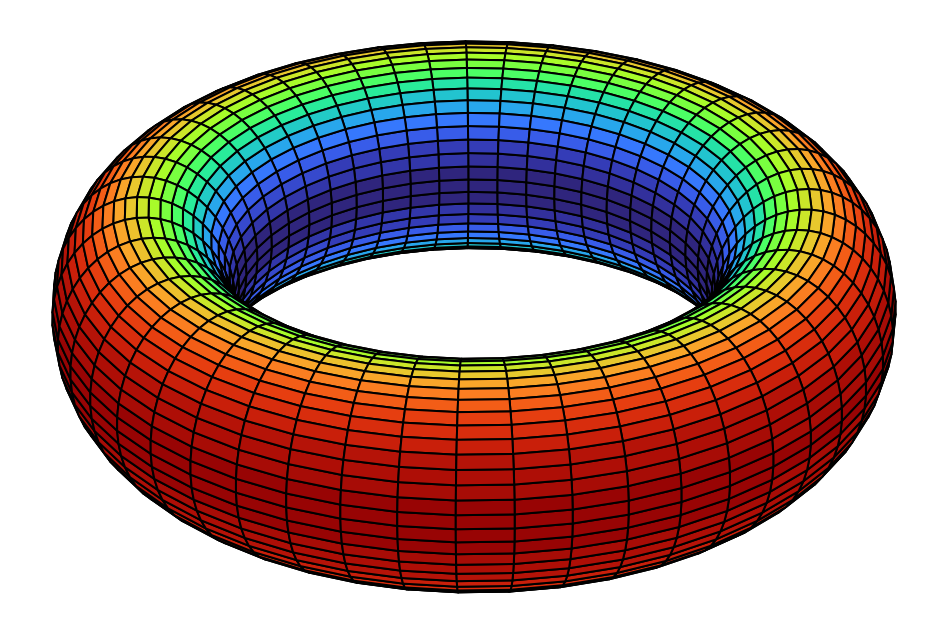}};
			\draw[latex'-latex',dashed,line width=1pt] (-2,0.55) to[out=135,in=15] (-7.5,1.65);
			\draw[latex'-latex',dashed,line width=1pt] (-2.3,0.5) to[out=135,in=5] (-4.5,1.65);
			\draw[thick,black](-8,1)node{\includegraphics[width=0.066\textwidth]{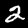}};
			\draw[thick,black](-5,1)node{\includegraphics[width=0.066\textwidth]{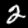}};
			
			\draw[latex'-latex',dashed,line width=1pt] (-2,-0.5) to[out=165,in=15] (-7.5,-0.35);
			\draw[latex'-latex',dashed,line width=1pt] (-2.07,-0.8) to[out=165,in=15] (-4.5,-0.35);
			\draw[thick,black](-8,-1)node{\includegraphics[width=0.066\textwidth]{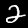}};
			\draw[thick,black](-5,-1)node{\includegraphics[width=0.066\textwidth]{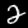}};
			
			\path[-latex,draw,line width=1pt] (-7,1) -- (-6,1);
			\path[-latex,draw,line width=1pt] (-7,-1) -- (-6,-1);
			\draw[thick,black](-6.5,1.3)node{\scriptsize Rotate};
			\draw[thick,black](-6.5,-0.7)node{\scriptsize Rotate};
		\end{tikzpicture}
	\end{center}
	\caption{An illustrative example of the product manifold: the circle with a major radius characterizes an augmentation-invariant structure, while the circle with a minor radius represents an irrelevant structure due to data augmentation (rotation). The rotation allows us to expand an observed point to a circle with a minor radius.}
	\label{fg:productmanifold}
\end{figure}

Suppose we observe a data set of $m$ samples, $\{X_1,\ldots,X_m\}\subset \Mcal$, and apply a stochastic data augmentation transformation $\Tcal$ to obtain a multi-view data set $(X_{i,1},\ldots, X_{i,n})$, $i=1,\ldots,m$, where $n$ different augmented data of each sample are obtained by applying data augmentation transformation $n$ times. The augmented data transformation includes flipping, rotation, colorization, and scaling \citep{shorten2019survey}. To model the transformation, we assume the Riemannian manifold $\Mcal$ is an isometric embedding of a product manifold
\begin{equation}
	\label{eq:prod}
	\Mcal=T(\Ncal_s\times \Ncal_v),
\end{equation}
where $T$ is an isometry and $\Ncal_s$, $\Ncal_v$ are two no-boundary Riemannian manifolds with dimensions $d_s$ and $d_v$ such that $d=d_s+d_v$. Here, $\Ncal_s$ represents the structure of interest and $\Ncal_v$ corresponds to irrelevant nuisance structures that result from data augmentation. The above product manifold offers a straightforward way to model stochastic data augmentation transformation $\Tcal$. Specifically, when $X\in \Mcal$, there exists latent variables $\phi \in \Ncal_s$ and $\psi\in \Ncal_v$ such that $X=T(\phi,\psi)$ and we can write the corresponding augmented data as $\Tcal(X)=T(\phi,\psi')$ where $\psi'$ is randomly drawn from $ \Ncal_v$. A toy example is shown in Figure~\ref{fg:productmanifold} to illustrate the model of data augmentation on the product manifold. Given the above model of data augmentation, we consider the following way to generate the multi-view augmented data, $\phi_1,\ldots,\phi_m \stackrel{i.i.d.}{\sim} f_s(\phi)$ and $\psi_{i,1},\ldots,\psi_{i,n} \stackrel{i.i.d.}{\sim} f_v(\psi|\phi_i)$, where $f_s(\phi)$ is some conditional probability density function defined on $\Ncal_s$ and $f_v(\psi|\phi_i)$ is some probability density function defined on $\Ncal_v$. Our observed augmented data is $X_{i,j}=T(\phi_i,\psi_{i,j})$ for $i=1,\ldots,m$ and $j=1,\ldots, n$. Similar models are also considered in \cite{berry2018iterated,lederman2018learning,talmon2019latent,salhov2020multi,lindenbaum2020multi}, but their sampling processes are different from our model. This model indicates that each sample's augmented data lies in a manifold $\Mcal$'s fiber, $X_{i,1},\ldots,X_{i,n}\in \Mcal(\phi_i)$, where $\Mcal(\phi)=\{x\in \Mcal:x=T(\phi,\psi),\psi\in \Ncal_v\}$. The elements within the same fiber are equivalent up to some data augmentation transformation. This model also suggests that applying data augmentation transformation (infinite times) allows us to expand our observation from a set of observed data points, $\{X_1,\ldots,X_m\}$, to a collection of fibers, $\{\Mcal(\phi_1),\ldots, \Mcal(\phi_m)\}$.

Given the data augmentation model, what are the desired data representations? Here, the data representation is defined as a map $\Theta:\Mcal \to \RR^N$.
\begin{itemize}
	\item \textbf{Augmentation invariant} Similar to other self-supervised learning methods, the desired data representation should preserve similarity between augmented data of the same sample since it is usually believed that data augmentation only perturbs irrelevant information \citep{chan2021redunet}. Putting mathematically, the map $\Theta$ is invariant to data augmentation, that is, $\Theta(x)=\Theta(\tilde{x}),$ if $x,\tilde{x}\in \Mcal(\phi)$.
	\item \textbf{Local similarity} As the data lie on a low-dimensional manifold, the desired data representations should be able to capture the intrinsic geometric structure and preserve the local information. More concretely, the map $\Theta$ should map similar samples to similar data representations, that is, $\phi_x\approx\phi_{\tilde{x}}$ if and only if $\Theta(x)\approx \Theta(\tilde{x})$ when $x\in \Mcal(\phi_x),\ \tilde{x}\in \Mcal(\phi_{\tilde{x}})$.
\end{itemize}
These two good properties make one wonder if such an ideal map exists and, if so, how we should find it based on the augmented data $(X_{i,1},\ldots, X_{i,n})$, $i=1,\ldots,m$. 

\subsection{A Peek at Augmentation Invariant Manifold Learning}

This paper’s primary goal is to introduce a new representation learning method for the augmented data, called augmentation invariant manifold learning. The new framework is simple as it shares a similar procedure with classical manifold learning methods, Laplacian eigenmaps \citep{belkin2003laplacian}, and diffusion maps \citep{coifman2006diffusion}. The main difference from these two classical manifold learning methods is that augmentation invariant manifold learning integrates the kernels between any pair of augmented data to evaluate the similarity between two samples 
$$
W_{i_1,i_2}={1\over n^2}\sum_{j_1,j_2=1}^n\exp\left(-{\|X_{i_1,j_1}-X_{i_2,j_2}\|^2\over t}\right),
$$
where $t>0$. Our investigation shows that, with this simple modification, augmentation invariant manifold learning can recover the eigenvalues and eigenfunctions of the Laplace-Beltrami operator on $\Ncal_s$ instead of $\Mcal$. Therefore, the data representations/eigenfunctions learned from augmentation invariant manifold learning 1) is invariant to data augmentation and 2) can capture the intrinsic geometric structures of the underlying manifold. 

To better illustrate the idea, we consider a simple example where $\Ncal_s=\Ncal_v=S^1$ where $S^1$ is a circle, and $\Mcal=S^1\times S^1$ is a torus in 3 dimensional space $x=(x_1,x_2,x_3)$ such that $x_1=(10+5\cos \phi)\cos \psi$, $x_2=(10+5\cos \phi)\sin \psi$, and $x_3=5\sin \phi$. We choose $m=400$ and $n=3$ in the augmented data. If we only consider the first two eigenvectors in augmentation invariant Laplacian eigenmaps as our data representation, Figure~\ref{fg:representation} shows the plots of representations colored by the value of $\phi$ and $\psi$. The data representation from augmentation invariant Laplacian eigenmaps can capture the information on $\Ncal_s$ very well and is invariant to different $\psi$'s choices. 

\begin{figure}[h]
	\centering
	\includegraphics[width=0.35\textwidth]{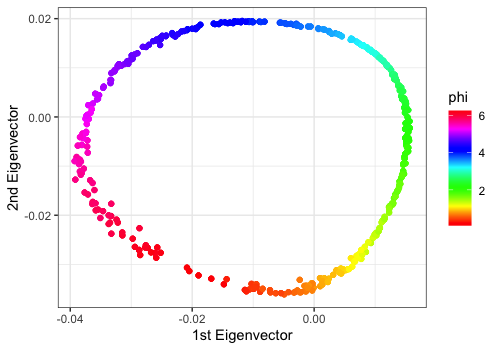}
	\hskip 20pt
	\includegraphics[width=0.35\textwidth]{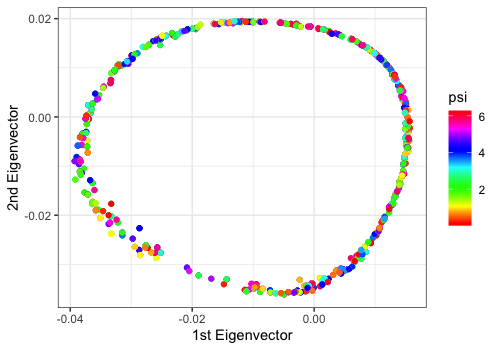}
	\caption{Plots of new data representation colored by the value of $\phi$ (left) and $\psi$ (right).}\label{fg:representation}
\end{figure}

The new data representation is invariant to data augmentation, but can it help improve downstream analyses? Our analysis shows that the new data representation can lead to more powerful $k$-nearest neighbor ($k$-NN) classifier if we further assume $\gamma(x)=\gamma(\tilde{x})$ for any  $x, \tilde{x}\in \Mcal(\phi)$, where $\gamma(x)=\PP(Y=1|X=x)$ is the regression function. Specifically, when we compare the original $X$ and new data representation $\Theta(X)$, the excess risk of misclassification error of $k$-NN can be improved from $s^{-\alpha(1+\beta)/(2\alpha+d)}$ to $s^{-\alpha(1+\beta)/(2\alpha+d_s)}$, where $s$ is the sample size in downstream analysis, $\alpha$ and $\beta$ are the smoothness of the regression function and parameter for the Tsybakov margin condition, and $d$ and $d_s$ are the dimensions of $\Mcal$ and $\Ncal_s$. The intuitive explanation behind this improvement is that points in the neighborhood defined by $\Theta(X)$ have more similar values in $\phi$ than in the neighborhood defined by $X$. In Figure~\ref{fg:neighbour}, we illustrate this intuition by comparing the neighborhood defined by different representations in previous torus example. The results here also suggest that a more complex data augmentation (larger $d_v$ and thus smaller $d_s$) can lead to more improvement in the downstream analysis, which is consistent with the empirical observation in \cite{chen2020simple} and our numerical experiments in Section~\ref{sc:numerical}. 

\begin{figure}[h]
	\centering
	\includegraphics[width=0.35\textwidth]{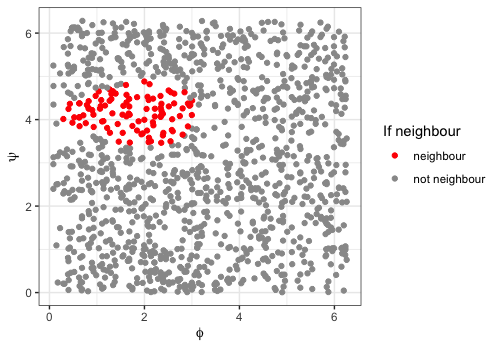}
	\hskip 20pt
	\includegraphics[width=0.35\textwidth]{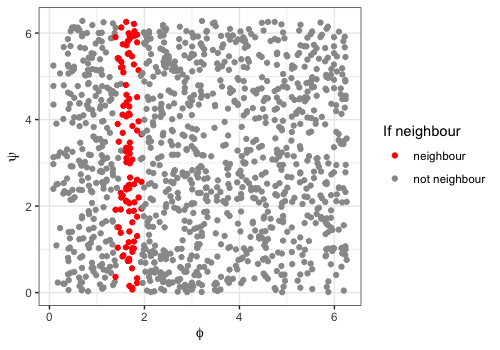}
	\caption{Neighborhood defined by $X$ (left) and new data representation $\Theta(X)$ (right).}\label{fg:neighbour}
\end{figure}

Despite the guaranteed efficiency, the above formulation of augmentation invariant manifold learning is computationally expensive on large data sets because of $O(m^2)$ computational complexity and poses practical challenges in generalizing the augmentation invariant representation to new points. To address these challenges, we parameterize the data representation and reformulate the augmentation invariant manifold learning as a stochastic optimization problem. Specifically, if we parameterize the data representation maps as $\Theta_\beta$ where $\beta\in \Bcal$, e.g., a deep neural network encoder, the stochastic optimization algorithm can minimize the following objective function on each minibatch $\Scal\subset \{1,\ldots,m\}$
$$
\hat{\ell}(\beta):=\underbrace{\sum_{i\in \Scal}W_{i,\pi(i)}\|\Theta_\beta(X'_{i})-\Theta_\beta(X''_{\pi(i)})\|^2}_{\rm unsupervised\ signal}+\underbrace{\lambda_1 \sum_{i\in \Scal}\|\Theta_\beta(X'_{i})-\Theta_\beta(X''_{i})\|^2}_{\rm self-supervised\ signal}+\underbrace{\lambda_2\Rcal(\Theta_\beta)}_{\rm regularization},
$$
where $X'_i$ and $X''_i$ are two independent augmented copies of $X_i$, $\pi$ is a random permutation of $\Scal$ such that $\pi(i)$ is matched with $i$, $W_{i,\pi(i)}$ is the weight between $X'_{i}$ and $X''_{\pi(i)}$, $\Rcal(\Theta_\beta)$ is the regularization term, and $\lambda_1,\lambda_2$ are tuning parameters. Unlike existing self-supervised representation learning methods, the new loss function keeps the local similarity between negative data pairs, which are usually defined as augmented copies of different samples in contrastive learning \citep{chen2020simple,radford2021learning}, and adds a regularization term to avoid dimensional collapse. The reformulation of augmentation invariant manifold learning allows the generalization of the representation to the new points in a constant time and improves the computational complexity to $O(m)$.

\subsection{Contributions and Comparisons with Existing Works}
In this section, we summarize the contributions of this paper and compare them with the literature.
\begin{enumerate}
	\item Pre-training the deep neural network on augmented data has been one of the most popular approaches to learning representation in computer vision \citep{pathak2017learning,gidaris2018unsupervised,he2020momentum}. Despite the popularity, most, if not all, methods lack any explicit theoretical guarantees, especially under a general nonlinear setting. One of the major difficulties is the need for a suitable model that can characterize the invariant property of augmented data in a general way and also allow convenient theoretical analysis. To overcome this challenge, we adopt a product manifold model to capture the observed data's low-dimensional and invariant structure. Unlike existing theoretical frameworks \citep{arora2019theoretical,wei2020theoretical,tian2020understanding,wang2022self,wen2022mechanism}, this general and nonlinear model can offer a natural way to model the data augmentation transformation and decompose the observed data into the data-augmentation invariant and irrelevant nuisance structure.
	\item Building on the product manifold model, we introduce a new self-supervised representation learning method called augmentation invariant manifold learning. Interestingly, the investigation shows that preserving the similarity between augmented data is equivalent to integrating the kernels between any pair of augmented data in manifold learning. Because of this observation, the augmentation invariant manifold learning method has two ways of formulations: one shares a similar procedure with the classical manifold learning method but is equipped with an integrated kernel; the other is cast as a stochastic optimization problem like existing self-supervised learning methods but keeps the local similarity of negative pairs.  Compared with existing self-supervised learning methods, the new method's manifold learning formulation allows for the exploration of the low-dimensional structure of data and the characterization of explicit asymptotic property, providing theoretical guarantees. Unlike classical manifold learning methods, the new method can capture augmentation invariant structure of observed data and has a generalizable and computationally efficient algorithm.
	\item As mentioned above, this paper establishes the connection between classical manifold learning and self-supervised learning methods. This new connection allows the borrowing of the technical tools in the manifold learning field to study the theoretical properties of self-supervised learning. Besides applying these powerful tools, we also need to address new challenges in studying the theoretical properties of augmentation invariant manifold learning. In particular, the integrated kernel between augmented data is not a determined kernel of observed data but a randomized kernel of latent variables, so we need to develop a new perturbation analysis for the randomized kernel. In addition, the representation resulting from manifold learning is mainly used for data visualization or clustering in the literature, while its performance in classification is largely unknown. By addressing these challenges, we establish the convergence of the new method and theoretically characterize how the resulting data representation can improve the $k$-NN classifier. Given the product manifold model and the newly established connection between augmentation invariant representation and kernel integration, these theoretical results are somewhat expected. However, these new results can provide new insight into the role of data augmentation in self-supervised learning, showing how augmented data can help better reduce dimensionality nonlinearly (or compress the nuisance information) and why the augmentation invariant representation can lead to the improvement in downstream analysis. 	
\end{enumerate}

\section{Augmentation Invariant Manifold Learning}
\label{sc:submanifold}

\subsection{Classical Manifold Learning}
\label{sc:manifold}

Several different nonlinear methods are proposed in the literature to extract features from the data lying on a low-dimensional manifold, including Isomap \citep{tenenbaum2000global}, locally linear embedding (LLE) \citep{roweis2000nonlinear}, maximum variance unfolding \citep{weinberger2006unsupervised}, Hessian maps \citep{donoho2003hessian}, Laplacian eigenmaps \citep{belkin2003laplacian}, local tangent space alignment \citep{zhang2004principal}, diffusion maps \citep{coifman2006diffusion}, and vector diffusion map \citep{singer2012vector}. Most manifold learning algorithms explore the locally linear structure of the underlying manifold. In particular, the spectral graph-based method, such as Laplacian eigenmaps and diffusion maps, is one of the most popular strategies, which constructs a graph by connecting close points and then using the eigenvectors of graph Laplacian as data representation. 

Before introducing the new method, we briefly review how Laplacian eigenmaps was developed. See more details in \cite{belkin2003laplacian}. Suppose $\theta:\Mcal\to \RR$ is a twice differentiable map, which could be one coordinate of data representation $\Theta$. If we aim to find a map that can preserve the local similarity of nearby points, i.e., $\theta(x)\approx \theta(\tilde{x})$ when $x\approx \tilde{x}$, we can consider the following optimization problem for Dirichlet energy
\begin{equation}
	\label{eq:dirichlet}
	\min_{\|\theta\|_{L^2(\Mcal)}=1} \int_\Mcal \|\nabla \theta(x)\|^2dx,
\end{equation}
where $dx$ is with respect to the Riemannian volume form of $\Mcal$. Here, $\|\nabla \theta(x)\|$ is used to characterize the distortion of $\theta$ in the local region. By Stokes' theorem, this is equivalent to the following optimization problem
$$
\min_{\|\theta\|_{L^2(\Mcal)}=1} \int_\Mcal \Lcal(\theta)\theta(x) dx,
$$
where $\Lcal$ is the Laplace-Beltrami operator on $\Mcal$. Therefore, the solution for the optimization problem in \eqref{eq:dirichlet} for Dirichlet energy is the eigenfunctions of the Laplace-Beltrami operator, which can be used as data representation $\Theta$. 

In practice, we only observe a collection of points $X_1,\ldots, X_m$ randomly drawn from the manifold $\Mcal$ instead of knowing the manifold $\Mcal$ exactly. We can discretize the continuous map $\theta$ as $\vec{\theta}=(\theta_1,\ldots,\theta_m)=(\theta(X_1),\ldots,\theta(X_m))$ and consider a discrete version of Dirichlet energy in \eqref{eq:dirichlet}
$$
\min_{\theta_1,\ldots,\theta_m:\sum_i\theta_i^2=1}{1\over 2}\sum_{i_1,i_2=1}^mW_{i_1,i_2}(\theta_{i_1}-\theta_{i_2})^2,
$$
where $W_{i_1,i_2}$ is the weight of the edge connecting $X_{i_1}$ and $X_{i_2}$. The introduction of weight allows to only keep similarity between $\theta(X_{i_1})$ and $\theta(X_{i_2})$ when $X_{i_1}$ and $X_{i_2}$ are close, mimicking the $\|\nabla \theta(x)\|^2$ in Dirichlet energy. See more discussions on the connection between continuous and discrete versions of Dirichlet energy in \cite{garcia2020error}. There are several different ways to choose the weights $W_{i_1,i_2}$. For example, Laplacian eigenmaps adopts the following weight
$$
W_{i_1,i_2}={h_t(\|X_{i_1}-X_{i_2}\|)\over \sqrt{\sum_{i_3=1}^mh_t(\|X_{i_1}-X_{i_3}\|)} \sqrt{\sum_{i_3=1}^mh_t(\|X_{i_2}-X_{i_3}\|)}  }
$$
where $h_t(z)=\exp(z^2/t)$. Given the weight matrix $W$, we introduce the graph Laplacian $L=D-W$, where $D$ is the degree matrix of $W$, that is, $D$ is a diagonal matrix such that $D_{i,i}=\sum_{i'=1}^m W_{i,i'}$. Since the graph Laplacian $L$ converges to the Laplace-Beltrami operator $\Lcal$ \citep{belkin2008towards}, it is sufficient to find the eigenvectors of $L$ to estimate the eigenfunctions of Laplace-Beltrami operator $\Lcal$. Thus, the eigenvectors of $L$ can form the data representation at $\{X_1,\ldots, X_m\}$.  The data representation resulting from $L$'s eigenvectors can capture the intrinsic geometric structures of $\Mcal$ but is not augmentation invariant.

\subsection{ Augmentation Invariant Manifold Learning}
\label{sc:algorithm}
Besides low-dimensional manifold assumption, we have extra information on which two data points are equivalent in the augmented data. How shall we incorporate such information to find an augmentation invariant data representation? We still consider the optimization problem for Dirichlet energy, but need to add a constraint for augmented data
\begin{equation}
	\label{eq:condiri}
	\min_{\|\theta\|_{L^2(\Mcal)}=1} \int_\Mcal \|\nabla \theta(x)\|^2dx,\qquad {\rm s.t.}\quad \theta(x)=\theta(\tilde{x}),\quad {\rm if}\quad x,\tilde{x}\in \Mcal(\phi).
\end{equation}
The constraint requires that $\theta$ has the same value on $\Mcal(\phi)$ for each $\phi$ so that we can write any $\theta(x)$ satisfying this constraint as $\theta(x)=\tilde{\theta}(\phi)$ for some function $\tilde{\theta}$ defined on $\Ncal_s$. Therefore, the above optimization problem can be reformulated as
$$
\min_{\|\tilde{\theta}\|_{L^2(\Ncal_s)}=1} \int_{\Ncal_s} \|\nabla \tilde{\theta}(\phi)\|^2d\phi.
$$
By Stokes' theorem again, the solution in the above optimization problem is the eigenfunctions of Laplace Beltrami operator $\Lcal_{\Ncal_s}$ on $\Ncal_s$, which are augmentation invariant. 

To estimate these eigenfunctions from the observed data, we can still consider the similar discrete version of Dirichlet energy in \eqref{eq:condiri}
\begin{equation}
	\label{eq:discopti}
	\min_{\sum_{i,j}\theta_{i,j}^2=1}{1\over 2}\sum_{i_1,i_2=1}^m\sum_{j_1,j_2=1}^nW_{i_1,j_1,i_2,j_2}(\theta_{i_1,j_1}-\theta_{i_2,j_2})^2,\qquad {\rm s.t.}\ \theta_{i,1}=\ldots=\theta_{i,n}, i=1,\ldots,m,
\end{equation}
where $W_{i_1,j_1,i_2,j_2}$ is the weight of the edge connecting $X_{i_1,j_1}$ and $X_{i_2,j_2}$ and $\theta_{i,j}=\theta(X_{i,j})$. If we write $\theta_i=\theta_{i,1}=\ldots=\theta_{i,n}$, the above optimization problem can be reformulated as the following equivalent problem
$$
\min_{\theta_1,\ldots,\theta_m:\sum_i\theta_i^2=1}{1\over 2}\sum_{i_1,i_2=1}^mW_{i_1,i_2}(\theta_{i_1}-\theta_{i_2})^2,\qquad {\rm where}\quad W_{i_1,i_2}=\sum_{j_1,j_2=1}^nW_{i_1,j_1,i_2,j_2}.
$$
The new optimization problem suggests that finding augmentation invariant representation is equivalent to integrating the kernels between augmented data. Compared with the classical spectral graph-based method, the main difference is that the weights between all possible augmented data of two samples are combined into a single weight. Integrating several kernels/weights is also used in solving sensor fusion problems \citep{gustafsson2010statistical,lahat2015multimodal}, which aims to extract common information from several different sensors \citep{lederman2018learning,talmon2019latent,lindenbaum2020multi}. In the sensor fusion problem, the data from different sensors are usually from different spaces and thus not comparable, so only the kernel for the data from the same sensor(view) is considered. Unlike methods designed for sensor fusion problems, we evaluate the weights on all augmented data pairs of two samples because the augmented data lie on the same manifold $\Mcal$. 

\begin{algorithm}[h!]
	\caption{Augmentation Invariant Laplacian Eigenmaps}
	\label{ag:aiml}
	\begin{algorithmic}
		\REQUIRE A set of augmented data: $(X_{i,1},\ldots,X_{i,n})$, $i=1,\ldots,m$.
		\STATE Step 1: Calculate the weights between samples 
		$$
		W_{i_1,i_2}={1\over n^2}\sum_{j_1,j_2=1}^n\exp\left(-{\|X_{i_1,j_1}-X_{i_2,j_2}\|^2\over t}\right),\qquad i_1,i_2=1,\ldots,m.
		$$
		\STATE  Step 2: Find the Laplacian matrix $L=D-W$ where $D$ is a diagonal degree matrix of $W$, i.e., $D_{i,i}=\sum_jW_{i,j}$. 
		\STATE  Step 3: Find the first $N$ eigenvectors $\vec{\eta}_1,\ldots,\vec{\eta}_N$ for the generalized eigenvector problem
		$$
		L\vec{\eta}=\lambda D\vec{\eta}.
		$$
		\ENSURE The representation for each sample: $(X_{i,1},\ldots,X_{i,n}) \to (\eta_{1,i},\ldots, \eta_{N,i})$
	\end{algorithmic}
\end{algorithm}

To find the data representation, we can apply the idea of integrating weights to most spectral graph-based methods, such as Laplacian eigenmaps and diffusion maps. For example, Algorithm~\ref{ag:aiml} is the augmentation invariant Laplacian eigenmaps method, combining the weights on all augmented data pairs of two samples. Algorithm S1 in Supplement Material is the augmentation invariant version of diffusion maps. It is worth noting that the data representations returned by Algorithms~\ref{ag:aiml} are only at $(X_{i,1},\ldots,X_{i,n})$, $i=1,\ldots,m$, and computational complexity is $O(m^2)$. In Section~\ref{sc:representation}, we will develop a more computationally efficient and generalizable version of Algorithm~\ref{ag:aiml}.

\subsection{Convergence Analysis}
\label{sc:theory}
In this section, we show that the empirical data representation in Algorithm~\ref{ag:aiml} converges to the eigenfunctions of Laplace Beltrami operator $\Lcal_\phi$ on $\Ncal_s$. To study the theoretical properties, we must define the point cloud operators in Algorithm~\ref{ag:aiml} carefully. Specifically, given $\phi\in \Ncal_s$, we write
$$
W_n(\phi_i,\phi_{i'})={1\over n^2}\sum_{j,j'=1}^n\exp\left(-{\|X_{i,j}-X_{i',j'}\|^2\over t}\right)
$$
and
$$
W_n(\phi_i,\phi)={1\over n}\sum_{j=1}^n\int_{\Mcal(\phi)}\exp\left(-{\|X_{i,j}-x\|^2\over t}\right)f_v(x|\phi)dx.
$$
The  point cloud operator in Algorithm~\ref{ag:aiml} is defined as
$$
L_{m,n}^tg(\phi)={1\over m}\sum_{i=1}^m {1\over t} {W_n(\phi_i,\phi)\over \sqrt{D_{m,n}^t(\phi)}\sqrt{D_{m,n}^t(\phi_i)}}\left(g(\phi)-g(\phi_i)\right),
$$
where $g(\phi)$ is a function defined on $\Ncal_s$, and $D_{m,n}^t(\phi)$ and $D_{m,n}^t(\phi_i)$ are 
$$
D_{m,n}^t(\phi)={1\over m}\sum_{i=1}^m W_n(\phi_i,\phi)\qquad {\rm and}\qquad D_{m,n}^t(\phi_i)={1\over m}\sum_{i'=1}^m W_n(\phi_i,\phi_{i'}).
$$
The following theorem shows that the point cloud operator converges to the weighted Laplace Beltrami operator on $\Ncal_s$ so that the data representation in Algorithm~\ref{ag:aiml} are eigenvectors of the Laplace Beltrami operator. 

\begin{theorem}
	\label{thm:aiml}
	Suppose that $f_v(\psi|\phi)$ is uniform distribution on $\Mcal(\phi)$ for any $\phi$, there exists a constant $\kappa$ such that $1/\kappa<f_s(\phi)<\kappa$, and $f_s(\phi)$ is twice differentiable. If we choose $t=m^{-1/(d+4)}$, then
	$$
	\lim_{m\to \infty}L_{m,n}^tg(\phi)={1\over 2}\Lcal_{\Ncal_s,f_s}g(\phi),
	$$
	where $\Lcal_{\Ncal_s,f_s}$ is weighted Laplace Beltrami operator $\Lcal_{\Ncal_s,f_s}g(\phi)=f_s^{-1}(\phi){\rm div}(f_s(\phi)\nabla_{\Ncal_s}g(\phi))$ and the limit is taken in probability.
	In particular, if $f_s$ is uniform distribution on $\Ncal_s$, we have
	$$
	\lim_{m\to \infty}L_{m,n}^tg(\phi)={1\over 2}\Lcal_{\Ncal_s}g(\phi),
	$$
	where $\Lcal_{\Ncal_s}$ is Laplace Beltrami operator on $\Ncal_s$.
\end{theorem}
This theorem suggests that the operator in Algorithm~\ref{ag:aiml} converges to the weighted Laplace Beltrami operator on $\Ncal_s$. Instead of $\Mcal$, the eigenvectors are defined on $\Ncal_s$ so that they are augmentation invariant. The intuition behind the results is that the weight in Algorithm~\ref{ag:aiml} uses a randomized kernel defined on $\Ncal_s$, and the expectation of the randomized kernel is  
$$
\EE(W_{i_1,i_2})={1\over {\rm Vol}^2\Ncal_v}\int_{\Mcal(\phi_{i_1})}\int_{\Mcal(\phi_{i_2})}\exp\left(\|x-y\|^2\over t\right)dxdy,
$$
while the classical manifold learning uses a deterministic kernel to evaluate weight. The above expectation can also be seen as a kernel between two fibers, $\Mcal(\phi_{i_1})$ and $\Mcal(\phi_{i_2})$, and can thus assess the similarity between $\phi_{i_1}$ and $\phi_{i_2}$. We can also show similar results for augmentation invariant diffusion maps in Algorithm S1 (see Theorem S1 of Supplement Material). Besides the point convergence presented in Theorem~\ref{thm:aiml}, we also provide a spectral convergence in the $L^\infty$ sense in Theorem~\ref{thm:convergencerate}. Also, see the spectral convergence in the $L^2$ sense in a recent paper \citep{wang2023linear}.

\section{Benefits for Downstream Analysis}
\label{sc:knn}
The previous section shows that Algorithms~\ref{ag:aiml} and S1 can find augmentation invariant representation when integrating information between augmented data. One may wonder whether the new data representation can help improve downstream analysis. If yes, to what extent can the new data representation improve the downstream analysis? To answer these questions, we study how the new data representation help improve one of the most popular classification methods: the $k$-nearest neighbor ($k$-NN) classifier \citep{fix1985discriminatory,altman1992introduction,biau2015lectures}. 

To be specific, suppose we observe a collection of data $(X_1,Y_1),\ldots,(X_s,Y_s)$ in the downstream task such that $X_1,\ldots, X_s \in \Mcal$ and $Y_1,\ldots,Y_s \in \{-1,1\}$. The goal is to build a classifier $h: \Mcal \to \{-1,1\}$ to predict the label $Y$ for any given input $X$. The high-level idea in $k$-NN is that the majority vote of $k$-nearest neighbor is the predicted label. Specifically, the $k$-NN classifier is defined as the following:
$$
\hat{h}_X(x)=\begin{cases}
	1, & \sum_{i=1}^k\bI(Y_{(i)}=1)>k/2\\
	-1,& {\rm otherwise}
\end{cases},
$$
where $(X_{(1)},Y_{(1)}),\ldots,(X_{(s)},Y_{(s)})$ is a permutation of $(X_1,Y_1),\ldots,(X_s,Y_s)$ such that $\|X_{(1)}-x\|\le\ldots\le \|X_{(s)}-x\|$. For simplicity, we always use Euclidean distance $\|\cdot\|$ here. We can define the $k$-NN classifier $\hat{h}_{\Theta(X)}(x)$ similarly when we adopt data representation $\Theta(X)$. Unlike $\hat{h}_X(x)$, the $k$ nearest neighbors in $\hat{h}_{\Theta(X)}(x)$ is defined by a permutation of $\Theta(X_1),\ldots, \Theta(X_s)$ rather than the original data, that is, $\|\Theta(X_{(1)})-\Theta(x)\|\le\ldots\le \|\Theta(X_{(s)})-\Theta(x)\|$. The goal of this section is to compare the performance of $\hat{h}_{\Theta(X)}(x)$ and $\hat{h}_X(x)$.

\subsection{Infinite Samples}
\label{sc:infinite}
To study the effect of new data representation, we consider the ideal case that infinite samples are observed in the augmentation invariant manifold learning stage, that is, $m,n=\infty$. In other words, we know exactly the Laplace Beltrami operator on $\Ncal_s$ and its eigenfunctions. Specifically, we consider two types of data representations returned by Algorithms~\ref{ag:aiml} and S1 in this section
$$
\Theta_1(x)=(\eta_1(\phi),\ldots,\eta_N(\phi))\qquad {\rm and}\qquad \Theta_2(x)=(e^{-l\lambda_1}\eta_1(\phi),\ldots,e^{-l\lambda_N}\eta_N(\phi)),
$$
where $x=T(\phi,\psi)$, $\eta_1(\phi),\ldots,\eta_N(\phi)$ are the first $N$ eigenfunctions of $\Lcal_{\Ncal_s}$, and $0=\lambda_0<\lambda_1\le \ldots\le \lambda_N$ are corresponding eigenvalues. $\Theta_1(x)$ and $\Theta_2(x)$ can be recovered by Algorithms~\ref{ag:aiml} and S1, respectively, and the exponential weights can make $\Theta_2(x)$ approximately preserve the metric on the manifold \citep{portegies2016embeddings}. To compare $\hat{h}_X(x)$, $\hat{h}_{\Theta_1(X)}(x)$, and $\hat{h}_{\Theta_2(X)}(x)$, we consider the excess risk of misclassification error as our performance measure, that is, $r(\hat{h})=\EE\left(\PP(Y\ne \hat{h}(X))\right)-\PP(Y\ne h^\ast(X))$, where $\hat{h}$ is a classifier estimated from the data and $h^\ast$ is the optimal Bayes classification rule. To characterize the theoretical properties of $k$-NN, we consider the following assumptions:
\begin{assumption}
	\label{ap:knn}
	It holds that
	\begin{enumerate}
		\item Let $\gamma(x)=\PP(Y=1|X=x)$ be the regression function. We assume $\gamma(x)=\gamma(\tilde{x})$ if $x, \tilde{x}\in \Mcal(\phi)$. In other words, there exists a function $\tilde{\gamma}$ on $\Ncal_s$ such that $\gamma(x)=\gamma(T(\phi,\psi))=\tilde{\gamma}(\phi)$;
		\item $\tilde{\gamma}(\phi)$ is $\alpha$-H\"older continuous in $\phi$, i.e., $|\tilde{\gamma}(\phi)-\tilde{\gamma}(\phi')|\le Ld_{\Ncal_s}(\phi,\phi')^\alpha$, where $\phi,\phi'\in \Ncal_s$;
		\item The distribution of $X$ satisfies $\beta$-marginal assumption on $\Mcal$, i.e., $\PP(0<|\gamma(X)-1/2|\le t)\le C_0t^\beta$ for some constant $C_0$;
		\item $\Mcal$ is a compact manifold. If we write the probability density function of $X$ as $f_\mu(x)$, we assume $1/\kappa\le f_\mu(x)\le \kappa$ for some $\kappa>1$;
		\item The volume of manifold $\Mcal$ is upper bounded by $V>0$, the Ricci curvature on $\Mcal$ is bounded below by $\zeta>0$, and the injectivity radius on $\Mcal$ is bounded below by $\iota>0$. See Supplementary Material for the definition of notions in differential geometry.  
	\end{enumerate}
\end{assumption}
The first assumption is the key assumption for data augmentation. It is usually believed that the data augmentation cannot change the label of the data \citep{chen2020group}, so we can assume $\gamma(x)$ has the same value on $ \Mcal(\phi)$. The next three assumptions are standard conditions used in the theoretical investigation of $k$-NN classifier \citep{audibert2007fast,samworth2012optimal,wang2022self}, but we extend them to the manifold setting here. The last assumption is used to characterize how many eigenfunctions are needed to represent the manifold \citep{bates2014embedding,portegies2016embeddings}. With these assumptions, the following theorem shows the convergence rate of  excess risk in $\hat{h}_X(x)$, $\hat{h}_{\Theta_1(X)}(x)$, and $\hat{h}_{\Theta_2(X)}(x)$. 
\begin{theorem}
	\label{thm:convergence}
	Suppose the Assumption~\ref{ap:knn} holds. There exists constants $l_0$ and $N_0$ (relying on $d_s$, $V$, $\kappa$, and $\iota$) such that if we choose $N=N_0$, $l=l_0$ in $\Theta_1(X)$ and $\Theta_2(X)$, and $k=c_0 s^{2\alpha/(2\alpha+d_s)}$ for some constant $c_0$, then 
	$$
	r(\hat{h}_{\Theta_1(X)})\le \tilde{C}'(e^{-d_s l_0\lambda_{N_0}}s)^{-{\alpha(1+\beta)\over 2\alpha+d_s}} \qquad {\rm and} \qquad r(\hat{h}_{\Theta_2(X)})\le \tilde{C}' s^{-{\alpha(1+\beta)\over 2\alpha+d_s}},
	$$
	where $\tilde{C}'=3C_0 (2L)^{\beta+1}\left(2^{3d_s+1}\kappa /{\rm Vol}\Ncal_v w_{d_s}\right)^{\alpha(\beta+1)/d_s}c_0^{\alpha(\beta+1)/d_s}$.
	If we choose $k=c_0s^{2\alpha/(2\alpha+d)}$ for some constant $c_0$, then we have 
	$$
	r(\hat{h}_X)\le \tilde{C} s^{-{\alpha(1+\beta)\over 2\alpha+d}},
	$$
	where $\tilde{C}=3C_0 (2L)^{\beta+1}\left(4\kappa/ w_d\right)^{\alpha(\beta+1)/d}c_0^{\alpha(\beta+1)/d}$.
	Here, $w_d$ is the volume of Euclidean ball of unit radius, and $d$ and $d_s$ are the dimensions of $\Mcal$ and $\Ncal_s$.
\end{theorem}
Theorem~\ref{thm:convergence} provides the upper bound for the convergence rate of $\hat{h}_X(x)$, $\hat{h}_{\Theta_1(X)}(x)$, and $\hat{h}_{\Theta_2(X)}(x)$. The rate of these upper bounds are sharp since they match the lower bounds for the flat manifold, i.e., subspace \citep{wang2022self}. A comparison between these convergence rates suggests the new data representation can improve the performance of $k$-NN, and the data representation in the diffusion map, $\Theta_2(x)$, is a better choice than $\Theta_1(x)$ because $\Theta_2(x)$ can better preserve the metric on the manifold than $\Theta_1(x)$. $N_0$ in Theorem~\ref{thm:convergence} represents the minimal number of eigenfunctions that can locally recover the geometry of the underlying manifold. $N_0$ depends on the manifold's dimension, volume, and geometric bound in a complex form \citep{portegies2016embeddings}, but we can expect a larger $N_0$ when the dimension and curvature of the manifold are larger. 

\subsection{Finite Samples}
\label{sc:finite}
In the last section, we show that new data representation can improve $k$-NN classifier when we know $\Theta_1(x)$ and $\Theta_2(x)$ in advance. However, in practice, we still need to estimate the data representation from the unlabeled augmented data. One may naturally wonder if the data representation estimated by Algorithms~\ref{ag:aiml} or S1 can still help improve $k$-NN classifier similarly. In this section, we show that this is possible when the sample size of the unlabeled augmented data is sufficiently large. In particular, we shall focus on augmentation invariant diffusion maps with parameter $\alpha=1$ in Algorithm S1, that is $P^{(1)}$, since it can help recover the eigenfunctions of the Laplace Beltrami operator on $\Ncal_s$ regardless of the sampling distribution. We write the estimated data representation as $\hat{\Theta}_1(x)=(\hat{\eta}_{1,m,n,t}(\phi),\ldots,\hat{\eta}_{N,m,n,t}(\phi))$ and $\hat{\Theta}_2(x)=(e^{-l\hat{\lambda}_{1,m,n,t}}\hat{\eta}_{1,m,n,t}(\phi),\ldots,e^{-l\hat{\lambda}_{N,m,n,t}}\hat{\eta}_{N,m,n,t}(\phi))$, where $x=T(\phi,\psi)$, $(\hat{\lambda}_{l,m,n,t},\hat{\eta}_{l,m,n,t})$ is the estimated eigenvalue and eigenfunction by $P^{(1)}$ on $m$ samples of $n$-views data. We also write the resulting $k$-NN classifier as $\hat{h}_{\hat{\Theta}_1(X)}(x)$ and $\hat{h}_{\hat{\Theta}_2(X)}(x)$ in this section. We need the following assumptions to study the performance of new data representation.

\begin{assumption}
	\label{ap:rate}
	It holds that
	\begin{enumerate}
		\item Suppose $f_v(\psi|\phi)$ is uniform distribution on $\Mcal(\phi)$ for any $\phi$;
		\item There exists a constant $\kappa$ such that $1/\kappa<f_s(\phi)<\kappa$, and $f_s(\phi)$ is twice differentiable;
		\item We choose $t=c_1\left((\log m/m)^{2/(4d+13)}+(\log m/n)^{2/(2d+5)}\right)$ for some constant $c_1$;
		\item $m$ is larger than a constant that relies on the smallest gap between the first $N+1$ eigenvalues of $\Lcal_{\Ncal_s}$, i.e., $\min_{1\le l\le N}|\lambda_{l+1}-\lambda_{l}|$, the density $f_s(\phi)$, and the volume, injectivity radius, curvature of the manifolds $\Mcal$ and $\Ncal_s$;
		\item We assume $n\gg \log m$.
	\end{enumerate}
\end{assumption}

These assumptions are used to investigate the convergence rate of estimated eigenvalues and eigenvectors in the $\ell_\infty$ norm. Similar assumptions also appear in proving the spectral convergence rate of diffusion maps \citep{dunson2021spectral}. In the last assumption, we assume the number of views in the randomized kernel is large enough. This assumption can be easily satisfied when we adopt the formulation in Section~\ref{sc:representation} where $n=\infty$. With these assumptions, the following proposition characterizes how fast the eigenvalues and eigenvectors estimated by $P^{(1)}$ converge to the eigenvalues and eigenfunctions of $\Lcal_{\Ncal_s}$.

\begin{theorem}
	\label{thm:convergencerate}
	Suppose that the Assumption~\ref{ap:rate} holds. Let $(\lambda_{l},\eta_l(\phi))$ be the eigenvalues and eigenfunctions of $\Lcal_{\Ncal_s}$ and $(\lambda_{l,m,n,t},\vec{\eta}_{l,m,n,t})$ be the eigenvalues and eigenvectors of $(I-P^{(1)})/t$. With probability at least $1-m^{-2}$, there exist constants $c_{\kappa,l}$ such that $1/\kappa<c_{\kappa,l}<\kappa$,
	$$
	|\lambda_{l}-\lambda_{l,m,n,t}|\le \tilde{C}''\left(\left({\log m\over m}\right)^{3/(8d+26)}+\left({\log m\over n}\right)^{3/(4d+10)}\right),\qquad 1\le l\le N,
	$$
	and 
	$$
	|a_lc_{\kappa,l}[\vec{\eta}_{l,m,n,t}]_i-\eta_l(\phi_i)|\le \tilde{C}''\left(\left({\log m\over m}\right)^{1/(4d+13)}+\left({\log m\over n}\right)^{1/(2d+5)}\right),\qquad 1\le l\le N,\ 1\le i\le m.
	$$
	Here, $\tilde{C}''$ is a constant depending on $c_1$, $d_s$, $d_v$, $d$, $\kappa$, the diameter of $\Ncal_s$ and $\Mcal$, the volume of $\Ncal_s$ and $\Ncal_v$, $C^0$ and $C^2$ norm of $f_s$, the curvature of $\Ncal_v$ and $\Mcal$ and the injectivity radius and reach of manifold $\Mcal$.
\end{theorem}

There are two terms in the convergence rates presented in Theorem~\ref{thm:convergencerate}: the first one mainly relies on the sample size $m$ while the second one depends on the number of views due to the randomized kernel. The first term in the convergence rate of eigenvalues is the same as the one in \cite{dunson2021spectral}, while the first term in the convergence rate of eigenvectors is faster because we do not need to normalize the eigenvectors. The second term is new and comes from the perturbation analysis for the randomized kernel. Although it is unknown whether these convergence rates are sharp or not in our settings, they can help characterize the convergence rate of excess risk in $\hat{h}_{\hat{\Theta}_1(X)}(x)$ and $\hat{h}_{\hat{\Theta}_2(X)}(x)$. The following theorem shows that the estimated data representations can improve the $k$-NN classifier if we have a large enough unlabeled augmented data set.

\begin{theorem}
	\label{thm:finiteconvergence}
	Suppose that the Assumption~\ref{ap:knn} and \ref{ap:rate} hold, and $\hat{\Theta}_1(x)$ and $\hat{\Theta}_2(x)$ are estimated by $m$ samples of $n$-views data. If we choose $l$, $N$, and $k$ in the same way as Theorem~\ref{thm:convergence}, then, with probability at least $1-m^{-2}$,
	$$
	r(\hat{h}_{\hat{\Theta}_1(X)})\le 2\tilde{C}'\kappa^{\alpha(\beta+1)}(e^{d_s l_0\lambda_{N_0}}s)^{-{\alpha(\beta+1)\over 2\alpha+d_s}}+\tilde{C}'''\epsilon_{m,n}^{\alpha(\beta+1)}
	$$
	and
	$$
	r(\hat{h}_{\hat{\Theta}_2(X)})\le 2\tilde{C}'\kappa^{\alpha(\beta+1)}s^{-{\alpha(\beta+1)\over 2\alpha+d_s}}+\tilde{C}'''\epsilon_{m,n}^{\alpha(\beta+1)},
	$$
	where $\tilde{C}'''=6C_0(2L)^{\beta+1}(4\kappa)^{\alpha(\beta+1)}N_0^{\alpha(\beta+1)/2}$ and 
	$$
	\epsilon_{m,n}=\tilde{C}''\left(\left({\log m\over m}\right)^{1/(4d+13)}+\left({\log m\over n}\right)^{1/(2d+5)}\right).
	$$
	Here, $\tilde{C}'$ and $\tilde{C}''$ are the constants in Theorem~\ref{thm:convergence} and \ref{thm:convergencerate}.
\end{theorem}

Theorem~\ref{thm:finiteconvergence} suggests that the estimated data representation $\hat{\Theta}_1(x)$ and $\hat{\Theta}_2(x)$ can improve $k$-NN similarly to $\Theta_1(x)$ and $\Theta_2(x)$ when $m$ and $n$ are large enough. In particular, the convergence rates in Theorem~\ref{thm:finiteconvergence} are the same as in Theorem~\ref{thm:convergence} when $m/\log m\gg s^{(4d+13)/(2\alpha+d_s)}$ and $n/\log m \gg s^{(2d+5)/(2\alpha+d_s)}$. The above requirement for $n$ can be easily satisfied as the next section shows that the number of augmented data can be chosen as infinity when the proposed method is reformulated as a stochastic optimization problem. Therefore, we can expect that $s^{-\alpha(\beta+1)/(2\alpha+d_s)}$ dominates the error bound in Theorem~\ref{thm:finiteconvergence} under the setting of self-supervised learning, where we have limited access to labeled data but a large amount of unlabeled data.

\section{ A Computationally Efficient Formulation}
\label{sc:representation}
Although the method proposed in Section~\ref{sc:submanifold} can help find augmentation invariant data representation, they pose practical challenges in generalizability and computational efficiency when applied to large data sets. The output of Algorithm~\ref{ag:aiml} are just new representations of data points in the underlying data set. Nystr\"om extension is one commonly used way to extend the representation to some new data point \citep{nystrom1930praktische}. However, this way can be computationally expensive since the computational complexity of extending one new point is $O(m)$. In addition, Algorithm~\ref{ag:aiml} needs to evaluate the pairwise distance between any augmented data, so the computation complexity is at least $O(m^2n^2)$ when other parameters are fixed. Is it possible to develop a more generalizable and computationally efficient way for augmentation invariant manifold learning?

Motivated by recent works in \cite{haochen2021provable} and \cite{balestriero2022contrastive}, we borrow an idea from existing self-supervised learning methods. The idea is to parameterize the data representation maps and then apply stochastic optimization techniques to solve the problem. Specifically, let $\Theta_\beta(x):\Mcal\to \RR^N$ be a map with the parameter $\beta\in\Bcal$, which could possibly include a projection head \citep{chen2020simple}. For example, a deep neural network encoder is one of the most popular models for $\Theta_\beta(x)$ in self-supervised learning methods \citep{chen2020simple,grill2020bootstrap,chen2021exploring,zbontar2021barlow}. With the parameterized representation $\Theta_\beta(x)$, the optimization problem in \eqref{eq:discopti} can be rewritten as
\begin{equation}
	\label{eq:optimization}
	\begin{split}
		\min_{\beta\in\Bcal} &\quad {1\over 2}\sum_{i_1,i_2=1}^m\sum_{j_1,j_2=1}^nW_{i_1,j_1,i_2,j_2}\|\Theta_\beta(X_{i_1,j_1})-\Theta_\beta(X_{i_2,j_2})\|^2\\
		{\rm s.t.}&\quad  \Theta_\beta(X_{i,1})=\ldots=\Theta_\beta(X_{i,n}),\quad  1\le i\le m\\
		&\quad  {1\over n(n-1)}\sum_{i=1}^m\sum_{j_1\ne j_2=1}^n\Theta_{\beta,l_1}(X_{i,j_1})\Theta_{\beta,l_2}(X_{i,j_2})=\delta_{l_1,l_2},\quad 1\le l_1,l_2\le N
	\end{split}
\end{equation}
where $\Theta_{\beta,l}(x)$ is the $l$th component of $\Theta_\beta(x)$ and $\delta_{l_1,l_2}$ is the Kronecker delta, that is, $\delta_{l_1,l_2}=1$ if $l_1=l_2$ and $\delta_{l_1,l_2}=0$ if $l_1\ne l_2$. The last constraint is to enforce finding orthogonal eigenvectors. When we observe the infinite number of augmented copies of each sample, we can transform the optimization problem in \eqref{eq:optimization} into the following unconstrained stochastic optimization problem
\begin{equation}
	\label{eq:unconsoptimization}
	\min_{\beta\in\Bcal}\quad \ell(\beta):=\ell_u(\beta) + \lambda_1 \ell_s(\beta) +\lambda_2 \ell_r(\beta)
\end{equation}
where $\ell_u(\beta)$ and $\ell_s(\beta)$ can capture unsupervised and self-supervised signal
$$
\ell_u(\beta)=\sum_{i_1,i_2=1}^m\EE\left(W'_{i_1,i_2}\|\Theta_\beta(X'_{i_1})-\Theta_\beta(X'_{i_2})\|^2\right),\qquad \ell_s(\beta)=\sum_{i=1}^m\EE\left(\|\Theta_\beta(X'_{i})-\Theta_\beta(X''_{i})\|^2\right),
$$
and $\ell_r(\beta)$ is an penalty function to enforce the orthogonality 
$$
\ell_r(\beta)=\sum_{1\le l_1\le l_2\le N}\EE\left(\sum_{i=1}^m\Theta_{\beta,l_1}(X'_{i})\Theta_{\beta,l_2}(X''_{i})-\delta_{l_1,l_2}\right)^2.
$$
Here, $X'_i$ and $X''_i$ are independent augmented copies of $X_i$, $W'_{i_1,i_2}$ is the weight between $X'_{i_1}$ and $X'_{i_2}$, and the expectation is taken with respect to the random data augmentation transformation.

To solve the stochastic optimization problem in \eqref{eq:unconsoptimization}, we can adopt the idea in stochastic approximation, also known as stochastic gradient descent \citep{robbins1951stochastic,wright2022optimization}. More concretely, given a small batch of samples $\Scal\subset \{1,\ldots,m\}$, we can consider an estimator of the loss function $\ell(\beta)$
\begin{equation}
	\label{eq:spoptimization}
	\hat{\ell}(\beta):=\underbrace{\sum_{i\in \Scal}W_{i,\pi(i)}\|\Theta_\beta(X'_{i})-\Theta_\beta(X''_{\pi(i)})\|^2}_{\rm unsupervised\ signal}+\underbrace{\lambda_1 \sum_{i\in \Scal}\|\Theta_\beta(X'_{i})-\Theta_\beta(X''_{i})\|^2}_{\rm self-supervised\ signal}+\underbrace{\lambda_2\Rcal(\Theta_\beta)}_{\rm regularization}
\end{equation}
We now discuss the three parts of the above estimator in detail. 
\begin{itemize}
	\item \textbf{Unsupervised signal.} In $\ell_u(\beta)$, we need to evaluate weights between any pairs of data points, which is the most computationally intensive part. To overcome this issue, we adopt a sub-sampling strategy to approximate the graph Laplacian part. Specifically, let $X'_i$ and $X''_i$ be independent augmented copies of $X_i$, and $\pi$ be a random permutation of $\Scal$ where $\pi(i)$ is matched with $i$. If $\Scal$ is a random subset of $\{1,\ldots,m\}$, then we can expect
	$$
	{1\over |\Scal|}\sum_{i\in \Scal} W'_{i,\pi(i)}\|\Theta_\beta(X'_{i})-\Theta_\beta(X''_{\pi(i)})\|^2\approx {1\over m^2}\sum_{i_1,i_2=1}^m\EE\left(W'_{i_1,i_2}\|\Theta_\beta(X'_{i_1})-\Theta_\beta(X'_{i_2})\|^2\right),
	$$ 
	where $W'_{i,\pi(i)}$ is the weight between $X'_{i}$ and $X''_{\pi(i)}$. Like the classical manifold learning methods, the unsupervised signal part tries to push the representations closer when the Euclidean distance between data points is small.
	\item \textbf{Self-supervised signal.} Both the constraint $\Theta_\beta(X_{i,1})=\ldots=\Theta_\beta(X_{i,n})$ and penalty $\ell_s(\beta)$ aim to find augmentation invariant data representation. To estimate $\ell_s(\beta)$, we can consider the following approximated penalty
	$$
	{1\over |\Scal|}\sum_{i\in \Scal}\|\Theta_\beta(X'_{i})-\Theta_\beta(X''_{i})\|^2\approx {1\over m} \sum_{i=1}^m\EE\left(\|\Theta_\beta(X'_{i})-\Theta_\beta(X''_{i})\|^2\right).
	$$
	This penalty function aims to push the augmented data of the same sample close to each other. 
	\item \textbf{Regularization.} We consider another penalty function to enforce the orthogonality of $\Theta_\beta$ in an unconstrained problem
	$$
	\Rcal(\Theta_\beta)=\sum_{1\le l_1\le l_2\le N}\left(\sum_{i\in \Scal}\Theta_{\beta,l_1}(X'_{i})\Theta_{\beta,l_2}(X''_{i})-\delta_{l_1,l_2}\right)^2.
	$$
	In this way, we can force the different components of $\Theta_\beta$ to be orthogonal to each other. The regularization term is similar to the loss function in Barlow Twins \citep{zbontar2021barlow}.
\end{itemize}
With approximated loss function $\hat{\ell}(\beta)$, we can employ the stochastic gradient descent to minimize $\ell(\beta)$, i.e., $\beta_{k+1}=\beta_{k}-\gamma_k\nabla \hat{\ell}(\beta)$, where $ \hat{\ell}(\beta)$ is evaluated on independent copies of the random batch and augmented data in each iteration. The loss function in \eqref{eq:spoptimization} also allows the application of other stochastic optimization techniques, such as Adam \citep{kingma2014adam}. The pseudocode for the computationally efficient version of augmentation invariant manifold learning is shown in Algorithm~\ref{ag:aimlop}.

\begin{algorithm}[h!]
	\caption{Augmentation Invariant Manifold Learning}\label{ag:aimlop}
	\begin{algorithmic}
		\REQUIRE A set of data $\{X_1,\ldots, X_m\}$, batch size $m'$, encoder $\Theta_\beta$, stochastic data augmentation transformation $\Tcal$, tuning parameters $(t,\lambda_1,\lambda_2)$.
		\FOR{sampled minibatch $\{X_i: i\in \Scal \}$} 
		\STATE Generate two independent augmented copies of each sample $X'_i=\Tcal(X_i)$ and $X''_i=\Tcal(X_i)$ for $i\in \Scal$.
		\STATE Evaluate  representation of each augmented sample $Z'_i=\Theta_\beta(X'_i)$ and $Z''_i=\Theta_\beta(X''_i)$ for $i\in \Scal$.
		\STATE Generate random permutations $\pi_1$ and $\pi_2$ and evaluate 
		$$
		\hat{\ell}_u(\beta)={1\over 2}\sum_{i\in \Scal}\left( W'_{i,\pi_1(i)}\|Z'_i-Z''_{\pi_1(i)}\|^2+W'_{\pi_2(i),i}\|Z'_{\pi_2(i)}-Z''_i\|^2\right).
		$$
		\STATE Evaluate $\hat{\ell}_s(\beta)=\sum_{i\in \Scal}\|Z'_{i}-Z''_{i}\|^2$ and $hat{\ell}_r(\beta)=\sum_{1\le l_1\le l_2\le N}\left(\sum_{i\in \Scal}Z'_{i,l_1}Z''_{i,l_2}-\delta_{l_1,l_2}\right)^2$.
		\STATE Update $\Theta_\beta$ to minimize $\hat{\ell}(\beta)=\hat{\ell}_u(\beta)+\lambda_1\hat{\ell}_s(\beta)+\lambda_2\hat{\ell}_r(\beta)$
		\ENDFOR
		\ENSURE Encoder $\Theta_\beta$
	\end{algorithmic}
\end{algorithm}

We now compare the two formulations of augmentation invariant manifold learning in Algorithms~\ref{ag:aiml} and \ref{ag:aimlop}. Algorithm~\ref{ag:aiml} can return the representation at $m$ points in $O(m^2)$ time if other parameters, such as $N$, $D$, and $n$, are fixed, while Algorithm~\ref{ag:aimlop} can return a generalizable continuous representation map in $O(m)$ time. The continuous representation map returned in Algorithm~\ref{ag:aimlop} allows generalizing the data representation in a constant time, while the Nystr\"om extension needs $O(m)$ time. Notably, the constants in the above computational complexity can rely on many factors, such as the dimension of data, the architecture of $\Theta_\beta$, training strategy, and the representation dimension. Therefore, Algorithm~\ref{ag:aiml} may be faster than Algorithm~\ref{ag:aimlop} on a small data set, while Algorithm~\ref{ag:aimlop} is more suitable for large-scale data sets. To apply Algorithm~\ref{ag:aimlop}, we must carefully choose a set of tuning parameters, including batch size, encoder $\Theta_\beta$, learning rate, and tuning parameters $(t,\lambda_1,\lambda_2)$. However, Algorithm~\ref{ag:aiml} only has two tuning parameters, $t$ and $N$, which can be tuned easily. Another key difference is that finite augmented data of each sample define the objective function in Algorithm~\ref{ag:aiml}, but Algorithm~\ref{ag:aimlop} aims to minimize an objective function defined by infinite augmented copies of each sample. When the encoder $\Theta_\beta$ is complex enough to allow interpolation, such as deep neural network \citep{belkin2018overfitting},  the resulting encoder from Algorithm~\ref{ag:aimlop} is an interpolated map of sample-level representation obtained from Algorithm~\ref{ag:aiml}. We leave more discussions on the two formulations in Section~\ref{sc:numerical}.

We also compare the loss function in \eqref{eq:spoptimization} with those in existing self-supervised representation data learning methods \citep{chen2020simple,grill2020bootstrap,zbontar2021barlow}. Similar to our new method, most self-supervised learning methods are formulated as a stochastic optimization problem and aim to preserve the closeness between the augmented data of the same sample. There are also several key unique characteristics in our new loss function. First, the new loss function keeps the local similarity between the negative pair of data, while the existing self-supervised learning methods either ignore the negative pair (in non-contrastive methods, such as Barlow Twins \citep{zbontar2021barlow}) or push the negative pair far from each other despite their Euclidean distance (in contrastive methods, such as SimCLR \citep{chen2020simple}). Second, while similar to the loss function of Barlow Twins \citep{zbontar2021barlow}, the regularization part in \eqref{eq:spoptimization} can help avoid dimensional collapse observed in some self-supervised learning methods \citep{hua2021feature,wen2022mechanism}. Third, the tuning parameters $\lambda_1$ and $\lambda_2$ help balance the unsupervised and self-supervised signals.  

\section{Numerical Experiments}
\label{sc:numerical}

In this section, we study the numerical performance of augmentation invariant manifold learning. We conducted the numerical experiments on both simulated and real data sets. Besides the results presented below, Section S2.1 in Supplement Material studies tuning parameters' effects in Algorithm~\ref{ag:aimlop}.

\subsection{Performance on Simulated Data}
\label{sc:simulation}

To simulate the data, we consider three product manifolds: the first one is the torus used in the introduction (torus), that is, $x=(x_1,x_2,x_3)$ such that $x_1=(10+5\cos \phi)\cos \psi$, $x_2=(10+5\cos \phi)\sin \psi$, and $x_3=5\sin \phi$, where $ \phi,\psi\in (0,2\pi]$; the second one is the Swiss roll in 3-dimensional space (Swiss roll 1), that is, $x=(x_1,x_2,x_3)$ such that $x_1=\phi\cos \phi$, $x_2=\phi\sin \phi$, and $x_3=\psi$, where $ \phi\in (1.5\pi,4.5\pi)$ and $\psi\in (0,10)$; the third one is the Swiss roll with changing role of $\phi$ and $\psi$ (Swiss roll 2), that is, $x=(x_1,x_2,x_3)$ such that $x_1=\psi\cos \psi$, $x_2=\psi\sin \psi$, and $x_3=\phi$, where $\phi\in (0,10)$ and $\psi\in (1.5\pi,4.5\pi)$. We follow the same procedure in Section~\ref{sc:model} to generate multi-view augmented data, assuming the probability density functions $f_s(\phi)$ and $f_v(\psi|\phi)$ are uniform distribution. To recover the geometrical structure of $\Ncal_s$, we apply three different augmentation invariant manifold learning methods: the Laplacian eigenmaps in Algorithm~\ref{ag:aiml} ($\hat{\Theta}_1$), the diffusion maps in Algorithm S1 with $\alpha=1/2$ and $l=0.1$($\hat{\Theta}_2$), and the diffusion maps in Algorithm S1 with $\alpha=1$ and $l=0.1$ ($\hat{\Theta}_3$). Figure~\ref{fg:differentmanifold} summarizes the 2-dimensional embedding by applying these three methods to the three different types of manifolds. From Figure~\ref{fg:differentmanifold}, it is clear that all three different methods can recover the geometrical structure of $\Ncal_s$ very well. 

\begin{figure}[h!]
	\begin{center}
		\begin{tikzpicture}[scale=0.85]
			\node[inner sep=0pt] at (0,2)
			{\footnotesize Laplacian Eigenmaps};
			\node[inner sep=0pt] at (4.5,2)
			{\footnotesize Diffusion Maps ($\alpha=0.5$)};
			\node[inner sep=0pt] at (9,2)
			{\footnotesize Diffusion Maps ($\alpha=1$)};
			
			\node[inner sep=0pt] at (0,0)
			{\includegraphics[width=.2\textwidth]{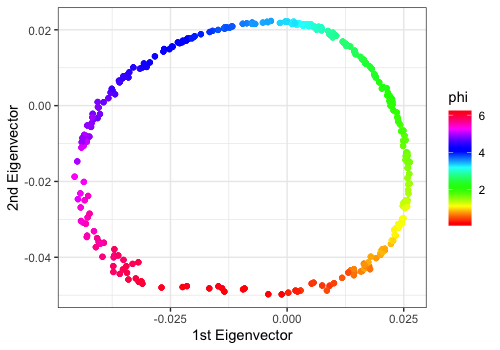}};
			\node[inner sep=0pt] at (4.5,0)
			{\includegraphics[width=.2\textwidth]{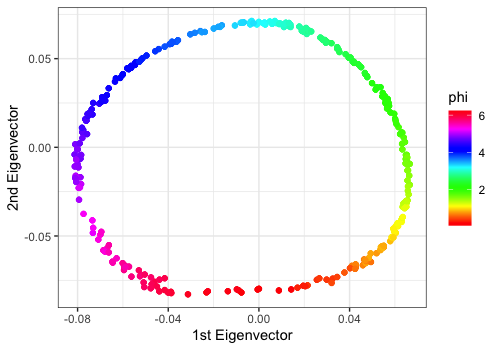}};
			\node[inner sep=0pt] at (9,0)
			{\includegraphics[width=.2\textwidth]{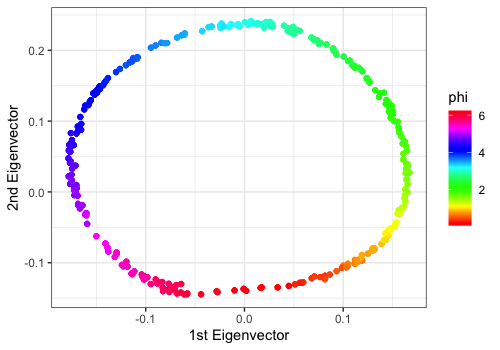}};
			
			\node[inner sep=0pt] at (0,-3.5)
			{\includegraphics[width=.2\textwidth]{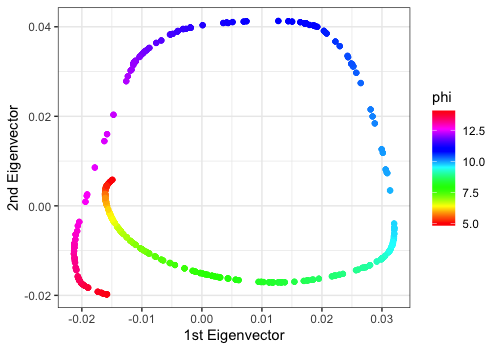}};
			\node[inner sep=0pt] at (4.5,-3.5)
			{\includegraphics[width=.2\textwidth]{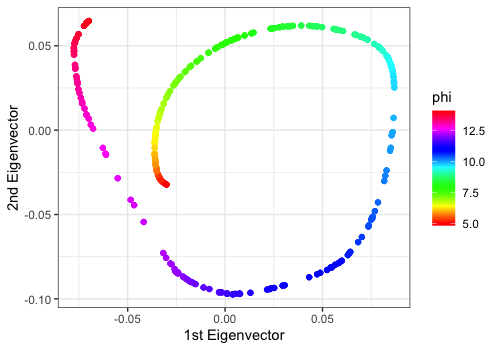}};
			\node[inner sep=0pt] at (9,-3.5)
			{\includegraphics[width=.2\textwidth]{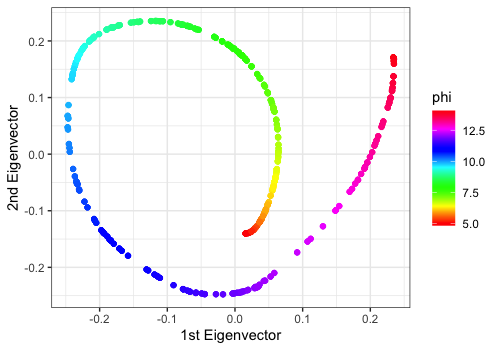}};
			
			\node[inner sep=0pt] at (0,-7)
			{\includegraphics[width=.2\textwidth]{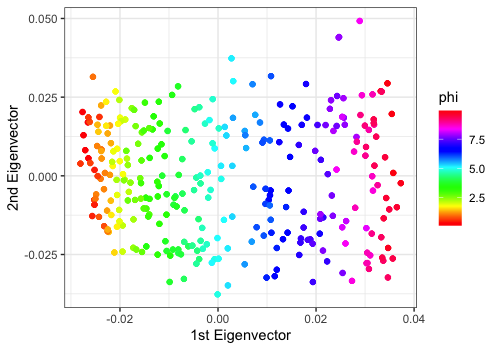}};
			\node[inner sep=0pt] at (4.5,-7)
			{\includegraphics[width=.2\textwidth]{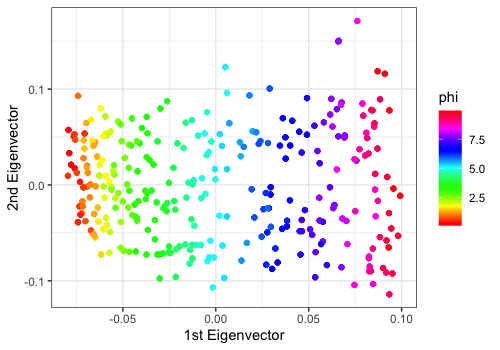}};
			\node[inner sep=0pt] at (9,-7)
			{\includegraphics[width=.2\textwidth]{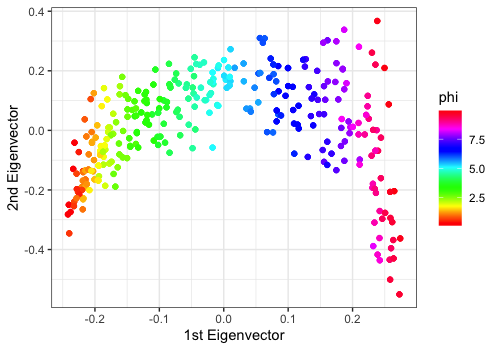}};
			
			\node[inner sep=0pt,rotate=90] at (-2.6,0)
			{\footnotesize Torus};
			\node[inner sep=0pt,rotate=90] at (-2.6,-3.5)
			{\footnotesize Swiss Roll 1};
			\node[inner sep=0pt,rotate=90] at (-2.6,-7)
			{\footnotesize Swiss Roll 2};
			
		\end{tikzpicture}
	\end{center}
	\caption{The embedding of three different product manifolds in $\RR^2$. Different columns corresponds to different augmentation invariant manifold learning methods. All figures are colored by $\phi$.}
	\label{fg:differentmanifold}
\end{figure}

\begin{table}[h!]
	\caption{Comparisons of different data representations on $k$-NN when sample size $s$ is different. The misclassification error is reported in the table. $\hat{\Theta}_1$: the augmentation invariant Laplacian eigenmaps; $\hat{\Theta}_2$: the augmentation invariant diffusion maps ($\alpha=1/2$ and $l=0.1$); $\hat{\Theta}_3$: the augmentation invariant diffusion maps ($\alpha=1$ and $l=0.1$)}\label{tb:samplesize}
	\centering
	\begin{tabular}{cccccc}  
		\hline
		Manifold &Data Representation& $s=50$ & $s=100$ & $s=200$ & $s=300$ \\ 
		\hline
		\multirow{4}{*}{Torus}&$\hat{h}_{X}$ & 0.423 & 0.380 & 0.301 & 0.274 \\ 
		&$\hat{h}_{\hat{\Theta}_1(X)}$ & 0.311 & 0.236 & 0.216 & 0.223 \\ 
		&$\hat{h}_{\hat{\Theta}_2(X)}$ & 0.306 & 0.232 & 0.218 & 0.221 \\ 
		&$\hat{h}_{\hat{\Theta}_3(X)}$ & 0.305 & 0.234 & 0.216 & 0.221 \\ 
		\hline
		\multirow{4}{*}{Swiss Roll}&$\hat{h}_{X}$ & 0.437 & 0.438 & 0.435 & 0.440 \\ 
		&$\hat{h}_{\hat{\Theta}_1(X)}$ & 0.423 & 0.427 & 0.379 & 0.343 \\ 
		&$\hat{h}_{\hat{\Theta}_2(X)}$ & 0.432 & 0.421 & 0.364 & 0.347 \\ 
		&$\hat{h}_{\hat{\Theta}_3(X)}$ & 0.434 & 0.410 & 0.357 & 0.332 \\ 
		\hline
	\end{tabular}
\end{table}

In the next simulation experiment, we study if the new data representation can help improve the downstream analysis. In particular, we focus on the classification on the manifold with the $k$-NN classifier. We consider two different product manifolds: the torus and Swiss roll 2 in the previous experiment and four data representations: the original data $X$ and data representations estimated by $\hat{\Theta}_1$, $\hat{\Theta}_2$, and $\hat{\Theta}_3$ in the previous experiment. To generate the label $Y$, we assume the regression function is $\gamma(x)=\PP(Y=1|X=x)=|\sin(\phi)|$ when $x\in \Mcal(\phi)$. We vary the sample size of the training sets $s=50, 100, 200,$ and $300$ and always choose the sample size in the testing set as $100$. The number of views in this simulation experiment is $3$, and the misclassification error is used as evaluation criteria. The results from 100 repeats of the simulation experiment are summarized in Table~\ref{tb:samplesize}. Table~\ref{tb:samplesize} suggests that the new data representations can help improve $k$-NN, and the misclassification error is smaller when the sample size $s$ is larger. In addition, the performance improved by three different representations are similar in Table~\ref{tb:samplesize}.

\begin{table}[h!]
	\caption{Comparisons of different data representations on $k$-NN when the regression function $\gamma(x)=\PP(Y=1|X=x)$ is different. The misclassification error is reported in the table. $\hat{\Theta}_1$: the augmentation invariant Laplacian eigenmaps; $\hat{\Theta}_2$: the augmentation invariant diffusion maps ($\alpha=1/2$ and $l=0.1$); $\hat{\Theta}_3$: the augmentation invariant diffusion maps ($\alpha=1$ and $l=0.1$).}\label{tb:response}
	\centering
	\begin{tabular}{ccccc}  
		\hline
		Data Representation & $\delta=1$ & $\delta=2$ & $\delta=3$ & $\delta=4$ \\ 
		\hline
		$\hat{h}_{X}$ & 0.267 & 0.400 & 0.430 & 0.425 \\ 
		$\hat{h}_{\hat{\Theta}_1(X)}$ & 0.208 & 0.226 & 0.242 & 0.280 \\ 
		$\hat{h}_{\hat{\Theta}_2(X)}$ & 0.207 & 0.222 & 0.244 & 0.273 \\ 
		$\hat{h}_{\hat{\Theta}_3(X)}$ & 0.211 & 0.223 & 0.243 & 0.275 \\ 
		\hline
	\end{tabular}
\end{table}

In the last set of simulation experiments, we consider the effect of different regression functions $\gamma(x)$. Specifically, we consider a similar setting to the previous simulation experiment. We only focus on the torus manifold with $s=300$ and choose the regression function as $\gamma_\delta(x)=|\sin(\delta*\phi)|$ for $\delta=1,2,3,4$ when $x\in \Mcal(\phi)$. When $\delta$ gets larger, the regression function becomes less smooth. We repeat the simulation 100 times and summarize the results in Table~\ref{tb:response}. Through Table~\ref{tb:response}, we can conclude that the classification problem gets more difficult when the regression is less smooths and the new data representation resulting from augmentation invariant manifold learning can also be helpful in the non-smooth case. 

\subsection{Comparisons of Two Formulations}
\label{sc:twoformulation}

This section compares two formulations of augmentation invariant manifold learning on the MNIST data set \citep{lecun1998gradient}. MNIST data set includes $60000$ training images and $10000$ testing images. All images are $28\times28$ gray-scale handwritten digits ranging from $0$ to $9$. To apply augmentation invariant manifold learning, we consider the following data augmentation transformation: the image is first rotated in $b$ degree, then resized to $a\times a$, and finally randomly cropped to $28\times 28$. Here, $a$ is a random number drawn from $\{29,30,31,32\}$, and $b$ is a random number drawn between $-10$ and $10$. In this experiment, we consider five augmentation invariant manifold learning methods: Algorithm~\ref{ag:aiml} with $n=3$ and $5$ and Algorithm~\ref{ag:aimlop} with $25, 50$, and $100$ epochs. In Algorithm~\ref{ag:aimlop}, we choose a standard convolutional neural network encoder as the encoder: two consecutive convolution layers (with relu activation)+pooling layers, and a fully connected layer \citep{lecun2015deep,goodfellow2016deep}, and the encoder is trained by Adam optimizer. 

\begin{figure}[h!]
	\begin{center}
		\includegraphics[width=.45\textwidth]{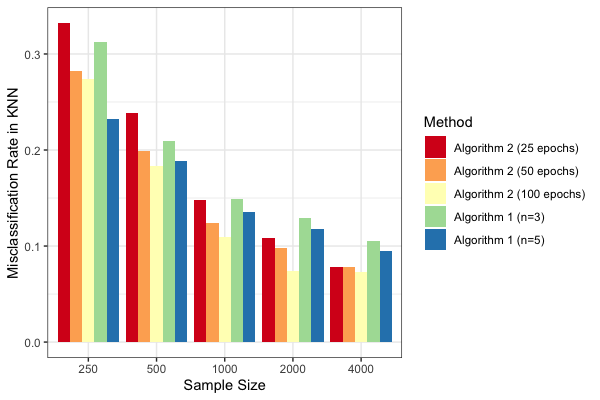}
		\hskip 20pt
		\includegraphics[width=.45\textwidth]{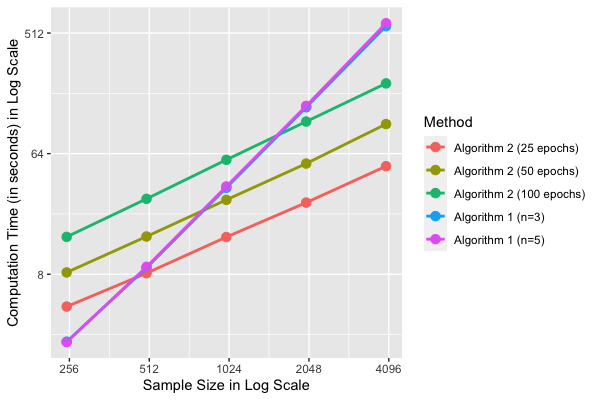}
	\end{center}
	\caption{Comparisons of two formulations of augmentation invariant manifold learning. The left figure summarizes the misclassification rate of different methods, and the right figure reports the computation time of different methods. The estimated slopes in the right figure are: 0.87 (Algorithm 2, 25 epochs), 0.92 (Algorithm 2, 50 epochs), 0.96 (Algorithm 2, 100 epochs), 1.97 (Algorithm 1, $n=3$), 1.99 (Algorithm 1, $n=5$).}
	\label{fg:twoformulation}
\end{figure}

To compare performance, we randomly draw $250, 500, 1000, 2000$, and $4000$ images from the $60000$ training images in the representation learning stage and randomly split them into two parts with  $80\%$ and $20\%$ in the downstream analysis: the former part is used to train $k$-NN classifier, and the later part is used to evaluate misclassification error. We chose misclassification rate and computation time as our evaluation criteria in this experiment. All these methods are evaluated on the same desktop (Intel Core i7 @3.8 GHz/16GB). The results are summarized in Figure~\ref{fg:twoformulation}. From the left panel of Figure~\ref{fg:twoformulation}, two different formulations achieve similar misclassification rates, although Algorithm~\ref{ag:aimlop}, with $100$ epochs, performs best in most cases. The slight improvement in Algorithm~\ref{ag:aimlop} could be explained by the regularization effect of the deep neural network and the infinite augmented data of each sample in stochastic optimization. In addition, the right panel of Figure~\ref{fg:twoformulation} suggests that Algorithm~\ref{ag:aimlop} is more computationally efficient and scalable than Algorithm~\ref{ag:aiml}. The results in Figure~\ref{fg:twoformulation} confirm that the computation complexities of Algorithms~\ref{ag:aiml} and \ref{ag:aimlop} are $O(m^2)$ and $O(m)$, respectively.

\subsection{Comparisons with Other Self-Supervised Learning Methods}
\label{sc:comparison}

\begin{figure}[b!]
	\begin{center}
		\includegraphics[width=.7\textwidth]{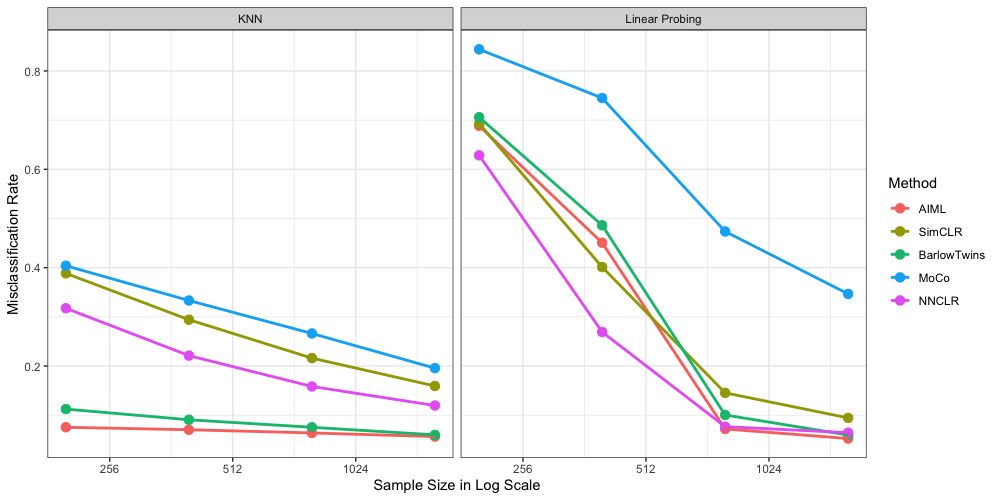}
	\end{center}
	\caption{Comparisons of five self-supervised learning methods on MNIST data set. The left figure summarizes the misclassification rate of the downstream $k$-NN classifier, and the right figure shows the misclassification rate of the downstream classifier trained by linear probing. }
	\label{fg:comparison_mnist}
\end{figure}

In this section, we compare the performance of augmentation invariant manifold learning with several state-of-the-art self-supervised learning methods. Specifically, we consider four existing methods: SimCLR \citep{chen2020simple}, Barlow Twins \citep{zbontar2021barlow}, MoCo \citep{he2020momentum}, and NNCLR \citep{dwibedi2021little}. We first compare these five methods on the MNIST data set. To make a fair comparison, the five methods are given the same encoder and data augmentation transformation and trained with the same batch size, the number of epochs, and the optimizer. In the downstream analysis, we consider two different classifiers: $k$-NN classifier and linear probing, that is, training a linear classifier on top of frozen representations. The representation learning methods are trained on $60000$ unlabeled images, and the downstream classifiers are trained on $200, 400, 800$, and $1600$ labeled images. The resulting misclassification rates are summarized in Figure~\ref{fg:comparison_mnist}. An interesting observation from Figure~\ref{fg:comparison_mnist} is that augmentation invariant manifold learning can achieve performance similar to that of the Barlow Twins. This observation is expected as the regularization term in loss function of augmentation invariant manifold learning is similar to the loss function of Barlow Twins (compared with augmentation invariant manifold learning, it conducts normalization before evaluating loss function). In addition, we can also observe that augmentation invariant manifold learning and Barlow Twins can achieve better performance in the $k$-NN classifier. While the representation learned by NNCLR and SimCLR can better improve linear probing in the low-sample phase, augmentation invariant manifold learning and Barlow Twins can perform better in the high-sample phase.

\begin{figure}[h!]
	\begin{center}
		\includegraphics[width=.7\textwidth]{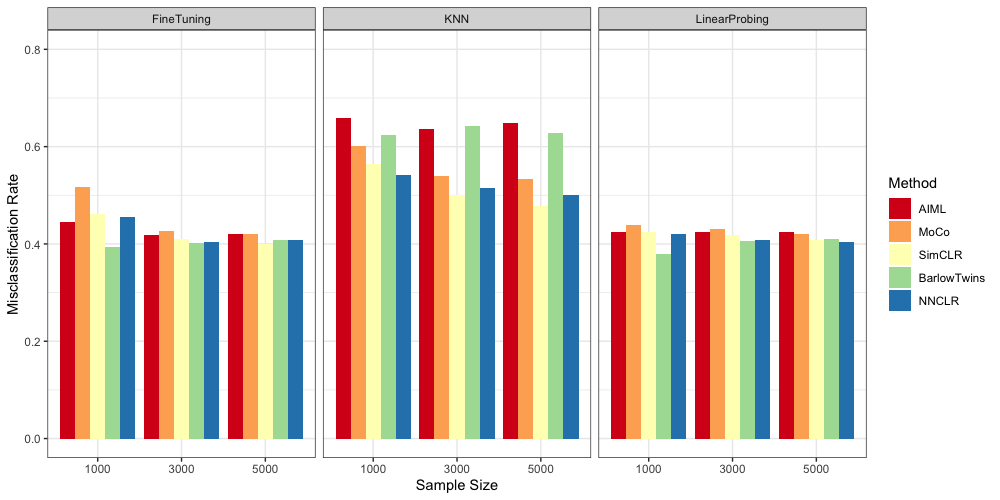}
	\end{center}
	\caption{Comparisons of five self-supervised learning methods on STL10 data set. The left figure summarizes the misclassification rate of the classifier trained by fine-tuning, the middle figure presents misclassification rate of the downstream $k$-NN classifier, and the right figure shows the misclassification rate of the downstream classifier trained by linear probing. }
	\label{fg:comparison_stl10}
\end{figure}

Besides the MNIST data set, we also applied these five methods to other data sets, including STL-10 \citep{pmlr-v15-coates11a} and ImageNet \citep{deng2009imagenet,chrabaszcz2017downsampled}. Unlike the MNIST data set, the STL-10 data set includes $100000$ unlabeled images for self-supervised learning, $5000$ labeled images for training, and $8000$ labeled images for testing. The images in the STL-10 data set are $96\times 96\times3$ color images and can be classified into 10 classes: airplane, bird, car, cat, deer, dog, horse, monkey, ship, and truck. Because of color images, the data augmentation transformation also adopts a random color jitter to vary images' brightness, contrast, hue, and saturation after resized cropping. Besides the $k$-NN classifier and linear probing, we also consider a classifier trained through fine-tuning in the downstream analysis, updating all model parameters in the training progress. The downstream classifiers are trained by $1000, 3000$, and $5000$ labeled images. The results are summarized in Figure~\ref{fg:comparison_stl10}. Figure~\ref{fg:comparison_stl10} shows that all five methods perform similarly in the downstream classifier trained by linear probing and fine-tuning. Augmentation invariant manifold learning and Barlow Twins slightly inflate the misclassification rate of the downstream $k$-NN classifier. Interestingly, the downstream classifiers trained by linear probing and fine-tuning perform much better than the $k$-NN classifier. This could be because the images in the STL-10 data set are drawn from a union of multiple disjoint manifolds rather than a single manifold \citep{wang2023linear}. In addition to the STL-10 data set, we also consider the ImageNet data set, which includes $1281167$ training samples and $50000$ validation samples. The images in the ImageNet data set are classified into 1000 classes, and each image has a size $32\times 32\times 3$ as we adopt a downsampled version of the original data set, ImageNet$32\times32$ \citep{chrabaszcz2017downsampled}. In all self-supervised learning methods, we adopt the same ResNet-20 network \citep{he2016deep} as the encoder and consider the same data augmentation as the previous experiment. The encoder is pre-trained on all training samples, and the downstream classifiers are trained via linear probing and fine-tuning on $1000, 4000, 16000$, and $64000$ labeled images. The top-5 accuracy summarized in Figure S3 of Supplementary Material indicates that these five methods achieve similar performance in the downstream analysis. Figures~\ref{fg:comparison_mnist}, \ref{fg:comparison_stl10}, and S3 in Supplementary Material show that augmentation invariant manifold learning is a competitive and practical self-supervised learning method despite its tractable theoretical properties.

\section{Concluding Remarks}
\label{sc:conclusion}

In this paper, we introduce a new product manifold model for data augmentation and theoretically characterize the role of data augmentation in self-supervised learning. Under the newly proposed model, the regression function defined on the product manifold can be decomposed into two parts
$$
\gamma(x)=\tilde{\gamma}\left(T^{-1}_\pi(x)\right),\qquad x\in \Mcal.
$$
Here, $T^{-1}_\pi(x)=\pi(T^{-1}(x))=\phi_x$ is a projection function that maps $x$ to $\phi_x$ when $x\in \Mcal(\phi_x)$, where $\pi(\phi,\psi)=\phi$ is a projection function on the product manifold. The augmentation invariant manifold learning tries to estimate $T^{-1}_\pi(x)$ (or an equivalent one) from the unlabeled augmented data. When $T^{-1}_\pi(x)$ can be estimated accurately, it is sufficient to estimate a $d_s$-dimensional function $\tilde{\gamma}(\phi)$ instead of a high dimensional function $\gamma(x)$ in the downstream analysis. This explains why the augmentation invariant manifold learning can help improve $k$-NN in the downstream analysis. 

In augmentation invariant manifold learning, several tuning parameters exist in Algorithms~\ref{ag:aiml} and \ref{ag:aimlop}, including bandwidth $t$, number of eigenvectors $N$, architecture of the encoder, and learning rate. The theoretical analysis provides some recommendations for these tuning parameters' choices, but it is challenging to evaluate these theoretical choices in practice as the properties of the underlying manifold are usually unknown in advance. Although some data-driven methods are proposed for bandwidth $t$ in manifold learning literature \citep{ding2020impact}, these methods cannot directly apply to augmentation invariant manifold learning. We recommend following the existing self-supervised learning methods to select tuning parameters \citep{chen2020simple,zbontar2021barlow}. Specifically, we can consider the performance measure of downstream analysis, such as prediction accuracy, on the validation set as our criteria for choosing tuning parameters and select the combination of tuning parameters that achieves the best performance on the validation set. 

To simplify the theoretical analysis, we assume this paper's conditional distribution $f_v(\psi|\phi)$ is uniform. The theoretical analysis can also be extended to cases where the conditional distribution $f_v(\psi|\phi)$ is not uniform, but $\phi$ is independent of $\psi$. It is unclear if the proposed augmentation invariant manifold learning works when $\phi$ and $\psi$ are dependent, and we leave it as the future work. Another critical assumption we made is that manifold $\Mcal$ is an isometric embedding of a product manifold. With this assumption, the simplified analysis offers theoretical insights and motivates new methodology. This assumption may be violated when the data have a complex structure. It would also be interesting to explore the asymptotic behavior of augmentation invariant manifold learning when this assumption is relaxed.  

\section*{Acknowledgment}
This project is supported by grants from the National Science Foundation (DMS-2113458 and DBI-2243257).

\bibliographystyle{plainnat}
\bibliography{SubmanifoldLearning}

\newpage
\setcounter{section}{0}
\setcounter{equation}{0}
\setcounter{figure}{0}
\setcounter{table}{0}
\setcounter{algorithm}{0}
\setcounter{theorem}{0}
\def\theequation{S\arabic{section}.\arabic{equation}}
\def\thesection{S\arabic{section}}
\def\thefigure{S\arabic{figure}}
\def\thetable{S\arabic{table}}
\def\thealgorithm{S\arabic{algorithm}}
\def\thetheorem{S\arabic{theorem}}

\begin{center}
	{\Large \bf Supplement to ``Augmentation Invariant Manifold Learning"
	}
\end{center}

In this supplementary material, we provide the proof for the main results and all the technical lemmas. In addition, the supplementary material also includes extra numerical results.

\section{Augmentation Invariant Diffusion Maps}
\label{sc:aidm}
This section presents an augmentation invariant version of diffusion maps, summarized in Algorithm~\ref{ag:aidm}. A similar argument in Theorem~\ref{thm:aiml} makes it sufficient to study the corresponding continuous version of Algorithm~\ref{ag:aidm}. Following the notation in the proof of Theorem~\ref{thm:aiml}, we define a weighted kernel for given $\alpha$ 
$$
\tilde{h}_t^{(\alpha)}(\phi_1,\phi_2)={\tilde{h}_t(\phi_1,\phi_2) \over f_{s,t}^\alpha(\phi_1)f_{s,t}^\alpha(\phi_2)},\qquad {\rm where}\quad f_{s,t}(\phi_1)=\int \tilde{h}_t(\phi_1,\phi_2) f_s(\phi_2)d\phi_2.
$$
Then, the operator for a continuous version of Algorithm~\ref{ag:aidm} is defined as 
$$
\Pcal_{t,\alpha} g(\phi_1)=\int {\tilde{h}_t^{(\alpha)}(\phi_1,\phi_2)\over D_t^{(\alpha)}(\phi_1)}g(\phi_2)f_s(\phi_2)d\phi_2,\qquad {\rm where}\quad D_t^{(\alpha)}(\phi_1)=\int \tilde{h}_t^{(\alpha)}(\phi_1,\phi_2) f_s(\phi_2)d\phi_2,
$$
where $g(\phi)$ is a function defined on $\Ncal_s$. The following theorem characterizes the asymptotic behavior of operator $\Pcal_{t,\alpha}$. 

\begin{theorem}
	\label{thm:aidm}
	Define $L_{t,\alpha}=(I-\Pcal_{t,\alpha})/t$. Then we can show that
	$$
	\lim_{t\to 0}L_{t,\alpha} g={\Lcal_{\Ncal_s}(gf_s^{1-\alpha})\over f_s^{1-\alpha}}-{\Lcal_{\Ncal_s}(f_s^{1-\alpha})\over f_s^{1-\alpha}}g,
	$$
	where $\Lcal_{\Ncal_s}$ is Laplace Beltrami operator on $\Ncal_s$ and $g(\phi)$ is a function defined on $\Ncal_s$.
\end{theorem}
We omit the proof here as the same arguments in \cite{coifman2006diffusion} still hold if we have Lemma 10. The result suggests we can always choose $\alpha=1$ to recover the Laplace Beltrami operator on $\Ncal_s$.

\begin{algorithm}[h]
	\caption{Augmentation Invariant Diffusion Maps}
	\label{ag:aidm}
	\begin{algorithmic}
		\REQUIRE set of augmented data: $(X_{i,1},\ldots,X_{i,n})$, $1\le i\le m$ and the parameters $\alpha$ and $l$.
		\STATE Step 1: Calculate the weights between samples 
		$$
		W_{i_1,i_2}={1\over n^2}\sum_{j_1,j_2=1}^n\exp\left(-{\|X_{i_1,j_1}-X_{i_2,j_2}\|^2\over t}\right),\qquad i_1,i_2=1,\ldots,m.
		$$
		\STATE   Step 2: Normalize the weight matrix $W^{(\alpha)}=D^{-\alpha}WD^{-\alpha}$, where $D$ is a diagonal degree matrix of $W$, i.e. $D_{i,i}=\sum_jW_{i,j}$.
		\STATE   Step 3: Evaluate the transition matrix $P^{(\alpha)}=(D^{(\alpha)})^{-1}W^{(\alpha)}$, where $D^{(\alpha)}$ is a diagonal degree matrix of $W^{(\alpha)}$.
		\STATE   Step 4: Find the first $N$ eigenvectors and eigenvalues of $P^{(\alpha)}$, name them $\vec{\eta}_1,\ldots,\vec{\eta}_N$ and $\lambda_1,\ldots,\lambda_N$. 
		\ENSURE The representation for each sample: $(X_{i,1},\ldots,X_{i,n}) \to (e^{-l\lambda_1}\eta_{1,i},\ldots, e^{-l\lambda_N}\eta_{N,i})$
	\end{algorithmic}
\end{algorithm}

\section{Extra Results of Numerical Experiments}

\subsection{Effect of Tuning Parameters}
\label{sc:tunning}

This section studies the effect of various tuning parameters in Algorithm~\ref{ag:aimlop}, including batch size, number of epochs, parameters $(t,\lambda_1,\lambda_2)$, the encoder's architecture, the encoder's output dimension, and the data augmentation transformation. The numerical experiments are conducted in the MNIST data set. In the first numerical experiment, we investigate how the architecture of the encoder affects the performance of Algorithm~\ref{ag:aimlop}. We consider three deep neural network architectures: one consecutive convolution layer (with relu activation)+pooling layer and a fully connected layer (Architecture 1), two consecutive convolution layers (with relu activation)+pooling layers and a fully connected layer (Architecture 2), three consecutive convolution layers (with relu activation)+pooling layers and a fully connected layer (Architecture 3). We draw $5000, 10000, 20000$, and $40000$ images from the $60000$ training images to learn representation, train downstream $k$-NN classifier on $20\%$ of them, and evaluate the performance on the other $20\%$ of them. The misclassification rates are summarized in the left figure of Figure~\ref{fg:tuning}. The results suggest that the architecture choices in the encoder are essential to the performance of Algorithm~\ref{ag:aimlop}. A simple encoder cannot well approximate the eigenfunctions and a complex encoder can make training difficult. Furthermore, we also compare the performance of Architecture 2 when the dimension of the final layer is different. The results summarized in Table~\ref{tb:width} show that a lower dimension of the encoder's output could result in reduced performance, while the encoder does not benefit from a much higher dimension of the encoder's output.

In the next set of numerical experiments, we study the effect of tuning parameters $(t,\lambda_1,\lambda_2)$. To evaluate the effect of $t$, we fix all the other parameters but vary $t=0.1, 0.05, 0.01, 0.005$, and $0.001$. Like the last experiment, $5000, 10000, 20000$, and $40000$ images are selected to learn representation. $20\%$ of them are used to train downstream $k$-NN classifier, and the other $20\%$ are used to evaluate the performance. The resulting misclassification rates are summarized in Figure~\ref{fg:temperature}. The results show that the best choice of $t$ depends on the sample size, and a smaller $t$ is needed when we have a larger sample size. Furthermore, we designed a separate experiment to study the effect of $\lambda_1$ and $\lambda_2$. In this experiment, we consider $\lambda_1=0.1, 1, 10, 100, 1000$ and $\lambda_2=0.1, 1, 10, 100, 1000$, and randomly select $10000$ images to train representation learning.  The misclassification rates summarized in Table~\ref{tb:lambda} suggest that Algorithm~\ref{ag:aimlop} can work well when $\lambda_2$ is larger than $\lambda_1$. If $\lambda_2$ is smaller than $\lambda_1$, the misclassification rates can be highly inflated, suggesting the necessity of strong regularization on the orthogonality of the resulting representation.

The next numerical experiment is designed to evaluate the effect of  batch size and the number of epochs in  Algorithm~\ref{ag:aimlop}. Similar to the last experiment, $10000$ images are randomly selected to learn representation, and we draw $4000$ of them as training and testing sets for the downstream $k$-NN classifier.  In Algorithm~\ref{ag:aimlop}, we vary batch size from $50$ to $400$ and consider $25, 50$, and $100$ epochs. The resulting misclassification rates in the $k$-NN classifier are summarized in Table~\ref{tb:batch}. From the results, it is clear that reducing the batch size and increasing the number of epochs can lead to a lower misclassification rate in the downstream $k$-NN classifier. This phenomenon is also consistent with the observation in Figure~\ref{fg:twoformulation}.

In the last experiment of this section, we study how different choices of data augmentation transformation in representation learning can affect the performance of the downstream $k$-NN classifier. Specifically, we consider two different ways of data augmentation: the first one is to resize the image to $a\times a$ and then randomly crop it to $28\times 28$, where $a$ is a random number drawn from $\{29,30,31,32\}$; the second one is to rotate the image in $b$ degree and then apply the same resizing and cropping as before, where $b$ is a random number drawn between $-10$ and $10$. The second data argumentation is more complex than the first one. Figure~\ref{fg:argumentationexample} illustrates these two ways of data augmentation on the image of a handwritten digit. The representation is trained on $60000$ unlabeled images, and the downstream $k$-NN classifier is trained on $1000, 2000, 3000$, and $4000$ labeled images. The reported misclassification rates are summarized in the right figure of Figure~\ref{fg:tuning}. In Figure~\ref{fg:tuning}, we observe that a more complex data augmentation transformation is more helpful in downstream analysis, consistent with our previous sections' theoretical results. 

\begin{table}[h!]
	\caption{Effect of the dimension of the encoder's output  on the misclassification rate of downstream $k$-NN classifier.}
	\label{tb:width}
	\centering
	\begin{tabular}{cc}  
		\hline
		Dimension of Final Layer & Misclassification Rate \\   
		\hline
		20 & 0.138 \\   
		40 & 0.080 \\   
		60 & 0.073 \\   
		80 & 0.086 \\    
		\hline
	\end{tabular}
\end{table}

\begin{figure}[h!]
	\begin{center}
		\includegraphics[width=.8\textwidth]{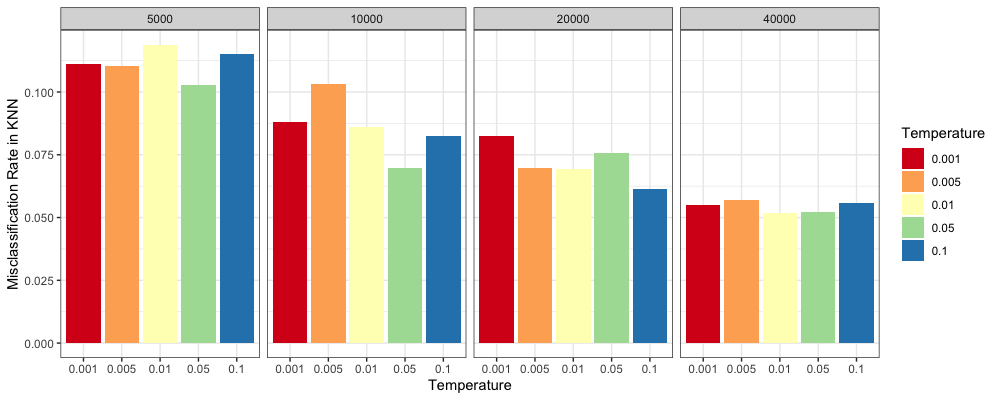}
	\end{center}
	\caption{Effect of $t$ on augmentation invariant manifold learning and the misclassification rate of downstream $k$-NN classifier.}
	\label{fg:temperature}
\end{figure}

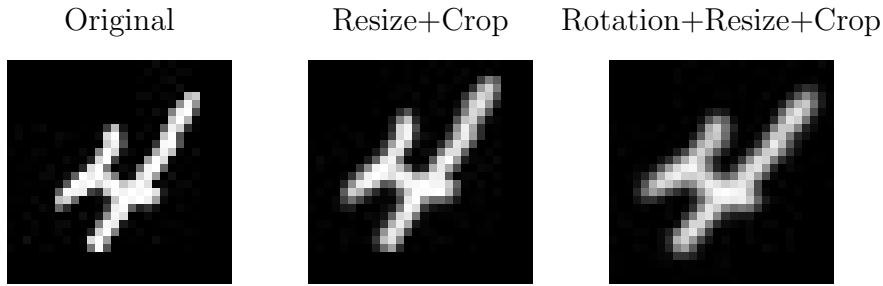
\begin{figure}[h!]
	\begin{center}
		\begin{tikzpicture}[scale=1]
			\node[inner sep=0pt] at (0,2)
			{Original};
			\node[inner sep=0pt] at (4,2)
			{Resize+Crop};
			\node[inner sep=0pt] at (8,2)
			{Rotation+Resize+Crop};
			
			\node[inner sep=0pt] at (0,0)
			{\includegraphics[width=.18\textwidth]{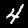}};
			\node[inner sep=0pt] at (4,0)
			{\includegraphics[width=.18\textwidth]{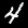}};
			\node[inner sep=0pt] at (8,0)
			{\includegraphics[width=.18\textwidth]{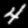}};
		\end{tikzpicture}
	\end{center}
	\caption{An example of augmented data in MNIST data set: from left to right are original, resize+crop, and rotation+resize+crop.}
	\label{fg:argumentationexample}
\end{figure}

\begin{table}[h!]
	\caption{Effect of the batch size and number of epochs on the misclassification rate of downstream $k$-NN classifier.}
	\label{tb:batch}
	\centering
	\begin{tabular}{ccccc}  
		\hline
		&\multicolumn{4}{c}{Batch Size}\\
		\cline{2-5}
		& 50 & 100 & 200 & 400 \\   
		\hline
		25 epochs & 0.101 & 0.167 & 0.336 & 0.586 \\ 
		50 epochs & 0.083 & 0.118 & 0.254 & 0.609 \\ 
		100 epochs & 0.082 & 0.095 & 0.124 & 0.438 \\ 
		\hline
	\end{tabular}
\end{table}

\begin{figure}[h!]
	\begin{center}
		\includegraphics[width=.8\textwidth]{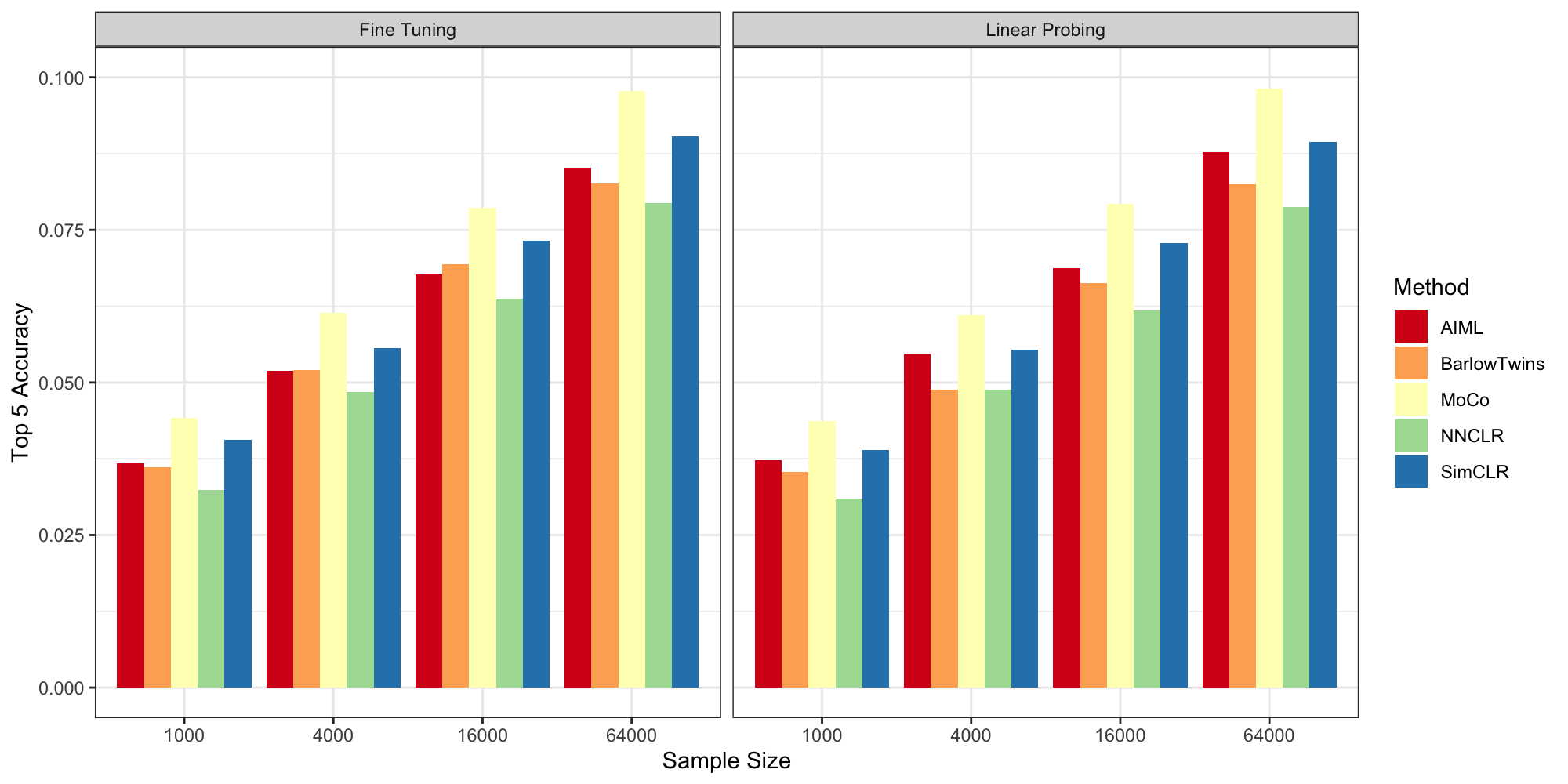}
	\end{center}
	\caption{Comparisons of five self-supervised learning methods on ImageNet data set. The left figure summarizes the top-5 accuracy of the downstream classifier trained by fine-tuning, and the right figure shows the top-5 accuracy of the downstream classifier trained by linear probing. }
	\label{fg:comparison_imagenet}
\end{figure}

\begin{figure}[h!]
	\begin{center}
		\includegraphics[width=.45\textwidth]{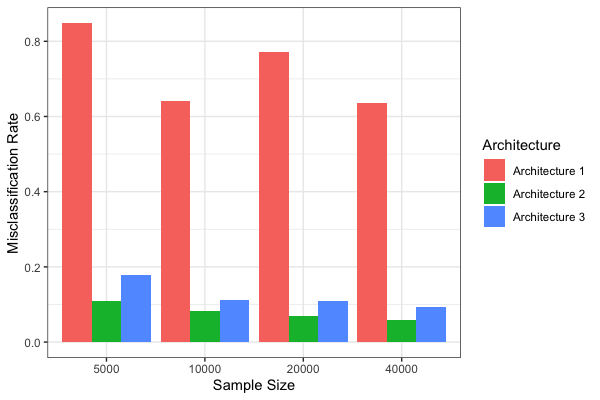}
		\hskip 20pt
		\includegraphics[width=.45\textwidth]{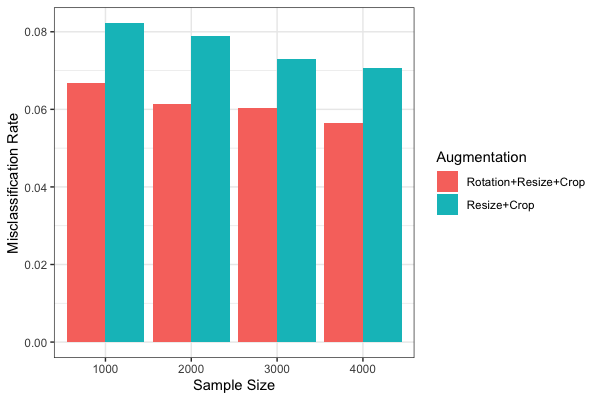}
	\end{center}
	\caption{Effect of the encoder's architecture and data augmentation on augmentation invariant manifold learning. The left figure summarizes the misclassification rate of the downstream $k$-NN classifier when the encoder's architectures are different. The right figure compares the misclassification rate of the downstream $k$-NN classifier when different data augmentation transformations are applied.}
	\label{fg:tuning}
\end{figure}

\begin{table}[h!]
	\caption{Effect of $(\lambda_1,\lambda_2)$ on the misclassification rate of downstream $k$-NN classifier.}\label{tb:lambda}
	\centering
	\begin{tabular}{cccccc}  
		\hline 
		& $\lambda_1=0.1$ & $\lambda_1=1$ & $\lambda_1=10$ & $\lambda_1=100$ & $\lambda_1=1000$ \\   
		\hline
		$\lambda_2=0.1$ & 0.870 & 0.897 & 0.903 & 0.899 & 0.892 \\ 
		$\lambda_2=1$ & 0.081 & 0.216 & 0.906 & 0.901 & 0.901 \\  
		$\lambda_2=10$ & 0.073 & 0.091 & 0.243 & 0.873 & 0.897 \\ 
		$\lambda_2=100$ & 0.089 & 0.079 & 0.097 & 0.231 & 0.900 \\ 
		$\lambda_2=1000$ & 0.085 & 0.082 & 0.090 & 0.076 & 0.432 \\ 
		\hline
	\end{tabular}
\end{table}

\section{Notions in Differential Geometry}

There are some notions of differential geometry used in this paper. See \cite{lee2012smooth} and \cite{tu2010introduction} for detailed explanations.
\begin{itemize}
	\item Injectivity radius of a manifold $\Mcal$: The injectivity radius at a point $p\in \Mcal$ is the largest radius of a ball for which the exponential map at $p$ is a diffeomorphism. The injectivity radius is the infimum of the injectivity radii at all points of $\Mcal$.
	\item Reach of a manifold $\Mcal$: The reach $R$ of $\Mcal$ is defined as
	\begin{align*}
		R:=\inf_{x\in \Mcal} d(x,{\rm Med}(\Mcal)),
	\end{align*}
	where ${\rm Med}(\Mcal)=\{z\in \RR^D: \exists p\ne q\in \Mcal,\|p-z\|=\|q-z\|=d(x,\Mcal)\}$
	\item Volume of a Riemannian manifold $(\Mcal,g)$: The volume of $\Mcal$ is defined as
	$$
	{\rm Vol}\Mcal=\int_{\Mcal} \sqrt{{\rm det}(g_{ij})}dx_1\wedge\ldots \wedge  dx_d,
	$$
	where $(x_1,\ldots,x_d)$ is a set of local coordinates.
	\item Ricci curvature on a Riemannian manifold $(\Mcal,g)$: Given $p\in \Mcal$ and $x,y\in T_p\Mcal$, the Ricci curvature is defined as
	$$
	{\rm Ric}_p(x,y)=\sum_{j=1}^d\langle R_p(x,e_j)y,e_j \rangle_p,
	$$
	where $e_1,\ldots,e_d$ is some orthonormal basis of $T_p\Mcal$ and $R_p(x,z)y$ is the Riemannian curvature of $\Mcal$ defined by Levi-Civita connection. Intuitively, it characterizes how geometry is different from ordinary Euclidean space. 
	\item Geodesic distance on $\Mcal$: The geodesic distance on $\Mcal$ is defined as
	$$
	d_\Mcal(x,y)=\inf\{L(\gamma):\gamma(0)=x, \gamma(1)=y\},
	$$
	where $L(\gamma)=\int_0^1 \sqrt{\langle \gamma'(t),\gamma'(t)\rangle_p}dt$ and $\gamma$ is a curve on $\Mcal$.
\end{itemize}

\section{Proof}

We shall use the capital letters $C$ and $c$ to denote some generic constant relying on the properties of manifold, which can be different at different places.

\subsection{Proof of Theorem~\ref{thm:aiml}}
The integrated weight $W_n(\phi_i,\phi_{i'})$ can be seen a noisy version of the following weight, that is a kernel defined on $\Ncal_s$
$$
\tilde{h}_t(\phi_1,\phi_2)={1\over {\rm Vol}^2\Ncal_v}\int_{\Mcal(\phi_1)}\int_{\Mcal(\phi_2)}\exp\left(-{\|x-y\|^2\over t}\right)dxdy,
$$
where $\Mcal(\phi)=\{x:x=T(\phi,\psi), \forall \ \psi\in \Ncal_v\}$. We define a continuous version of operator $L_{m,n}^tg(\phi)$
$$
L_{f_s}^tg(\phi_1)= {1\over t} \int_{\Ncal_s} {\tilde{h}_t(\phi_1,\phi_2) \over \sqrt{D^t(\phi_1)}\sqrt{D^t(\phi_2)}}\left(g(\phi_1)-g(\phi_2)\right)f_s(\phi_2)d\phi_2,
$$
where $D^t(\phi)$ is defined as
$$
D^t(\phi)=\int_{\Ncal_s} \tilde{h}_t(\phi,\phi')f_s(\phi')d\phi'.
$$
The rest of proof is divided into two steps: 1) we show $L_{f_s}^t$ converges to $\Lcal_{\Ncal_s,f_s}$ when $t\to 0$; 2) apply concentration inequality to show $L_{m,n}^t$ converges to $L_{f_s}^t$. 

\paragraph{Step 1}
For a function $g(\phi)$ defined on $\Ncal_s$, we define 
$$
G_t g(\phi)={1\over t^{d/2}}\int_{\Ncal_s}\tilde{h}_t(\phi,\phi')g(\phi')d\phi'.
$$
Clearly, the definition suggests $t^{-d/2}D^t(\phi)=G_t f_s(\phi)$ so Lemma~\ref{lm:convolution} leads to
$$
t^{-d/2}D^t(\phi)=m_0f_s(\phi)-t{m_2\over 2}\Lcal_{\Ncal_s}f_s(\phi)+O(t^2).
$$
Consequently, we have 
$$
{1\over \sqrt{t^{-d/2}D^t(\phi)}}={1\over \sqrt{m_0f_s(\phi)}}\left(1+t{m_2\over 4m_0}{\Lcal_{\Ncal_s}f_s(\phi)\over f_s(\phi)}+O(t^2)\right).
$$
Plugin the above to the definition of $L_{f_s}^tg(\phi)$ and we obtain
\begin{align*}
	L_{f_s}^tg(\phi_1)&= {1\over t} \int_{\Ncal_s} {\tilde{h}_t(\phi_1,\phi_2)/t^{d/2} \over m_0 \sqrt{f_s(\phi_1)}\sqrt{f_s(\phi_2)}}\left(g(\phi_1)-g(\phi_2)\right)f_s(\phi_2)(1+O(t))d\phi_2\\
	&={1+O(t)\over t} {1\over m_0\sqrt{f_s(\phi_1)}}G_t r(\phi_1)\\
	&={1+O(t)\over t} {1\over m_0\sqrt{f_s(\phi_1)}}\left(m_0r(\phi_1)-t{m_2\over 2}\Lcal_{\Ncal_s}r(\phi_1)+O(t^2)\right)\\
	&=\left(1+O(t)\right){1\over m_0\sqrt{f_s(\phi_1)}}\left(-{m_2\over 2}\left(-\sqrt{f_s(\phi_1)}\Lcal_{\Ncal_s}g(\phi_1)-2\langle \nabla g,\nabla \sqrt{f_s} \rangle\right)+O(t)\right)\\
	&={1\over 2}\left(\Lcal_{\Ncal_s}g(\phi_1)+{1\over f_s(\phi_1)}\langle \nabla g,\nabla f_s \rangle\right)+O(t)
\end{align*}
Here, we define $r(\phi)=\left(g(\phi_1)-g(\phi)\right)\sqrt{f_s(\phi)}$, use the fact $\Lcal(fg)=f\Lcal g+g\Lcal f + 2\langle \nabla f,\nabla g \rangle$ and apply Lemma~\ref{lm:convolution} again. So we complete the proof for $L_{f_s}^tg(\phi)=\Lcal_{\Ncal_s,f_s}g(\phi)/2+O(t)$.

\paragraph{Step 2}
We define an intermediate operator
$$
\tilde{L}_{m,n}^{t}g(\phi)={1\over m}\sum_{i=1}^m {1\over t} {W_n(\phi_i,\phi)\over \sqrt{D^t(\phi)}\sqrt{D^t(\phi_i)}}\left(g(\phi)-g(\phi_i)\right).
$$
We write $\bX_i=(X_{i,1},\ldots, X_{i,n})$ and know that there exist a constant $C_g$ such that
$$
{1\over t}{W_n(\phi_i,\phi)\over \sqrt{D^t(\phi)}\sqrt{D^t(\phi_i)}}\left(g(\phi)-g(\phi_i)\right)\le {C_g \over t^{(d+2)/2}}.
$$
Thus, an application of Hoeffding's inequality suggests 
$$
\PP\left(|\tilde{L}_{m,n}^{t}g(\phi)-L_{f_s}^{t}g(\phi)|>\epsilon\right)\le 2\exp\left(-c_gmt^{d+2}\epsilon^2\right),
$$
where $c_g$ is some constant depending on $C_g$. If we define the event
$$
\Acal_1=\left\{|\tilde{L}_{m,n}^{t}g(\phi)-L_{f_s}^{t}g(\phi)|>{r_m\over\sqrt{c_gmt^{d+2}}}\right\},
$$
we can know $\PP(\Acal_1)\to 0$ if $r_m\to \infty$. For any given $\phi$, we apply Hoeffding's inequality to have 
$$
\PP\left(\left|D_{m,n}^t(\phi)-D^t(\phi)\right|>\epsilon\right)\le 2\exp\left(-2mt^d\epsilon^2\right).
$$
We apply union bound to yield
$$
\PP\left(\sup_{1\le i\le m}\left|D_{m,n}^t(\phi_i)-D^t(\phi_i)\right|>\epsilon\right)\le 2m\exp\left(-2mt^d\epsilon^2\right).
$$
We define another event
$$
\Acal_2=\left\{\sup_{1\le i\le m}\left|D_{m,n}^t(\phi_i)-D^t(\phi_i)\right|>{\log m\over \sqrt{2mt^d}}\right\}
$$
and know $\PP(\Acal_2)\to 0$. Conditioned on the event $\Acal_1^c\cap \Acal_2^c$, we can have 
\begin{align*}
	|L_{m,n}^{t}g(\phi)-L_{f_s}^{t}g(\phi)|&\le |L_{m,n}^{t}g(\phi)-\tilde{L}_{m,n}^{t}g(\phi)|+|\tilde{L}_{m,n}^{t}g(\phi)-L_{f_s}^{t}g(\phi)|\\
	&\le {r_m\over\sqrt{c_gmt^{d+2}}}+{\log m\over \sqrt{2mt^d}}
\end{align*}
Since we take $t=m^{-1/(d+4)}$, we can know that 
$$
\PP\left(|L_{m,n}^{t}g(\phi)-L_{f_s}^{t}g(\phi)|>C{r_m\over\sqrt{mt^{d+2}}}\right)\to0.
$$
We complete the proof.

\subsection{Proof of Theorem~\ref{thm:convergence}}
\subsubsection{Proof for $\hat{h}_X$}
We write $\hat{\gamma}(x)=k^{-1} \sum_{i=1}^k\bI(Y_{(i)}=1)$ and $\hat{\gamma}^\ast(x)=k^{-1} \sum_{i=1}^k\gamma(X_{(i)})$. If we apply Hoeffding's inequality, then we know that
$$
\PP\left(\left|\hat{\gamma}(x)-\hat{\gamma}^\ast(x)\right|>t\right)=\PP\left({1\over k}\left| \sum_{i=1}^k\left(\bI(Y_{(i)}=1)-\gamma(X_{(i)})\right)\right|>t\right)\le 2e^{-2kt^2}.
$$
We write the ball defined by Euclidean distance as $\Bcal(x,r,\|\cdot\|)=\{y\in \Mcal:\|x-y\|\le r\}$ and $r_{a}$ as the radius such that $\mu(\Bcal(x,r_a,\|\cdot\|))=a$. Here $\mu$ is the probability of $X$ on $\Mcal$. By Chernoff bound, we have
$$
\PP\left(\|X_{(k+1)}-x\|>r_{2k/s}\right)=\PP\left(\sum_{i=1}^s\bI(X_{i}\in \Bcal(x,r_{2k/s},\|\cdot\|) )\le k\right)\le e^{-k/6}.
$$
Next, we can work on the largest difference $\sup_{y\in \Bcal(x,r_{2k/s},\|\cdot\|)}|\gamma(y)-\gamma(x)|$. Because $\|x-y\|\le d_\Mcal(x,y)$, we can know that 
$$
\Bcal(x,r_{2k/s},d_\Mcal)\subset \Bcal(x,r_{2k/s},\|\cdot\|),
$$
where $d_\Mcal$ is geodesic distance on the manifold $\Mcal$ and $\Bcal(x,r,d_\Mcal)=\{y\in \Mcal:d_\Mcal(x,y)\le r\}$. By the Bishop-Gromov inequality \citep{petersen2006riemannian}, we can know that 
$$
{\rm Vol}\Bcal(x,r_{2k/s},d_\Mcal)\ge w_d r_{2k/s}^d,
$$
where ${\rm Vol}$ is the volume on the manifold $\Mcal$ and $w_dr^d$ is the volume of Euclidean ball of radius $r$. Since the probability density function is lower bounded, we can know that 
$$
w_d r_{2k/s}^d\le {\rm Vol}\Bcal(x,r_{2k/s},\|\cdot\|)\le {1\over f_{\rm min}} \mu(\Bcal(x,r_{2k/s},\|\cdot\|))={2k\over f_{\rm min}s}.
$$
Thus, we can conclude that $r_{2k/s}\le (2/w_d f_{\rm min})^{1/d}(k/s)^{1/d}$. By Lemma 4.3 in \cite{belkin2008towards}, we can know that $d_\Mcal(x,y)\le r_{2k/s}(1+Cr_{2k/s})$ if $y\in \Bcal(x,r_{2k/s},\|\cdot\|)$, which immediately leads to $d_\Mcal(x,y)\le (4/w_d f_{\rm min})^{1/d}(k/s)^{1/d}$ when $y\in \Bcal(x,r_{2k/s},\|\cdot\|)$ and $k/s$ is small enough. Since $\Mcal=T(\Ncal_s\times \Ncal_v)$ where $T$ is an isometry, we have
$$
|\gamma(y)-\gamma(x)|=|\tilde{\gamma}(\phi_y)-\tilde{\gamma}(\phi_x)|\le Ld_{\Ncal_s}(\phi_x,\phi_y)^\alpha\le Ld_{\Mcal}(x,y)^\alpha\le L(4/w_d f_{\rm min})^{\alpha/d}(k/s)^{\alpha/d},
$$
where $y\in \Bcal(x,r_{2k/s},\|\cdot\|)$, $x\in \Mcal(\phi_x)$, and $y\in \Mcal(\phi_y)$. Therefore, we can know 
$$
\PP\left(\left|\hat{\gamma}^\ast(x)-\gamma(x)\right|>L\left(4\over w_d f_{\rm min}\right)^{\alpha/d}\left({k\over s}\right)^{\alpha/d}\right)\le e^{-k/6}.
$$ 
Putting $|\hat{\gamma}(x)-\hat{\gamma}^\ast(x)|$ and $|\hat{\gamma}^\ast(x)-\gamma(x)|$ together yields 
$$
\PP\left(\left|\hat{\gamma}(x)-\gamma(x)\right|>t+L\left(4\over w_d f_{\rm min}\right)^{\alpha/d}\left({k\over s}\right)^{\alpha/d}\right)\le 2e^{-2kt^2}+e^{-k/6}.
$$
If $2L\left(4/ w_d f_{\rm min}\right)^{\alpha/d}\left({k/s}\right)^{\alpha/d}\le t\le 1$, we can also have 
$$
\PP\left(\left|\hat{\gamma}(x)-\gamma(x)\right|>t\right)\le 3e^{-kt^2/2}.
$$

Next, we can follow a similar argument in \cite{wang2022self}. Let $\Delta=2L\left(4/ w_d f_{\rm min}\right)^{\alpha/d}\left({k/s}\right)^{\alpha/d}$, $A_0=\{x:0<|\eta(x)-1/2|<\Delta\}$ and $A_j=\{x:2^{j-1}\Delta<|\eta(x)-1/2|<2^j\Delta\}$ for $j=1,\ldots, J:=\lceil -\log(\Delta)/\log 2\rceil$. Write
\begin{align*}
	r(\hat{h}_X)&=\EE\left(|2\gamma(X)-1|\bI(\hat{h}_{X}(X)\ne h^\ast(X))\right)\\
	&=\sum_{j=0}^J\EE\left(|2\gamma(X)-1|\bI(\hat{h}_{X}(X)\ne h^\ast(X))\bI(X\in A_j)\right)
\end{align*}
We can work on each of above terms separately. Because $\bI(\hat{h}_{X}(X)\ne h^\ast(X))\le \bI(|\gamma(X)-1/2|<\left|\hat{\gamma}(X)-\gamma(X)\right|)$ and $\beta$-marginal assumption, so we can know 
\begin{align*}
	&\EE\left(|2\gamma(X)-1|\bI(\hat{h}_{X}(X)\ne h^\ast(X))\bI(X\in A_j)\right)\\
	\le & 2^{j+1}\Delta\EE\left(\bI(|\gamma(X)-1/2|<\left|\hat{\gamma}(X)-\gamma(X)\right|)\bI(X\in A_j)\right)\\
	\le & 2^{j+1}\Delta\EE\left(\bI(2^{j-1}\Delta<\left|\hat{\gamma}(X)-\gamma(X)\right|)\bI(X\in A_j)\right)\\
	\le & 2^{j+1}\Delta\EE\left(\PP(2^{j-1}\Delta<\left|\hat{\gamma}(X)-\gamma(X)\right|)\bI(X\in A_j)\right)\\
	\le & 2^{j+1}\Delta\EE\left((3e^{-k(2^{j-1}\Delta)^2/2})\bI(X\in A_j)\right)\\
	\le & 2^{j+1}\Delta 3e^{-k(2^{j-1}\Delta)^2/2} C_0(2^{j+1}\Delta)^\beta.
\end{align*}
Putting these terms together, we can know that 
$$
r(\hat{h}_X)\le 3C_0\Delta^{\beta+1}\sum_{j=0}^J 2^{(j+1)(\beta+1)} e^{-k(2^{j-1}\Delta)^2/2}\le \tilde{C} s^{-\alpha(\beta+1)/(2\alpha+d)}.
$$
where 
$$
\tilde{C}=3C_0 (2L)^{\beta+1}\left(4\kappa\over  w_d\right)^{\alpha(\beta+1)/d}c_0^{\alpha(\beta+1)/d}.
$$
\subsubsection{Proof for $\hat{h}_{\Theta_2(X)}$}	
Next, we work on $\hat{h}_{\Theta_2(X)}$. Recall the definition of $\Theta_2(x)$ and define $\tilde{\Theta}_2(\phi)$
$$
\tilde{\Theta}_2(\phi)=(2l)^{(d+2)/4}\sqrt{2}(4\pi)^{d/4}(e^{-l\lambda_1}\eta_1(\phi),\ldots,e^{-l\lambda_N}\eta_N(\phi)).
$$
We also write $\tilde{\Theta}_2(x)=\tilde{\Theta}_2(\phi)$ when $x\in \Mcal(\phi)$. Since $\tilde{\Theta}_2(x)$ is just scale version of $\Theta_2(x)$, it is sufficient to work on $\hat{h}_{\tilde{\Theta}_2(X)}$. If we choose $l$ and $N$ appropriately, the map $\tilde{\Theta}_2(\phi)$ defined above is almost an isometry. Specifically, the Theorem 5.1 in \cite{portegies2016embeddings} suggests there exists $l_0$ and $N_0$ such that if we choose $l=l_0$ and $N>N_0$ in $\tilde{\Theta}_2(\phi)$, then $\tilde{\Theta}_2(\phi)$ is an embedding of $\Ncal_s$ to $\RR^N$ and
$$
\sup_{v\in T_\phi \Mcal,\|v\|=1}\left|\left\|(D\tilde{\Theta}_2)_\phi(v)\right\|^2-1\right|<{1\over 4},\qquad \forall \phi\in \Ncal_s,
$$
where $(D\tilde{\Theta}_2)_\phi: T_\phi\Ncal_s \to T_{\tilde{\Theta}_2(\phi)} \Ncal'_s$ is the derivative of $\tilde{\Theta}_2(\phi)$ at $\phi$. Here, $\Ncal'_s=\tilde{\Theta}_2(\Ncal_s)$ is a manifold in $\RR^N$. For any $\phi_1,\phi_2\in \Ncal_s$ and a curve  connecting $\tilde{\Theta}_2(\phi_1)$ and $\tilde{\Theta}_2(\phi_2)$, $\chi:[0,1]\to \Ncal'_s$, we can know that $\tilde{\Theta}_2^{-1}\circ\chi$ is a curve in $\Ncal_s$. Note that
$$
d_{\Ncal_s}(\phi_1,\phi_2)\le \int_0^1\|D (\tilde{\Theta}_2^{-1}\circ\chi)(a)\|da\le 2\int_0^1\|D \chi(a)\|da.
$$
So we can conclude $d_{\Ncal_s}(\phi_1,\phi_2)\le 2d_{\Ncal'_s}(\tilde{\Theta}_2(\phi_1),\tilde{\Theta}_2(\phi_2))$ if taking the infimum of all possible curves $\chi$. Similarly, For any $\phi_1,\phi_2\in \Ncal_s$ and a curve connecting $\phi_1$ and $\phi_2$, $\chi:[0,1]\to \Ncal_s$, we can know that $\tilde{\Theta}_2\circ\chi$ is a curve in $\Ncal'_s$. Note that 
$$
d_{\Ncal'_s}(\tilde{\Theta}_2(\phi_1),\tilde{\Theta}_2(\phi_2))\le \int_0^1 \left\|D(\tilde{\Theta}_2\circ\chi)(a)\right\|da\le 2\int_0^1 \left\|D\chi(a)\right\|da.
$$
So we know $d_{\Ncal'_s}(\tilde{\Theta}_2(\phi_1),\tilde{\Theta}_2(\phi_2))\le 2 d_{\Ncal_s}(\phi_1,\phi_2)$. Therefore, we can conclude
$$
{1\over 2}d_{\Ncal_s}(\phi_1,\phi_2) \le d_{\Ncal'_s}(\tilde{\Theta}_2(\phi_1),\tilde{\Theta}_2(\phi_2))\le 2 d_{\Ncal_s}(\phi_1,\phi_2).
$$

If we adopt $\tilde{\Theta}_2(x)$, the main difference in proof for convergence rate of $k$NN is the shape of neighborhood. Instead of $\Bcal(x,r,\|\cdot\|)$, we consider the following small ball defined by $\tilde{\Theta}_2(x)$
$$
\Bcal(x,r,\|\cdot\|_{\tilde{\Theta}_2})=\{y\in \Mcal:\|\tilde{\Theta}_2(x)-\tilde{\Theta}_2(y)\|\le r\}.
$$
Similarly, we can define $r_a$ as the radius such that $\mu(\Bcal(x,r_a,\|\cdot\|_{\tilde{\Theta}_2}))=a$. It is sufficient to characterize $\sup_{y\in \Bcal(x,r_{2k/s},\|\cdot\|_{\tilde{\Theta}_2})}|\gamma(y)-\gamma(x)|$ as the rest of proof is the same with the case of $\hat{h}_{X}$. As $\|\tilde{\Theta}_2(x)-\tilde{\Theta}_2(y)\|\le d_{\Ncal'_s}(\tilde{\Theta}_2(x),\tilde{\Theta}_2(y))\le 2d_{\Ncal_s}(\phi_x,\phi_y)$ where $x\in \Mcal(\phi_x)$ and $y\in\Mcal(\phi_y)$, we have
$$
\Bcal(x,r/2,d_{\Ncal_s})=\left\{y\in \Mcal:d_{\Ncal_s}(\phi_x,\phi_y)\le r/2 \right\}\subset \Bcal(x,r,\|\cdot\|_{\tilde{\Theta}_2}).
$$
We can apply the Bishop-Gromov inequality and the fact $\Mcal=T(\Ncal_s\times \Ncal_v)$ to have
$$
{\rm Vol}\Bcal(x,r/2,d_{\Ncal_s})={\rm Vol}\Ncal_v\times {\rm Vol}\left\{\phi_y\in \Ncal_s:d_{\Ncal_s}(\phi_x,\phi_y)\le r/2 \right\}\ge {\rm Vol}\Ncal_v w_{d_s} (r/2)^{d_s},
$$
which immediately leads to
$$
f_{\rm min}{\rm Vol}\Ncal_v w_{d_s} r_{2k/s}^{d_s}/2^{d_s}\le {2k\over s}.
$$
Thus, we can conclude $r_{2k/s}\le (2^{d_s+1}/f_{\rm min}{\rm Vol}\Ncal_v w_{d_s})^{1/d_s}(k/s)^{1/d_s}$. If $\|\tilde{\Theta}_2(x)-\tilde{\Theta}_2(y)\|\le r_{2k/s}$, then, for large enough $s$, 
$$
d_{\Ncal_s}(\phi_x,\phi_y)\le 2d_{\Ncal'_s}(\tilde{\Theta}_2(\phi_x),\tilde{\Theta}_2(\phi_y))\le 2r_{2k/s}(1+Cr_{2k/s})\le 4\left(2^{d_s+1}\over f_{\rm min}{\rm Vol}\Ncal_v w_{d_s}\right)^{1/d_s}(k/s)^{1/d_s}.
$$
This leads to that if $\|\tilde{\Theta}_2(x)-\tilde{\Theta}_2(y)\|\le r_{2k/s}$, then
$$
|\gamma(y)-\gamma(x)|=|\tilde{\gamma}(\phi_y)-\tilde{\gamma}(\phi_x)|\le Ld_{\Ncal_s}(\phi_x,\phi_y)^\alpha\le L\left(2^{3d_s+1}\over f_{\rm min}{\rm Vol}\Ncal_v w_{d_s}\right)^{\alpha/d_s}(k/s)^{\alpha/d_s}.
$$
An application of the same arguments in te proof for $\hat{h}_{X}$ suggests
$$
r(\hat{h}_{\Theta_2(X)})\le \tilde{C}' s^{-\alpha(\beta+1)/(2\alpha+d_s)}.
$$
where 
$$
\tilde{C}'=3C_0 (2L)^{\beta+1}\left(2^{3d_s+1}\kappa \over {\rm Vol}\Ncal_v w_{d_s}\right)^{\alpha(\beta+1)/d_s}c_0^{\alpha(\beta+1)/d_s}.
$$
\subsubsection{Proof for $\hat{h}_{\Theta_1(X)}$}	
Finally, we work on $\hat{h}_{\Theta_1(X)}$. Note that
$$
\|\tilde{\Theta}_2(x)-\tilde{\Theta}_2(y)\|\le \|\tilde{\Theta}_1(x)-\tilde{\Theta}_1(y)\|\le e^{l_0\lambda_{N_0}}	\|\tilde{\Theta}_2(x)-\tilde{\Theta}_2(y)\|.
$$
We can apply the similar arguments to obtain 
$$
r(\hat{h}_{\Theta_1(X)})\le \tilde{C}'\left(e^{-d_s l_0\lambda_{N_0}}s\right)^{-\alpha(\beta+1)/(2\alpha+d_s)}.
$$

\subsection{Proof for Theorem~\ref{thm:finiteconvergence}}

\subsubsection{Proof for $\hat{h}_{\hat{\Theta}_2(X)}$}
When we adopt a different data representation in $k$-NN, the main difference in the proof is the shape of neighborhood. Specifically, we need to consider the following neighborhood
$$
\Bcal(x,r,\|\cdot\|_{\hat{\Theta}_2})=\{y\in \Mcal:\|\hat{\Theta}_2(x)-\hat{\Theta}_2(y)\|\le r\}.
$$
We define $r_a$ as the radius such that $\mu(\Bcal(x,r_a,\|\cdot\|_{\hat{\Theta}_2}))=a$.
By Theorem~\ref{thm:convergencerate}, we can know that 
$$
{1\over \kappa}\|\Theta_2(x)-\Theta_2(y)\|-{1\over \kappa}\sqrt{N_0}\epsilon_{m,n}\le \|\hat{\Theta}_2(x)-\hat{\Theta}_2(y)\| \le \kappa\|\Theta_2(x)-\Theta_2(y)\|+\kappa\sqrt{N_0}\epsilon_{m,n}
$$
where 
$$
\epsilon_{m,n}=\tilde{C}''\left(\left({\log m\over m}\right)^{1/(4d+13)}+\left({\log m\over n}\right)^{1/(2d+5)}\right).
$$ 
Therefore, we can know that 
$$
\Bcal\left(x,{r-\kappa\sqrt{N_0}\epsilon_{m,n}\over \kappa},\|\cdot\|_{\Theta_2}\right)\subset \Bcal(x,r,\|\cdot\|_{\hat{\Theta}_2})
$$
Following the same arguments in the proof of Theorem~\ref{thm:convergence}, we have 
$$
f_{\rm min}{\rm Vol}\Ncal_v w_{d_s} \left(r_{2k/s}-\kappa\sqrt{N_0}\epsilon_{m,n}\over \kappa\right)^{d_s}/2^{d_s}\le {2k\over s},
$$
which leads to
$$
r_{2k/s}\le \kappa\left[\sqrt{N_0}\epsilon_{m,n}+\left(2^{d_s+1}\over f_{\rm min}{\rm Vol}\Ncal_v w_{d_s}\right)^{1/d_s}(k/s)^{1/d_s}\right].
$$
Hence, we can conclude that when $y\in \Bcal(x,r_{2k/s},\|\cdot\|_{\hat{\Theta}_2})$, 
$$
d_{\Ncal_s}(\phi_x,\phi_y)\le 4\kappa\left[\sqrt{N_0}\epsilon_{m,n}+\left(2^{d_s+1}\over f_{\rm min}{\rm Vol}\Ncal_v w_{d_s}\right)^{1/d_s}(k/s)^{1/d_s}\right].
$$
where $x\in\Mcal(\phi_x)$ and $y\in\Mcal(\phi_y)$. This immediately suggests
\begin{align*}
	r(\hat{h}_{\hat{\Theta}_2(X)})&\le 3C_0\left(2L\kappa^\alpha\left[\left(2^{3d_s+1}\over f_{\rm min}{\rm Vol}\Ncal_v w_{d_s}\right)^{1/d_s}(k/s)^{1/d_s}+4\sqrt{N_0}\epsilon_{m,n}\right]^\alpha\right)^{\beta+1}\\
	&\le 2\tilde{C}'\kappa^{\alpha(\beta+1)}s^{-\alpha(\beta+1)/(2\alpha+d_s)}+\tilde{C}'''\epsilon_{m,n}^{\alpha(\beta+1)},
\end{align*}
where 
$$
\tilde{C}'''=6C_0(2L)^{\beta+1}(4\kappa)^{\alpha(\beta+1)}N_0^{\alpha(\beta+1)/2}.
$$

\subsubsection{Proof for $\hat{h}_{\hat{\Theta}_1(X)}$}
We can have the similar conclusion if we note that 
$$
{1\over \kappa}\|\Theta_1(x)-\Theta_1(y)\|-{1\over \kappa}\sqrt{N_0}\epsilon_{m,n}\le \|\hat{\Theta}_1(x)-\hat{\Theta}_1(y)\|\le \kappa\|\Theta_1(x)-\Theta_1(y)\|+\kappa\sqrt{N_0}\epsilon_{m,n}.
$$

\subsection{Proof for Theorem~\ref{thm:convergencerate}}
The main idea in the proof is to apply a similar strategy in \cite{dunson2021spectral}, but we will point out the difference and provide the new gradients necessary for the new proof. See also \cite{von2008consistency}. The main difference is due to the facts that we use a randomized kernel and we do not normalize the eigenvectors. We divided the proof in five steps. Recall the definition
$$
\tilde{h}_t^{(1)}(\phi_1,\phi_2)={\tilde{h}_t(\phi_1,\phi_2) \over f_{s,t}(\phi_1)f_{s,t}(\phi_2)},\qquad {\rm where}\quad f_{s,t}(\phi_1)=\int \tilde{h}_t(\phi_1,\phi_2) f_s(\phi_2)d\phi_2
$$
and 
$$
\Pcal_{t,1} g(\phi_1)=\int {\tilde{h}_t^{(1)}(\phi_1,\phi_2)\over D_t^{(1)}(\phi_1)}g(\phi_2)f_s(\phi_2)d\phi_2,\qquad {\rm where}\quad D_t^{(1)}(\phi_1)=\int \tilde{h}_t^{(1)}(\phi_1,\phi_2) f_s(\phi_2)d\phi_2.
$$
Similarly, we define the corresponding empirical version
$$
\tilde{h}_{m,t}^{(1)}(\phi,\phi')={\tilde{h}_t(\phi,\phi') \over f_{s,m,t}(\phi)f_{s,m,t}(\phi')},\qquad {\rm where}\quad f_{s,m,t}(\phi)={1\over m}\sum_{i=1}^m \tilde{h}_t(\phi,\phi_i)
$$
and 
$$
\Pcal_{m,t,1} g(\phi)={1\over m}\sum_{i=1}^m {\tilde{h}_{m,t}^{(1)}(\phi,\phi_i)\over D_{m,t}^{(1)}(\phi)}g(\phi_i),\qquad {\rm where}\quad D_{m,t}^{(1)}(\phi)={1\over m}\sum_{i=1}^m \tilde{h}_{m,t}^{(1)}(\phi,\phi_i).
$$
The discrete version operator defined by $\tilde{h}_t(\phi,\phi')$ can be written as
$$
\tilde{W}^{(1)}=\tilde{D}^{-1}\tilde{W}\tilde{D}^{-1},\qquad {\rm where}\quad \tilde{D}_{i,i}={1\over m}\sum_j\tilde{W}_{i,j}
$$
and
$$
\tilde{P}^{(1)}=(\tilde{D}^{(1)})^{-1}\tilde{W}^{(1)},\qquad {\rm where}\quad \tilde{D}^{(1)}_{i,i}=\sum_j\tilde{W}^{(1)}_{i,j}
$$
Different from $P^{(1)}$, the weight is defined as $\tilde{W}_{i,j}=\tilde{h}_t(\phi_i,\phi_j)$ in $\tilde{P}^{(1)}$. Recall $P^{(1)}$ can be written as the following equivalent form
$$
W^{(1)}=D^{-1}WD^{-1},\qquad {\rm where}\quad D_{i,i}={1\over m}\sum_jW_{i,j}
$$
and
$$
P^{(1)}=(D^{(1)})^{-1}W^{(1)},\qquad {\rm where}\quad D^{(1)}_{i,i}=\sum_jW^{(1)}_{i,j}.
$$

\subsubsection{Step 1: spectral convergence of $(I-\Pcal_{t,1})/t$ to $\Lcal_{\Ncal_s}$}
To show the convergence, we define the heat kernel based operator $\Hcal_t$ on $\Ncal_s$
$$
\Hcal_t g(\phi_1)=\int H_t(\phi_1,\phi_2)g(\phi_2)d\phi_2,\qquad {\rm where}\quad H_t(\phi_1,\phi_2)=\sum_{l=0}^\infty e^{-\lambda_lt}\eta_l(\phi_1)\eta_l(\phi_2),
$$
where $\lambda_1,\ldots$ and $\eta(\phi),\ldots$ are the eigenvalues and eigenvectors of $\Lcal_{\Ncal_s}$. Given $\Hcal_t$, we can also define a residue operator $\Rcal_t=(I-\Hcal_{t})/t-(I-\Pcal_{t,1})/t=(\Pcal_{t,1}-\Hcal_t)/t$. Lemma~\ref{lm:bounded} and \ref{lm:residue} characterize the properties of $\Rcal_t$ in our setting. After introducing Lemma~\ref{lm:operatorapp}, \ref{lm:bounded}, and \ref{lm:residue}, we can study the convergence of $(I-\Pcal_{t,1})/t$ in the same way as \cite{dunson2021spectral} (see its Proposition 1). Specifically, if we denote the eigenfunction and eigenvalue of $(I-\Pcal_{t,1})/t$ by $(\lambda_{l,t},\eta_{l,t})$ for $l=1,\ldots,N$, we can show that when $t$ is smaller than some constant depending on $\lambda_N$ and $\Gamma_N$, then 
$$
|\lambda_{l,t}-\lambda_l|\le t^{3/4},\qquad \|a_l\eta_{l,t}-\eta_l\|\le t^{3/4},\qquad {\rm and}\qquad \|a_l\eta_{l,t}-\eta_l\|_\infty\le t^{1/2},
$$
where $a_l\in\{-1,1\}$. 

\subsubsection{Step 2: convergence of $\Pcal_{m,t,1}$ to $\Pcal_{t,1}$}
To show the convergence of $\Pcal_{m,t,1}$, we need to study the following the process
$$
\sup_{\phi'\in \Ncal_s}\left|{1\over m}\sum_{i=1}^m\tilde{h}_t(\phi_i,\phi')-\int_{\Ncal_s}\tilde{h}_t(\phi,\phi')f_s(\phi)d\phi\right|,
$$
$$
\sup_{\phi'\in \Ncal_s}\left|{1\over m}\sum_{i=1}^mg(\phi_i)\tilde{h}^{(1)}_t(\phi_i,\phi')-\int_{\Ncal_s}g(\phi)\tilde{h}^{(1)}_t(\phi,\phi')f_s(\phi)d\phi\right|,
$$
and
$$
\sup_{\phi',\phi''\in \Ncal_s}\left|{1\over m}\sum_{i=1}^m\tilde{h}^{(1)}_t(\phi_i,\phi'')\tilde{h}^{(1)}_t(\phi_i,\phi')-\int_{\Ncal_s}\tilde{h}^{(1)}_t(\phi,\phi'')\tilde{h}^{(1)}_t(\phi,\phi')f_s(\phi)d\phi\right|.
$$
The Lemma~\ref{lm:gcclass} help characterize the convergence of the above empirical process. After having Lemma~\ref{lm:gcclass}, we can then apply the same arguments in  Proposition 3 of \cite{dunson2021spectral} to show that if $m$ is large enough so that $\left(\sqrt{-\log t}+\sqrt{\log m}\right)/\sqrt{m}t^{d/2}<C$, then 
$$
|\lambda_{l,t}-\lambda_{l,m,t}|\le {C\over \sqrt{m}t^{(2d+5)/2}}\left(\sqrt{-\log t}+\sqrt{\log m}\right)
$$
and 
$$
\|a_l\eta_{l,t}-\eta_{l,m,t}\|_\infty\le {C\over \sqrt{m}t^{(2d+3)/2}}\left(\sqrt{-\log t}+\sqrt{\log m}\right)
$$
with probability at least $1-m^{-2}$. Here $(\lambda_{l,m,t},\eta_{l,m,t})$ for $l=1,\ldots,N$ are the eigenvalues and eigenfunctions of $(I-\Pcal_{m,t,1})/t$.

\subsubsection{Step 3: connection between $\Pcal_{m,t,1}$ and $\tilde{P}^{(1)}$}

This step aims to build the connection between $\Pcal_{m,t,1}$ and $\tilde{P}^{(1)}$, that is, there is a one-to-one correspondence between the eigenpairs of $\Pcal_{m,t,1}$ and the eigenpairs of $\tilde{P}^{(1)}$. Specifically, let $(\lambda,g(\phi))$ be an eigenpair of $\Pcal_{m,t,1}$ and $\vec{g}$ be a $m$-dimensional vector such that $\vec{g}_i=g(\phi_i)$. Note that 
\begin{align*}
	[\tilde{P}^{(1)}\vec{g}]_i&={ \sum_{i'}\tilde{W}^{(1)}_{i,i'}g(\phi_{i'})\over \sum_{i'}\tilde{W}^{(1)}_{i,i'} }={\sum_{i'}\tilde{W}_{i,i'}g(\phi_{i'})/\tilde{D}_{i,i}\tilde{D}_{i',i'}\over \sum_{i'}\tilde{W}_{i,i'}/\tilde{D}_{i,i}\tilde{D}_{i',i'}}\\
	&={ \sum_{i'}\tilde{h}_t(\phi_i,\phi_{i'})g(\phi_{i'})/f_{s,m,t}(\phi_i)f_{s,m,t}(\phi_{i'})\over \sum_{i'}\tilde{h}_t(\phi_i,\phi_{i'})/f_{s,m,t}(\phi_i)f_{s,m,t}(\phi_{i'})}\\
	&={ \sum_{i'}\tilde{h}^{(1)}_{m,t}(\phi_i,\phi_{i'})g(\phi_{i'})\over \sum_{i'}\tilde{h}^{(1)}_{m,t}(\phi_i,\phi_{i'})}\\
	&=\Pcal_{m,t,1} g(\phi_i)=\lambda g(\phi_i).
\end{align*}
Therefore, $(\lambda,\vec{g})$ is an eigenpair of $\tilde{P}^{(1)}$. 

On the other hand, let $(\lambda,\vec{g})$ be an eigenpair of $\tilde{P}^{(1)}$. Given $\vec{g}$, define a function
$$
g(\phi)={1\over \lambda}{ \sum_{i'}\tilde{h}_t(\phi,\phi_{i'})\vec{g}_{i'}/f_{s,m,t}(\phi)f_{s,m,t}(\phi_{i'})\over \sum_{i'}\tilde{h}_t(\phi,\phi_{i'})/f_{s,m,t}(\phi)f_{s,m,t}(\phi_{i'})}.
$$
It is easy to check that 
$$
g(\phi_i)={1\over \lambda}{ \sum_{i'}\tilde{h}_t(\phi_i,\phi_{i'})\vec{g}_{i'}/f_{s,m,t}(\phi_i)f_{s,m,t}(\phi_{i'})\over \sum_{i'}\tilde{h}_t(\phi_i,\phi_{i'})/f_{s,m,t}(\phi_i)f_{s,m,t}(\phi_{i'})}={1\over \lambda}[\tilde{P}^{(1)}\vec{g}]_i=\vec{g}_i
$$
and 
\begin{align*}
	\Pcal_{m,t,1} g(\phi)&={ \sum_{i'}\tilde{h}_t(\phi,\phi_{i'})g(\phi_{i'})/f_{s,m,t}(\phi)f_{s,m,t}(\phi_{i'})\over \sum_{i'}\tilde{h}_t(\phi,\phi_{i'})/f_{s,m,t}(\phi)f_{s,m,t}(\phi_{i'})}\\
	&={ \sum_{i'}\tilde{h}_t(\phi,\phi_{i'})\vec{g}_{i'}/f_{s,m,t}(\phi)f_{s,m,t}(\phi_{i'})\over \sum_{i'}\tilde{h}_t(\phi,\phi_{i'})/f_{s,m,t}(\phi)f_{s,m,t}(\phi_{i'})}\\
	&=\lambda g(\phi).
\end{align*}
Therefore, $(\lambda,g(\phi))$ is an eigenpair of $\Pcal_{m,t,1}$. So we prove that each eigenpair of $\tilde{P}^{(1)}$ corresponds to each eigenpair of $\Pcal_{m,t,1}$. For each eigenpair of $(I-\Pcal_{m,t,1})/t$, i.e., $(\lambda_{l,m,t},\eta_{l,m,t})$, we write the corresponding eigenpair of $(I-\tilde{P}^{(1)})/t$ as $(\lambda_{l,m,t},\vec{\eta}_{l,m,t})$.

\subsubsection{Step 4: convergence of $P^{(1)}$ to $\tilde{P}^{(1)}$}
Similar to Step 2, the idea is to apply perturbation theory for spectral projections \citep[see, e.g.,][]{atkinson1967numerical}. Recall $\vec{\eta}_{l,m,t}$ is the eigenvector of $\tilde{P}^{(1)}$. Let $\vec{\eta}_{l,m,n,t}$ be the eigenvector of $P^{(1)}$ and ${\rm Pr}_{\vec{\eta}_{l,m,n,t}}$ be the projection on $\vec{\eta}_{l,m,n,t}$. We can apply Theorem 3 in \cite{atkinson1967numerical} or Theorem 5 in \cite{dunson2021spectral} to have
$$
\left\|\vec{\eta}_{l,m,t}-{\rm Pr}_{\vec{\eta}_{l,m,n,t}}\vec{\eta}_{l,m,t}\right\|_\infty\le 8\left(\left\|(P^{(1)}-\tilde{P}^{(1)}){\vec{\eta}_{l,m,t}\over \|\vec{\eta}_{l,m,t}\|_\infty}\right\|_\infty+{48\over \Gamma_Nt}\left\|(\tilde{P}^{(1)}-P^{(1)})P^{(1)}\right\|_\infty\right)\|\vec{\eta}_{l,m,t}\|_\infty
$$
An application of Lemma~\ref{lm:crossrate} and Lemma~\ref{lm:pointrate} suggests 
$$
\left\|\vec{\eta}_{l,m,t}-{\rm Pr}_{\vec{\eta}_{l,m,n,t}}\vec{\eta}_{l,m,t}\right\|_\infty\le {1\over t}\left({C\over t^d\sqrt{mn}}\left(\sqrt{-\log t}+\sqrt{\log m}\right)+{Ct^{2-d/2}\sqrt{\log m}\over \sqrt{n}}\right)\|\vec{\eta}_{l,m,t}\|_\infty
$$
with probability at least $1-m^{-2}$. Following the same arguments in Proposition 3 of \cite{dunson2021spectral}, we can show 
$$
|\lambda_{l,m,t}-\lambda_{l,m,n,t}|\le {C\over t^{(2d+5)/2}\sqrt{mn}}\left(\sqrt{-\log t}+\sqrt{\log m}\right)+{Ct^{2-(d+5)/2}\sqrt{\log m}\over \sqrt{n}}
$$
and 
$$
\|a_l\vec{\eta}_{l,m,t}-\vec{\eta}_{l,m,n,t}\|_\infty\le {C\over t^{(2d+3)/2}\sqrt{mn}}\left(\sqrt{-\log t}+\sqrt{\log m}\right)+{Ct^{2-(d+3)/2}\sqrt{\log m}\over \sqrt{n}}
$$
with probability at least $1-m^{-2}$. Here, $\lambda_{l,m,n,t}$ is the eigenvalue of $(I-P^{(1)})/t$. 

\subsubsection{Step 5: finishing proof}
We are ready to put the results from previous 4 steps together. Specifically, we have 
$$
|\lambda_{l}-\lambda_{l,m,n,t}|\le t^{3/4}+{C\over \sqrt{m}t^{(2d+5)/2}}\left(\sqrt{-\log t}+\sqrt{\log m}\right)+{C\sqrt{\log m}\over \sqrt{n}t^{(d+1)/2}}.
$$
Since we choose $t\asymp(\log m/m)^{2/(4d+13)}+(\log m/n)^{2/(2d+5)}$, we have
$$
|\lambda_{l}-\lambda_{l,m,n,t}|\le C\left(\left({\log m\over m}\right)^{3/(8d+26)}+\left({\log m\over n}\right)^{3/(4d+10)}\right).
$$
To establish the bound for eigenvectors, we need to be aware of the difference in normalization for eigenfunctions and eigenvectors. The eigenfunction is normalized in the sense of $L^2(\Ncal_s)$, while eigenvectors is normalized in the sense of $\ell_2$ which relies on the density $f_s(\phi)$. Because $1/\kappa<f_s(\phi)<\kappa$, we can know that there exists a normalization constant $c_{\kappa,l}$ such that $1/\kappa<c_{\kappa,l}<\kappa$ and 
$$
|a_lc_{\kappa,l}[\vec{\eta}_{l,m,n,t}]_i-\eta_l(\phi_i)|\le  t^{1/2}+{C\over \sqrt{m}t^{(2d+3)/2}}\left(\sqrt{-\log t}+\sqrt{\log m}\right)+{C\sqrt{\log m}\over \sqrt{n}t^{(d-1)/2}}.
$$
Since we choose $t\asymp(\log m/m)^{2/(4d+13)}+(\log m/n)^{2/(2d+5)}$, we have
$$
|a_lc_{\kappa,l}[\vec{\eta}_{l,m,n,t}]_i-\eta_l(\phi_i)|\le C\left(\left({\log m\over m}\right)^{1/(4d+13)}+\left({\log m\over n}\right)^{1/(2d+5)}\right).
$$
We now complete the proof. 

\subsection{Technical Lemmas and Proofs}

\begin{lemma}
	\label{lm:crossrate}
	Assume the assumptions in Lemma~\ref{lm:operatorproperty} hold. Then, we have
	$$
	\|(\tilde{P}^{(1)}-P^{(1)})P^{(1)}\|_\infty \le {C\over t^d\sqrt{mn}}\left(\sqrt{-\log t}+\sqrt{\log m}\right)+{Ct^{2-d/2}\sqrt{\log m}\over \sqrt{n}}.
	$$
	with probability at least $1-m^{-2}$. Here, the constant $C$ depends on $d$, diameter of $\Mcal$, the curvature of $\Ncal_v$, $d_v$, the injectivity radius and reach of manifold $\Mcal$, $d_s$, the diameter of $\Ncal_s$, $\kappa$ and $C^0$ norm of $f_s$.
\end{lemma}
\begin{proof}
	We introduce an intermediate operator between $P^{(1)}$ and $\tilde{P}^{(1)}$, which is defined as 
	$$
	\bar{P}^{(1)}=(\tilde{D}^{(1)})^{-1}\bar{W}^{(1)},\qquad {\rm where}\quad \bar{W}^{(1)}=\tilde{D}^{-1}W\tilde{D}^{-1}.
	$$
	By definition, we have
	\begin{align*}
		\|(\tilde{P}^{(1)}-P^{(1)})P^{(1)}\|_\infty\le& \|\tilde{P}^{(1)}P^{(1)}-\tilde{P}^{(1)}\bar{P}^{(1)}\|_\infty+\|\tilde{P}^{(1)}\bar{P}^{(1)}-\bar{P}^{(1)}\bar{P}^{(1)}\|_\infty\\
		&+\|\bar{P}^{(1)}\bar{P}^{(1)}-\bar{P}^{(1)}P^{(1)}\|_\infty+\|\bar{P}^{(1)}P^{(1)}-P^{(1)}P^{(1)}\|_\infty\\
		\le & \|P^{(1)}-\bar{P}^{(1)}\|_\infty(\|\tilde{P}^{(1)}\|_\infty+\|\bar{P}^{(1)}\|_\infty+\|P^{(1)}\|_\infty)+\|(\tilde{P}^{(1)}-\bar{P}^{(1)})\bar{P}^{(1)}\|_\infty\\
		\le & C\|P^{(1)}-\bar{P}^{(1)}\|_\infty+\|(\tilde{P}^{(1)}-\bar{P}^{(1)})\bar{P}^{(1)}\|_\infty.
	\end{align*}
	Hence, it is sufficient to bound $\|P^{(1)}-\bar{P}^{(1)}\|_\infty$ and $\|(\tilde{P}^{(1)}-\bar{P}^{(1)})\bar{P}^{(1)}\|_\infty$. 
	
	For $\|P^{(1)}-\bar{P}^{(1)}\|_\infty$, note that
	\begin{align*}
		\|P^{(1)}-\bar{P}^{(1)}\|_\infty&=\sup_{g:|g_i|\le 1}\sup_i \left|[(P^{(1)}-\bar{P}^{(1)})g]_i\right|\\
		&=\sup_{g:|g_i|\le 1}\sup_i\left|{\sum_{i'}W^{(1)}_{i,i'}g_{i'}\over \sum_{i'}W^{(1)}_{i,i'}}-{\sum_{i'}\bar{W}^{(1)}_{i,i'}g_{i'}\over \sum_{i'}\tilde{W}^{(1)}_{i,i'}}\right|\\
		&\le \sup_{g:|g_i|\le 1}\sup_i\left(\left|{\sum_{i'}W^{(1)}_{i,i'}g_{i'}\over \sum_{i'}W^{(1)}_{i,i'}}-{\sum_{i'}W^{(1)}_{i,i'}g_{i'}\over \sum_{i'}\tilde{W}^{(1)}_{i,i'}}\right|+\left|{\sum_{i'}W^{(1)}_{i,i'}g_{i'}\over \sum_{i'}\tilde{W}^{(1)}_{i,i'}}-{\sum_{i'}\bar{W}^{(1)}_{i,i'}g_{i'}\over \sum_{i'}\tilde{W}^{(1)}_{i,i'}}\right|\right)\\
		&\le \sup_{g:|g_i|\le 1}\sup_i\left(\left|\sum_{i'}W^{(1)}_{i,i'}g_{i'}\right|{\left|\sum_{i'}W^{(1)}_{i,i'}-\sum_{i'}\tilde{W}^{(1)}_{i,i'}\right|\over \sum_{i'}W^{(1)}_{i,i'}\sum_{i'}\tilde{W}^{(1)}_{i,i'}}+{\left|\sum_{i'}(W^{(1)}_{i,i'}-\bar{W}^{(1)}_{i,i'})g_{i'}\right|\over \sum_{i'}\tilde{W}^{(1)}_{i,i'}}\right)\\
		&\le \sup_i{\left|\sum_{i'}W^{(1)}_{i,i'}-\sum_{i'}\tilde{W}^{(1)}_{i,i'}\right|\over \sum_{i'}\tilde{W}^{(1)}_{i,i'}}+\sup_{g:|g_i|\le 1}\sup_i\left({\left|\sum_{i'}(W^{(1)}_{i,i'}-\bar{W}^{(1)}_{i,i'})g_{i'}\right|\over \sum_{i'}\tilde{W}^{(1)}_{i,i'}}\right)\\
		&\le \sup_i {\left|\sum_{i'}W^{(1)}_{i,i'}-\sum_{i'}\tilde{W}^{(1)}_{i,i'}\right|\over \sum_{i'}\tilde{W}^{(1)}_{i,i'}}+\sup_{g:|g_i|\le 1}\sup_i\left({1\over \sum_{i'}\tilde{W}^{(1)}_{i,i'}}\left|\sum_{i'}\left({W_{i,i'}\over D_{i,i}D_{i',i'}}-{W_{i,i'}\over \tilde{D}_{i,i}\tilde{D}_{i',i'}}\right)g_{i'}\right|\right)
	\end{align*}
	We can apply the similar arguments in the proof of Lemma~\ref{lm:pointrate} to show that 
	$$
	\|P^{(1)}-\bar{P}^{(1)}\|_\infty\le {C\over t^d\sqrt{mn}}\left(\sqrt{-\log t}+\sqrt{\log m}\right)+{Ct^{2-d/2}\sqrt{\log m}\over \sqrt{n}}.
	$$
	
	For $\|(\tilde{P}^{(1)}-\bar{P}^{(1)})\bar{P}^{(1)}\|_\infty$, note that 
	\begin{align*}
		&\|(\tilde{P}^{(1)}-\bar{P}^{(1)})\bar{P}^{(1)}\|_\infty\\
		=&\sup_{g:|g_i|\le 1}\sup_i \left|[(\tilde{P}^{(1)}-\bar{P}^{(1)})\bar{P}^{(1)}g]_i\right|\\
		=&\sup_i \sup_{g:|g_i|\le 1}\left|{1\over \sum_{i'}\tilde{W}^{(1)}_{i,i'}}\sum_{i'}\tilde{W}^{(1)}_{i,i'}{\sum_{i''}\bar{W}^{(1)}_{i',i''}g_{i''}\over \sum_{i''}\tilde{W}^{(1)}_{i',i''}}-{1\over \sum_{i'}\tilde{W}^{(1)}_{i,i'}}\sum_{i'}\bar{W}^{(1)}_{i,i'}{\sum_{i''}\bar{W}^{(1)}_{i',i''}g_{i''}\over \sum_{i''}\tilde{W}^{(1)}_{i',i''}}\right|\\
		\le &\sup_i {1\over (\min_i \sum_{i'}\tilde{W}^{(1)}_{i,i'})^2} \sup_{g:|g_i|\le 1}\left(\min_i \sum_{i'}\tilde{W}^{(1)}_{i,i'}\right)\left|\sum_{i'}\tilde{W}^{(1)}_{i,i'}{\sum_{i''}\bar{W}^{(1)}_{i',i''}g_{i''}\over \sum_{i''}\tilde{W}^{(1)}_{i',i''}}-\sum_{i'}\bar{W}^{(1)}_{i,i'}{\sum_{i''}\bar{W}^{(1)}_{i',i''}g_{i''}\over \sum_{i''}\tilde{W}^{(1)}_{i',i''}}\right|\\
		\le &{1\over (\min_i \sum_{i'}\tilde{W}^{(1)}_{i,i'})^2} \sup_i \sup_{g:|g_i|\le 1}\left|\sum_{i'}\tilde{W}^{(1)}_{i,i'}\sum_{i''}\bar{W}^{(1)}_{i',i''}g_{i''}-\sum_{i'}\bar{W}^{(1)}_{i,i'}\sum_{i''}\bar{W}^{(1)}_{i',i''}g_{i''}\right|\\
		\le& Ct^{d}\sup_i \sup_{g:|g_i|\le 1}{\max_{i',i''}\bar{W}^{(1)}_{i',i''} \over m^2}\left|\sum_{i'}\tilde{W}^{(1)}_{i,i'}\sum_{i''}{\bar{W}^{(1)}_{i',i''}\over\max_{i',i''}\bar{W}^{(1)}_{i',i''}}g_{i''}-\sum_{i'}\bar{W}^{(1)}_{i,i'}\sum_{i''}{\bar{W}^{(1)}_{i',i''}\over \max_{i',i''}\bar{W}^{(1)}_{i',i''}}g_{i''}\right|\\
		\le& C\sup_i \sup_{g:|g_i|\le 1}\left|{1\over m}\sum_{i''}\left({1\over m}\sum_{i'}\tilde{W}^{(1)}_{i,i'}-\bar{W}^{(1)}_{i,i'}\right)g_{i''}\right|
	\end{align*}
	By the definition, Lemma~\ref{lm:empiricalprocess} and \ref{lm:operatorproperty}, we can show that 
	\begin{align*}
		\left|{1\over m}\sum_{i'}\tilde{W}^{(1)}_{i,i'}-\bar{W}^{(1)}_{i,i'}\right|&={1\over \tilde{D}_{i,i}\tilde{D}_{i',i'}}\left|{1\over m}\sum_{i'}(W_{i,i'}-\tilde{W}_{i,i'})\right|\\
		&\le {C\over t^d\sqrt{mn}}\left(\sqrt{-\log t}+\sqrt{\log m}\right)+{Ct^{2-d/2}\sqrt{\log m}\over \sqrt{n}}.
	\end{align*}
	Putting above terms together yields
	$$
	\|(\tilde{P}^{(1)}-\bar{P}^{(1)})\bar{P}^{(1)}\|_\infty\le {C\over t^d\sqrt{mn}}\left(\sqrt{-\log t}+\sqrt{\log m}\right)+{Ct^{2-d/2}\sqrt{\log m}\over \sqrt{n}}.
	$$
	We now complete the proof. 
\end{proof}

\begin{lemma}
	\label{lm:pointrate}
	Assume the assumptions in Lemma~\ref{lm:operatorproperty} hold and let $g=(g_1,\ldots,g_m)$ be a $m$-dimensional vector such that $|g_i|\le 1$ for $1\le i\le m$. Then, we have
	$$
	\|(\tilde{P}^{(1)}-P^{(1)})g\|_\infty\le {C\over t^d\sqrt{mn}}\left(\sqrt{-\log t}+\sqrt{\log m}\right)+{Ct^{2-d/2}\sqrt{\log m}\over \sqrt{n}}.
	$$
	with probability at least $1-m^{-2}$. Here, the constant $C$ depends on $d$, diameter of $\Mcal$, the curvature of $\Ncal_v$, $d_v$, the injectivity radius and reach of manifold $\Mcal$, $d_s$, the diameter of $\Ncal_s$, $\kappa$ and $C^0$ norm of $f_s$.
\end{lemma}
\begin{proof}
	By definition of $P^{(1)}$ and $\tilde{P}^{(1)}$, we have 
	\begin{align*}
		\|(\tilde{P}^{(1)}-P^{(1)})g\|_\infty&=\sup_i\left|[(P^{(1)}-\tilde{P}^{(1)})g]_i\right|\\
		&=\sup_i\left|{\sum_{i'}W^{(1)}_{i,i'}g_{i'}\over \sum_{i'}W^{(1)}_{i,i'}}-{\sum_{i'}\tilde{W}^{(1)}_{i,i'}g_{i'}\over \sum_{i'}\tilde{W}^{(1)}_{i,i'}}\right|\\
		&\le \sup_i\left|{\sum_{i'}W^{(1)}_{i,i'}g_{i'}\over \sum_{i'}W^{(1)}_{i,i'}}-{\sum_{i'}W^{(1)}_{i,i'}g_{i'}\over \sum_{i'}\tilde{W}^{(1)}_{i,i'}}\right|+\sup_i\left|{\sum_{i'}W^{(1)}_{i,i'}g_{i'}\over \sum_{i'}\tilde{W}^{(1)}_{i,i'}}-{\sum_{i'}\tilde{W}^{(1)}_{i,i'}g_{i'}\over \sum_{i'}\tilde{W}^{(1)}_{i,i'}}\right|\\
		&\le \sup_i\left|\sum_{i'}W^{(1)}_{i,i'}g_{i'}\right|{\left|\sum_{i'}W^{(1)}_{i,i'}-\sum_{i'}\tilde{W}^{(1)}_{i,i'}\right|\over \sum_{i'}W^{(1)}_{i,i'}\sum_{i'}\tilde{W}^{(1)}_{i,i'}}+\sup_i{\left|\sum_{i'}(W^{(1)}_{i,i'}-\tilde{W}^{(1)}_{i,i'})g_{i'}\right|\over \sum_{i'}\tilde{W}^{(1)}_{i,i'}}\\
		&\le Ct^{d/2}\sup_i \left(\left|{1\over m}\sum_{i'}(W^{(1)}_{i,i'}-\tilde{W}^{(1)}_{i,i'})\right|+\left|{1\over m}\sum_{i'}(W^{(1)}_{i,i'}-\tilde{W}^{(1)}_{i,i'})g_{i'}\right|\right)
	\end{align*}
	Hence, It is sufficient to bound the following terms
	$$
	\left|{1\over m}\sum_{i'}(W^{(1)}_{i,i'}-\tilde{W}^{(1)}_{i,i'})g_{i'}\right|\qquad {\rm and}\qquad \left|{1\over m}\sum_{i'}(W^{(1)}_{i,i'}-\tilde{W}^{(1)}_{i,i'})\right|.
	$$
	Note that 
	\begin{align*}
		\left|{1\over m}\sum_{i'}(W^{(1)}_{i,i'}-\tilde{W}^{(1)}_{i,i'})g_{i'}\right|&=\left|{1\over m}\sum_{i'}\left({W_{i,i'}\over D_{i,i}D_{i',i'}}-{\tilde{W}_{i,i'}\over \tilde{D}_{i,i}\tilde{D}_{i',i'}}\right)g_{i'}\right|\\
		&=\underbrace{\left|{1\over m}{\sum_{i'} (W_{i,i'}-\tilde{W}_{i,i'})g_{i'}\over \tilde{D}_{i,i}\tilde{D}_{i',i'}}\right|}_{A_i}+\underbrace{\left|{1\over m}\sum_{i'}\left({W_{i,i'}\over D_{i,i}D_{i',i'}}-{W_{i,i'}\over \tilde{D}_{i,i}\tilde{D}_{i',i'}}\right)g_{i'}\right|}_{B_i}
	\end{align*}
	We can bound $A_i$ and $B_i$ separately. For $A_i$, we can directly apply Lemma~\ref{lm:empiricalprocess} and Lemma~\ref{lm:operatorproperty} to yield
	$$
	A_i\le {C\over t^d\sqrt{mn}}\left(\sqrt{-\log t}+\sqrt{\log m}\right)+{Ct^{2-d/2}\sqrt{\log m}\over \sqrt{n}}.
	$$
	For $B_i$, we have 
	\begin{align*}
		\left|{W_{i,i'}\over D_{i,i}D_{i',i'}}-{W_{i,i'}\over \tilde{D}_{i,i}\tilde{D}_{i',i'}}\right|&\le {|\tilde{D}_{i,i}\tilde{D}_{i',i'}-D_{i,i}D_{i',i'}| \over D_{i,i}D_{i',i'}\tilde{D}_{i,i}\tilde{D}_{i',i'}}\\
		&\le {|\tilde{D}_{i',i'}-D_{i',i'}| \over D_{i,i}D_{i',i'}\tilde{D}_{i',i'}}+{|\tilde{D}_{i,i}-D_{i,i}| \over D_{i,i}\tilde{D}_{i,i}\tilde{D}_{i',i'}}\\
		&\le {C\over t^{3d/2}\sqrt{mn}}\left(\sqrt{-\log t}+\sqrt{\log m}\right)+{Ct^{2-d}\sqrt{\log m}\over \sqrt{n}}.
	\end{align*}
	Here, we apply Lemma~\ref{lm:empiricalprocess} and Lemma~\ref{lm:operatorproperty}. This suggests
	$$
	B_i\le {C\over t^{3d/2}\sqrt{mn}}\left(\sqrt{-\log t}+\sqrt{\log m}\right)+{Ct^{2-d}\sqrt{\log m}\over \sqrt{n}}.
	$$
	Putting all terms together, we have 
	$$
	\left|{1\over m}\sum_{i'}(W^{(1)}_{i,i'}-\tilde{W}^{(1)}_{i,i'})g_{i'}\right| \le {C\over t^{3d/2}\sqrt{mn}}\left(\sqrt{-\log t}+\sqrt{\log m}\right)+{Ct^{2-d}\sqrt{\log m}\over \sqrt{n}}.
	$$
	We can have a similar bound for $\left|\sum_{i'}(W^{(1)}_{i,i'}-\tilde{W}^{(1)}_{i,i'})/m\right|$. Then we can complete the proof. 
\end{proof}

\begin{lemma}
	\label{lm:operatorproperty}
	Suppose the assumptions in Lemma~\ref{lm:empiricalprocess} hold and $g=(g_1,\ldots,g_m)$ is a $m$-dimensional vector such that $|g_i|\le 1$ for $1\le i\le m$. If $\left(\sqrt{-\log t}+\sqrt{\log m}\right)/\sqrt{m}t^{d/2}<C_1$ and $t^2\sqrt{\log m}/\sqrt{n}<C_2$, with probability at least $1-m^{-2}$, we have
	\begin{enumerate}[label=(\alph*)]
		\item For any $1\le i\le m$, $ct^{d/2}\le \sum_{i'}W_{i,i'}/m\le Ct^{d/2}$
		\item For any $1\le i\le m$, $|\sum_{i'}W^{(1)}_{i,i'}g_{i'}/m|\le Ct^{-d/2}$
		\item For any $1\le i\le m$, $\sum_{i'}W^{(1)}_{i,i'}/m\ge ct^{-d/2}$
		\item $\|P^{(1)}\|_\infty\le C$ and $\|\bar{P}^{(1)}\|_\infty\le C$ 
	\end{enumerate}
	Here, the constant $C$ depends on $d$, diameter of $\Mcal$, the curvature of $\Ncal_v$, $d_v$, the injectivity radius and reach of manifold $\Mcal$, $d_s$, the diameter of $\Ncal_s$, $\kappa$ and $C^0$ norm of $f_s$.
\end{lemma}
\begin{proof}
	From Step 2 in proof of Theorem~\ref{thm:convergencerate}, we can know 
	$$
	ct^{d/2} \le {1\over m}\sum_{i'}\tilde{W}_{i,i'}\le Ct^{d/2}.
	$$
	By Lemma~\ref{lm:empiricalprocess}, we can know that 
	$$
	\left|{1\over m}\sum_{i'}(W_{i,i'}-\tilde{W}_{i,i'})\right|\le {C\over \sqrt{mn}}\left(\sqrt{-\log t}+\sqrt{\log m}\right)+{Ct^{2+d/2}\sqrt{\log m}\over \sqrt{n}}\le Ct^{d/2}.
	$$
	So we can prove (a). For (b), note that
	$$
	\left|{1\over m}\sum_{i'}W^{(1)}_{i,i'}g_{i'}\right|=\left|{1\over mD_{i,i}}\sum_{i'}{W_{i,i'}g_{i'}\over D_{i',i'}}\right|\le {1\over (\min_i D_{i,i})^2}\left|{1\over m}\sum_{i'} W_{i,i'}\right| \le Ct^{-d/2}.
	$$
	Similarly, for (c), we have 
	$$
	{1\over m}\sum_{i'}W^{(1)}_{i,i'}={1\over mD_{i,i}}\sum_{i'}{W_{i,i'}\over D_{i',i'}}\ge {1\over (\max_i D_{i,i})^2}{1\over m}\sum_{i'}W_{i,i'}\ge ct^{-d/2}.
	$$
	After we have (b) and (c), we can naturally obtain (d). 
\end{proof}

\begin{lemma}
	\label{lm:empiricalprocess}
	Recall the definition of $W_{i,i'}$ and $\tilde{W}_{i,i'}$ in the proof of Theorem~\ref{thm:convergencerate}. 
	For any fixed $g=(g_1,\ldots,g_m)$ such that $|g_i|\le 1$, we have
	$$
	\PP\left(\sup_{1\le i\le m}\left|{1\over m}\sum_{i'} (W_{i,i'}-\tilde{W}_{i,i'})g_{i'}\right|\ge {C\over \sqrt{mn}}\left(\sqrt{-\log t}+\sqrt{\log m}\right)+{Ct^{2+d/2}\sqrt{\log m}\over \sqrt{n}}\right)\le {1\over m^{2}}.
	$$
	Here, the constant $C$ relies on $d$, diameter of $\Mcal$, the curvature of $\Ncal_v$, $d_v$ and the injectivity radius and reach of manifold $\Mcal$.
\end{lemma}
\begin{proof}
	We first decompose $m^{-1}\sum_{i'} (W_{i,i'}-\tilde{W}_{i,i'})g_{i'}$ into two parts
	\begin{align*}
		&{1\over m}\sum_{i'} (W_{i,i'}-\tilde{W}_{i,i'})g_{i'}\\
		=&\underbrace{{1\over m}\sum_{i'} \left(W_{i,i'}-\EE(W_{i,i'}|\phi_{i'},X_{i,1},\ldots,X_{i,n})\right)g_{i'}}_{A_i}+\underbrace{{1\over m}\sum_{i'}\left(\EE(W_{i,i'}|\phi_{i'},X_{i,1},\ldots,X_{i,n})-\tilde{W}_{i,i'}\right)g_{i'}}_{B_i}.
	\end{align*}
	We will bound the above terms $A_i$ and $B_i$ separately. To bound $A_i$, we define 
	$$
	U(Y)={1\over m}\sum_{i'} \left({1\over n}\sum_{j'=1}^n\exp\left(-{\|Y-X_{i',j'}\|^2\over t}\right)-{1\over {\rm Vol}\Ncal_v}\int_{\Mcal(\phi_{i'})}\exp\left(-{\|Y-x\|^2\over t}\right)dx\right)g_{i'},
	$$
	where $Y\in \Mcal$. Clearly, $A_i=\sum_jU(X_{i,j})/n$. We apply chaining technique \citep{talagrand2014upper,wang2021structured} to establish a bound for $\sup_{Y\in \Mcal} |U(Y)|$. For any fixed $Y\in \Mcal$, an application of Hoeffding's inequality suggests 
	$$
	\PP\left(|U(Y)|>r\right)\le 2\exp(-2mnr^2)
	$$
	For any $Y_1,Y_2\in \Mcal$,
	$$
	\sup_x\left|\exp\left(-{\|Y_1-x\|^2\over t}\right)-\exp\left(-{\|Y_2-x\|^2\over t}\right)\right|\le  {2\|Y_1-Y_2\|\over t}.
	$$
	We can apply Hoeffding's inequality again to yield
	$$
	\PP\left(|U(Y_1)-U(Y_2)|>r\right)\le 2\exp(-2mnt^2r^2/\|Y_1-Y_2\|^2).
	$$
	In addition, the same argument in the proof of Lemma 20 in \cite{dunson2021spectral} shows that 
	$$
	N(\Mcal,\gamma,\|\cdot\|)\le \left({8{\rm diam}(\Mcal)\over \gamma}\right)^{d(3d+11)/2},
	$$
	where $N(\Mcal,\gamma,\|\cdot\|)$ is the covering number of $\Mcal$ in the norm $\|\cdot\|$.
	Since $\EE(U(Y))=0$, a standard chaining argument \citep[see, e.g., Theorem2.2.27 in ][]{talagrand2014upper} suggests
	$$
	\PP\left(\sup_{Y\in \Mcal} |U(Y)|>{C\over \sqrt{mn}}\left(\sqrt{-\log t}+r\right)\right)\le C\exp(-r^2),
	$$
	where $C$ replies on $d$ and diameter of $\Mcal$.
	Hence, we can know that 
	$$
	\PP\left(|A_i|>{C\over \sqrt{mn}}\left(\sqrt{-\log t}+\sqrt{\log m}\right)\right)\le {1\over 2m^3}.
	$$
	
	We then work on $B_i$. Given $Y\in \Mcal$, define 
	$$
	V(Y)={1\over m}\sum_{i'=1}^m\left( {1\over {\rm Vol}\Ncal_v}\int_{\Mcal(\phi_{i'})}\exp\left(-{\|x-Y\|^2\over t}\right)dx-\tilde{W}_{i,i'}\right)g_{i'}.
	$$
	It is clear that $B_i=\sum_{j=1}^n V(X_{i,j})/n$ and $\EE(V(X_{i,j}))=0$. Lemma~\ref{lm:meanbound} suggests 
	$$
	\left|{1\over {\rm Vol}\Ncal_v}\int_{\Mcal(\phi_{i'})}\exp\left(-{\|x-Y\|^2\over t}\right)dx-\tilde{W}_{i,i'}\right|<C\tilde{W}_{i,i'}t^2+\exp\left(-{r^2\over t}\right).
	$$
	So, we have
	$$
	|V(Y)|\le {Ct^2\over m}\sum_{i'=1}^m \tilde{W}_{i,i'}+\exp\left(-{r^2\over t}\right)\le Ct^{2+d/2}+\exp\left(-{r^2\over t}\right)\le Ct^{2+d/2},
	$$
	where we use $\sum_{i'}\tilde{W}_{i,i'}/m\le Ct^{d/2}$. 
	An application of Hoeffding's inequality yields
	$$
	\PP\left(|B_i|>r\right)\le 2\exp\left(2r^2\over nCt^{4+d}\right).
	$$
	Putting results of $A_i$ and $B_i$ together yields
	$$
	\PP\left(|A_i+B_i|>{C\over \sqrt{mn}}\left(\sqrt{-\log t}+\sqrt{\log m}\right)+{Ct^{2+d/2}\sqrt{\log m}\over \sqrt{n}}\right)\le {1\over m^3}
	$$
	We can complete the proof by applying union bound. 
\end{proof}

\begin{lemma}
	\label{lm:meanbound}
	Recall the definition
	$$
	\tilde{h}_t(\phi_1,\phi_2)={1\over {\rm Vol}^2\Ncal_v}\int_{\Mcal(\phi_{1})}\int_{\Mcal(\phi_{2})}\exp\left(-{\|x-y\|^2\over t}\right)dxdy.
	$$
	There exists a constant $c$ such that when $t<c$, we have
	$$
	\left|{1\over {\rm Vol}\Ncal_v}\int_{\Mcal(\phi_{1})}\exp\left(-{\|x-y\|^2\over t}\right)dx-\tilde{h}_t(\phi_1,\phi_2)\right|<C\tilde{h}_t(\phi_1,\phi_2)t^2+\exp\left(-{r^2\over t}\right), \quad y\in \Mcal(\phi_2).
	$$
	where $r$ is some small enough constant that is smaller than half of the reach of manifold $\Mcal$ and the constant $C$ relies on curvature of $\Ncal_v$, $d_v$ and the injectivity radius and reach of manifold $\Mcal$.
\end{lemma}
\begin{proof}
	There exists a small ball $\Bcal=\{x:\|x-y\|\le r\}$ such that the exponential map is diffeomorphism within this small ball. Here, $r$ is the injectivity radius of $\Mcal$. Then, we have
	$$
	\int_{\Mcal(\phi_{1})}\exp\left(-{\|x-y\|^2\over t}\right)dx=\int_{\Mcal(\phi_{1})\cap\Bcal}\exp\left(-{\|x-y\|^2\over t}\right)dx+\int_{\Mcal(\phi_{1})\cap\Bcal^c}\exp\left(-{\|x-y\|^2\over t}\right)dx
	$$
	By the definition of $\Bcal$, we have
	$$
	{1\over {\rm Vol}\Ncal_v}\int_{\Mcal(\phi_{1})\cap\Bcal^c}\exp\left(-{\|x-y\|^2\over t}\right)dx\le \exp\left(-{r^2\over t}\right).
	$$
	Within $\Bcal$, Proposition 2 in \cite{garcia2020error} suggests that $d_\Mcal(x,y)\le r+8r^3/R^2$, where $R$ is the reach of manifold $\Mcal$. If $d_{\Ncal_s}(\phi_1,\phi_2)\ge r+8r^3/R^2$, we have $\Mcal(\phi_{1})\cap\Bcal=\emptyset$ and
	$$
	{1\over {\rm Vol}\Ncal_v}\int_{\Mcal(\phi_{1})\cap\Bcal}\exp\left(-{\|x-y\|^2\over t}\right)dx=0.
	$$
	Therefore, if $d_{\Ncal_s}(\phi_1,\phi_2)\ge r+8r^3/R^2$, 
	$$
	{1\over {\rm Vol}\Ncal_v}\int_{\Mcal(\phi_{1})}\exp\left(-{\|x-y\|^2\over t}\right)dx\le \exp\left(-{r^2\over t}\right).
	$$
	and 
	$$
	\left|{1\over {\rm Vol}\Ncal_v}\int_{\Mcal(\phi_{1})}\exp\left(-{\|x-y\|^2\over t}\right)dx-\tilde{h}_t(\phi_1,\phi_2)\right|\le \exp\left(-{r^2\over t}\right).
	$$
	
	Other other hand, if $d_{\Ncal_s}(\phi_1,\phi_2)<r+8r^3/R^2$, we can have 
	\begin{align*}
		&{1\over {\rm Vol}\Ncal_v}\int_{\Mcal(\phi_{1})\cap \Bcal}\exp\left(-{\|x-y\|^2\over t}\right)dx\\
		\le & {1\over {\rm Vol}\Ncal_v}\int_{\Mcal(\phi_{1})\cap \Bcal}\exp\left(-{d^2_\Mcal(x,y)(1-8d^2_\Mcal(x,y)/R^2)^2\over t}\right)dx\\
		\le & {1\over {\rm Vol}\Ncal_v}\int_{\Mcal(\phi_{1})\cap \Bcal}\exp\left(-{d^2_\Mcal(x,y)(1-16d^2_\Mcal(x,y)/R^2)\over t}\right)dx\\
		\le & {1\over {\rm Vol}\Ncal_v}\int_{\Mcal(\phi_{1})\cap \Bcal}\exp\left(-{d^2_\Mcal(x,y)\over t}\right)\left(1+{16d^4_\Mcal(x,y)\over tR^2}\exp\left(-{16d^4_\Mcal(x,y)\over tR^2}\right)\right)dx\\
		\le & {1\over {\rm Vol}\Ncal_v}\exp\left(-{d^2_{\Ncal_s}(\phi_1,\phi_2)\over t}\right)\int_{\Ncal_v} \exp\left(-{d^2_{\Ncal_v}(\psi_1,\psi_2)\over t}\right)\left(1+{16d^4_{\Ncal_v}(\psi_1,\psi_2)\over tR^2}\exp\left(-{16d^4_{\Ncal_v}(\psi_1,\psi_2)\over tR^2}\right)\right)d\psi_1\\
		\le & \exp\left(-{d^2_{\Ncal_s}(\phi_1,\phi_2)\over t}\right){t^{-d_v/2}\over {\rm Vol}\Ncal_v}\left(\tilde{m}_0+O(t^2)\right),
	\end{align*}
	where $\tilde{m}_0=\int_{\RR^{d_v}}\exp\left(\|u\|^2\right)du$.	Here, we use Lemma~\ref{lm:convolution}, $d^2_\Mcal(x,y)=d^2_{\Ncal_s}(\phi_1,\phi_2)+d^2_{\Ncal_v}(\psi_1,\psi_2)$, and 
	$$
	d_\Mcal(x,y)-{8\over R^2}d^3_\Mcal(x,y)\le \|x-y\|\le d_\Mcal(x,y).
	$$
	If we write $x=T(\phi_1,\psi_1)$ and $y=T(\phi_2,\psi_2)$, we have
	\begin{align*}
		&{1\over {\rm Vol}\Ncal_v}\int_{\Mcal(\phi_{1})}\exp\left(-{\|x-y\|^2\over t}\right)dx\\
		\ge & {1\over {\rm Vol}\Ncal_v}\int_{\Mcal(\phi_{1})}\exp\left(-{d^2_\Mcal(x,y)\over t}\right)dx\\
		=& {1\over {\rm Vol}\Ncal_v}\int_{\Mcal(\phi_{1})}\exp\left(-{d^2_{\Ncal_s}(\phi_1,\phi_2)+d^2_{\Ncal_v}(\psi_1,\psi_2)\over t}\right)dx\\
		=& \exp\left(-{d^2_{\Ncal_s}(\phi_1,\phi_2)\over t}\right){1\over {\rm Vol}\Ncal_v}\int_{\Ncal_v}\exp\left(-{d^2_{\Ncal_v}(\psi_1,\psi_2)\over t}\right)d\psi_1\\
		=& \exp\left(-{d^2_{\Ncal_s}(\phi_1,\phi_2)\over t}\right){t^{-d_v/2}\over {\rm Vol}\Ncal_v}\left(\tilde{m}_0+O(t^2)\right).
	\end{align*}
	Putting upper and lower bounds together suggests 
	$$
	\left|{1\over {\rm Vol}\Ncal_v}\int_{\Mcal(\phi_{1})}\exp\left(-{\|x-y\|^2\over t}\right)dx-\tilde{h}_t(\phi_1,\phi_2)\right|<C\tilde{h}_t(\phi_1,\phi_2)t^2+\exp\left(-{r^2\over t}\right).
	$$
\end{proof}

\begin{lemma}
	\label{lm:gcclass}
	Let $g(\phi)$ be a continuous function defined on $\Ncal_s$ such that $\|g\|_\infty\le 1$. If $t$ is a sufficiently small constant, then there is a constant $C$ such that with probability at least $1-m^{-2}$, we have 
	$$
	\sup_{\phi'\in \Ncal_s}\left|{1\over m}\sum_{i=1}^m\tilde{h}_t(\phi_i,\phi')-\int_{\Ncal_s}\tilde{h}_t(\phi,\phi')f_s(\phi)d\phi\right|\le {C\over \sqrt{m}}\left(\sqrt{-\log t}+\sqrt{\log m}\right),
	$$
	$$
	\sup_{\phi'\in \Ncal_s}\left|{1\over m}\sum_{i=1}^mg(\phi_i)\tilde{h}^{(1)}_t(\phi_i,\phi')-\int_{\Ncal_s}g(\phi)\tilde{h}^{(1)}_t(\phi,\phi')f_s(\phi)d\phi\right|\le {C\over \sqrt{m}t^{d}}\left(\sqrt{-\log t}+\sqrt{\log m}\right),
	$$
	\begin{align*}
		&\sup_{\phi',\phi''\in \Ncal_s}\left|{1\over m}\sum_{i=1}^m\tilde{h}^{(1)}_t(\phi_i,\phi'')\tilde{h}^{(1)}_t(\phi_i,\phi')-\int_{\Ncal_s}\tilde{h}^{(1)}_t(\phi,\phi'')\tilde{h}^{(1)}_t(\phi,\phi')f_s(\phi)d\phi\right|\\
		&\le {C\over \sqrt{m}t^{2d}}\left(\sqrt{-\log t}+\sqrt{\log m}\right).
	\end{align*}
	Here, the constant $C$ depends on $d_s$, the diameter of $\Ncal_s$, $\kappa$ and $C^0$ norm of $f_s$.
\end{lemma}
\begin{proof}
	If we can characterize the covering number of $\Fcal_1=\{\tilde{h}_t(\cdot,\phi'):\phi'\in \Ncal_s\}$, $\Fcal_2=\{t^dg(\cdot)\tilde{h}_t(\cdot,\phi'):\phi'\in \Ncal_s\}$, and $\Fcal_3=\{t^{2d}\tilde{h}_t(\cdot,\phi')\tilde{h}_t(\cdot,\phi''):\phi',\phi''\in \Ncal_s\}$ in the norm $\|\cdot\|_\infty$, we can follow the same strategy in Proposition 2 in \cite{dunson2021spectral} to show the above statements. The rest of proof is to bound the covering number of $\Fcal_1$, $\Fcal_2$, and $\Fcal_3$. Recall the covering number of $\Fcal$, $N(\Fcal,\gamma,\|\cdot\|_\infty)$, is defined as the smallest number of balls with radius $\gamma$ in the norm $\|\cdot\|_\infty$ that can cover $\Fcal$. Following the same argument in the proof of Lemma 20 in \cite{dunson2021spectral}, we can show that 
	$$
	N(\Ncal_s,\gamma,d_{\Ncal_s}(\cdot,\cdot))\le \left({8{\rm diam}(\Ncal_s)\over \gamma}\right)^{d_s(3d_s+11)/2}.
	$$
	Note that 
	\begin{align*}
		&\max_{\phi\in \Ncal_s}\left|\tilde{h}_t(\phi,\phi')-\tilde{h}_t(\phi,\phi'')\right|\\
		=&{1\over {\rm Vol}^2\Ncal_v}\max_{\phi\in \Ncal_s}\left|\int_{\Mcal(\phi)}\int_{\Mcal(\phi')}\exp\left(\|x-y\|^2\over t\right)dxdy-\int_{\Mcal(\phi)}\int_{\Mcal(\phi'')}\exp\left(\|x-y\|^2\over t\right)dxdy\right|\\
		=&{1\over {\rm Vol}^2\Ncal_v}\max_{\phi\in \Ncal_s}\left|\int_{\Mcal(\phi)}\int_{\Ncal_v}\exp\left(\|T(\phi',\psi)-y\|^2\over t\right)-\exp\left(\|T(\phi'',\psi)-y\|^2\over t\right)d\psi dy\right|\\
		\le &{1\over {\rm Vol}^2\Ncal_v}\max_{\phi\in \Ncal_s}\int_{\Mcal(\phi)}\int_{\Ncal_v} {d_{\Ncal_s}(\phi',\phi'')\over 2t} d\psi dy\\
		\le &{d_{\Ncal_s}(\phi',\phi'')\over 2t}
	\end{align*}
	Here we use the mean value theorem. Therefore, we can know that 
	$$
	N(\Fcal_1,\gamma,\|\cdot\|_\infty)\le \left({4{\rm diam}(\Ncal_s)\over t\gamma}\right)^{d_s(3d_s+11)/2}.
	$$
	When $\|g\|_\infty\le 1$, we have 
	\begin{align*}
		&\max_{\phi\in \Ncal_s}\left|g(\phi)\tilde{h}^{(1)}_t(\phi,\phi')-g(\phi)\tilde{h}^{(1)}_t(\phi,\phi'')\right|\\
		=& \max_{\phi\in \Ncal_s}\left|g(\phi){\tilde{h}_t(\phi,\phi') \over f_{s,t}(\phi)f_{s,t}(\phi')}-g(\phi){\tilde{h}_t(\phi,\phi'') \over f_{s,t}(\phi)f_{s,t}(\phi'')}\right|\\
		=& \max_{\phi\in \Ncal_s}\left|{g(\phi)\over f_{s,t}(\phi)}\right|\left|{\tilde{h}_t(\phi,\phi')f_{s,t}(\phi'')-\tilde{h}_t(\phi,\phi'')f_{s,t}(\phi') \over f_{s,t}(\phi')f_{s,t}(\phi'')}\right|\\
		=& \max_{\phi\in \Ncal_s}\left|{g(\phi)\over f_{s,t}(\phi)}\right|\left|{\tilde{h}_t(\phi,\phi')f_{s,t}(\phi'')-\tilde{h}_t(\phi,\phi'')f_{s,t}(\phi'')+\tilde{h}_t(\phi,\phi'')f_{s,t}(\phi'')-\tilde{h}_t(\phi,\phi'')f_{s,t}(\phi') \over f_{s,t}(\phi')f_{s,t}(\phi'')}\right|\\
		\le & \max_{\phi\in \Ncal_s}\left|{g(\phi)\over f_{s,t}(\phi)} \right|\left(\left|{\tilde{h}_t(\phi,\phi')-\tilde{h}_t(\phi,\phi'') \over f_{s,t}(\phi')}\right|+\left|{\tilde{h}_t(\phi,\phi'')(f_{s,t}(\phi'')-f_{s,t}(\phi')) \over f_{s,t}(\phi')f_{s,t}(\phi'')}\right|\right).
	\end{align*}
	The Lemma~\ref{lm:operatorapp} (a) suggests there are constants $C_1$ and $C_2$ such that $C_1t^{d/2}\le f_{s,t}(\phi)\le C_2t^{d/2}$. In addition, we have 
	$$
	\left|f_{s,t}(\phi'')-f_{s,t}(\phi')\right|=\left|\int \tilde{h}_t(\phi',\phi) f_s(\phi)d\phi-\int \tilde{h}_t(\phi'',\phi) f_s(\phi)d\phi\right|\le \max_{\phi\in \Ncal_s}\left|\tilde{h}_t(\phi,\phi')-\tilde{h}_t(\phi,\phi'')\right|.
	$$
	Putting above together yields
	$$
	\max_{\phi\in \Ncal_s}\left|g(\phi)\tilde{h}^{(1)}_t(\phi,\phi')-g(\phi)\tilde{h}^{(1)}_t(\phi,\phi'')\right|\le {C\over t^{3d/2}} \max_{\phi\in \Ncal_s}\left|\tilde{h}_t(\phi,\phi')-\tilde{h}_t(\phi,\phi'')\right|\le {Cd_{\Ncal_s}(\phi',\phi'')\over t^{(3d+2)/2}}.
	$$
	So we have 
	$$
	N(\Fcal_2,\gamma,\|\cdot\|_\infty)\le \left({C{\rm diam}(\Ncal_s)\over t^{(d+2)/2}\gamma}\right)^{d_s(3d_s+11)/2}.
	$$
	Similarly, we can show 
	$$
	N(\Fcal_3,\gamma,\|\cdot\|_\infty)\le \left({C{\rm diam}(\Ncal_s)\over t^{(d+2)/2}\gamma}\right)^{d_s(3d_s+11)/2}.
	$$
	We now complete the proof. 
\end{proof}

\begin{lemma}
	\label{lm:operatorapp}
	If we adopt the same notation in Lemma~\ref{lm:convolution}, then we have
	\begin{enumerate}[label=(\alph*)]
		\item $f_{s,t}(\phi)=m_0f_s(\phi)t^{d/2}+O(t^{(d+2)/2})$ and $\tilde{h}_t^{(1)}(\phi_1,\phi_2)=(1+O(t))\tilde{h}_t(\phi_1,\phi_2)/m_0^2f_s(\phi_1)f_s(\phi_2)t^d$.
		\item $\Pcal_{t,1}$ is a self-adjoint operator.
		\item For any function $g(\phi)$, $\Pcal_{t,1}g(\phi)$ is smooth function on $\Ncal_s$ and 
		$$
		\Pcal_{t,1}g(\phi_1)={\int \tilde{h}_t(\phi_1,\phi_2)g(\phi_2)d\phi_2 \over \int \tilde{h}_t(\phi_1,\phi_2)d\phi_2}+O(t).
		$$
		\item For any function $g(\phi)$,
		$$
		{g(\phi)-\Pcal_{t,1}g(\phi)\over t}=\Lcal_{\Ncal_s}g(\phi)+O(t).
		$$
	\end{enumerate}
\end{lemma}
\begin{proof}
	We omit the proof here since we can apply the same arguments in \cite{coifman2006diffusion} when we have Lemma~\ref{lm:convolution}.
\end{proof}

\begin{lemma}
	\label{lm:bounded}
	Let $\Rcal_t$ be the operator defined in proof for Theorem~\ref{thm:convergencerate}. For any function $g(\phi)$, there exists a constant $C$ such that when $t$ is small enough, then 
	$$
	\|\Rcal_t g\|\le C\|g\|.
	$$
	Here, the constant $C$ relies on curvature of $\Ncal_s$, $d_s$, the reach of manifold $\Ncal_s$ and the Ricci curvature.
\end{lemma}
\begin{proof}
	Recall $\Rcal_t=(\Pcal_{t,1}-\Hcal_t)/t$. By Lemma~\ref{lm:operatorapp} (c), we have 
	$$
	\Pcal_{t,1}g(\phi_1)={\int \tilde{h}_t(\phi_1,\phi_2)g(\phi_2)d\phi_2 \over \int \tilde{h}_t(\phi_1,\phi_2)d\phi_2}+O(t).
	$$
	The $O(t)$ comes from the residue term in Lemma~\ref{lm:convolution}. If we dig into the proof details in Lemma~\ref{lm:convolution}, we can find that $O(t^2)$ can also be written as $O(t^2g(\phi))$. Therefore, if we apply Lemma~\ref{lm:convolution}, we can have 
	$$
	\left\|\Pcal_{t,1}g(\phi)-\left(g-t\Lcal_{\Ncal_s}(g)\right)\right\|\le O(t^2\|g\|).
	$$
	By Lemma 3 (a) in \cite{dunson2021spectral}, we can have
	$$
	\left\|\Hcal_tg(\phi)-\int (4\pi t)^{-d_s/2}e^{-{d_{\Ncal_s}(\phi,\phi')^2\over 4t}}g(\phi')d\phi'\right\|\le O(t\|g\|)
	$$
	By the similar arguments in Lemma~\ref{lm:convolution}, we can know 
	$$
	\int (4\pi t)^{-d_s/2}e^{-{d_{\Ncal_s}(\phi,\phi')^2\over 4t}}g(\phi')d\phi'=g-t\Lcal_{\Ncal_s}(g)+O(t^2g(\phi)).
	$$
	Putting all above terms together leads to 
	$$
	\|\Rcal_t g\|\le (C+O(t))\|g\|.
	$$
\end{proof}

\begin{lemma}
	\label{lm:residue}
	Let $\eta_l$ be the $l$th eigenfunction of $\Lcal_{\Ncal_s}$ such that $\|\eta_l\|=1$. There exists a constant $C$ such that when $t$ is small enough,
	$$
	\|\Rcal_t \eta_l\|\le Ct(1+\lambda_l^{d_s/2+5}).
	$$
	Here, the constant $C$ relies on $\kappa$, the $C^2$ norm of $f_s$, and the Ricci curvature.
\end{lemma}

\begin{proof}
	Recall $\Rcal_t=(\Pcal_{t,1}-\Hcal_t)/t$. By Lemma~\ref{lm:operatorapp} (d), we have 
	$$
	\Pcal_{t,1}\eta_l=\eta_l-t\lambda_l\eta_l+O(t^2)\eta_l.
	$$
	By the definition of heat kernel, we have 
	$$
	\Hcal_{t}\eta_l=e^{-\lambda_lt}\eta_l.
	$$
	Therefore, we have
	$$
	\|\Rcal_t \eta_l\|\le {(1-t\lambda_l-e^{-\lambda_lt})\over t}\eta_l+O(t)\eta_l \le Ct(1+\lambda_l^2)\eta_l.
	$$
	We can apply a similar strategy in Lemma 8 of \cite{dunson2021spectral} to show 
	$$
	\|\Rcal_t \eta_l\|\le Ct(1+\lambda_l^{d_s/2+5}).
	$$
\end{proof}

\begin{lemma}
	\label{lm:convolution}
	If we write 
	$$
	m_0=\int_{\RR^d}\exp\left(\|u\|^2\right)du\qquad {\rm and}\qquad m_2=\int_{\RR^d}u_1^2\exp\left(\|u\|^2\right)du,
	$$
	then we have
	$$
	G_t g(\phi)=m_0g(\phi)-t{m_2\over 2}\Lcal_{\Ncal_s}g(\phi)+O(t^2).
	$$
\end{lemma}
\begin{proof}
	We can rewrite $G_t g(\phi)$ as 
	\begin{align*}
		G_t g(\phi)&={1\over t^{d/2}{\rm Vol}^2\Ncal_v}\int_{\Ncal_s}\int_{\Mcal(\phi)}\int_{\Mcal(\phi')}\exp\left(-{\|x-y\|^2\over t}\right)g(\phi')dxdyd\phi'\\
		&={1\over t^{d/2}{\rm Vol}\Ncal_v}\int_{\Mcal(\phi)} \int_{\Mcal} \exp\left(-{\|x-y\|^2\over t}\right)\left(\int_{\Ncal_s}\bI(y\in \Mcal(\phi')) g(\phi')d\phi'\right)dydx.
	\end{align*}
	If we define
	$$
	\tilde{g}(y)=\int_{\Ncal_s}\bI(y\in \Mcal(\phi')) g(\phi')d\phi',
	$$
	then 
	$$
	G_t g(\phi)={1\over t^{d/2}{\rm Vol}\Ncal_v}\int_{\Mcal(\phi)} \int_{\Mcal} \exp\left(-{\|x-y\|^2\over t}\right)\tilde{g}(y)dydx.
	$$
	Different from $g(\phi)$, $\tilde{g}(x)$ is defined on $\Mcal$. It is clear that $\tilde{g}(x)=\tilde{g}(y)=g(\phi)$ if $x,y\in\Mcal(\phi)$. To evaluate $G_t g(\phi)$, it is sufficient to work on 
	$$
	F_t\tilde{g}(x)={1\over t^{d/2}} \int_{\Mcal} \exp\left(-{\|x-y\|^2\over t}\right)\tilde{g}(y)dy.
	$$
	We can first reduce the integral on a small ball $\{y:\|x-y\|\le r\}$ such that the exponential map is diffeomorphism within this small ball. Outside of the ball, 
	$$
	{1\over t^{d/2}} \int_{y\in \Mcal: \|x-y\|>r} \exp\left(-{\|x-y\|^2\over t}\right)\tilde{g}(y)dy\le C\sup_{y\in \Mcal} |\tilde{g}(y)| \exp(-r^2/t)=O(t^2).
	$$
	Hence, we have 
	$$
	F_t\tilde{g}(x)={1\over t^{d/2}} \int_{y\in \Mcal: \|x-y\|\le r} \exp\left(-{\|x-y\|^2\over t}\right)\tilde{g}(y)dy+O(t^2). 
	$$
	At $x\in\Mcal$, we write $T_x\Mcal$ as the tangent space of $\Mcal$ at $x$. Since $\Mcal=T(\Ncal_s\times\Ncal_v)$, we can decompose the tangent space into two subspaces $T_x\Mcal=V_x\Mcal \oplus H_x\Mcal$, where $V_x\Mcal$ is vertical space (corresponds to $\Ncal_v$) and $H_x\Mcal$ is horizontal space (corresponds to $\Ncal_s$). Let $e_1,\ldots, e_{d_s},e_{d_s+1},\ldots, e_{d}$ be a fixed orthonormal basis of $T_x\Mcal$ such that $e_1,\ldots, e_{d_s}$ corresponds to the directions of $\Ncal_s$ and $e_{d_s+1},\ldots, e_{d}$ corresponds to the directions of $\Ncal_v$. By the exponential map $\exp_x(s)$, any point $y$ in a small neighborhood of $x$ has a set of normal coordinates $s=(s_1,\ldots,s_d)$. Then, the function $\tilde{g}(y)$ can be written as $\tilde{g}^\ast(s_1,\ldots,s_{d_s})$ since $\tilde{g}(x)=\tilde{g}(y)=g(\phi)$ if $x,y\in\Mcal(\phi)$. We can change variable as in \cite{belkin2008towards}, 
	$$
	F_t\tilde{g}(x)={1\over t^{d/2}} \int_{\tilde{B}} \exp\left(-{\|x-\exp_x(s)\|^2\over t}\right)\tilde{g}^\ast(s)(1+O(\|s\|^2))ds+O(t^2),
	$$
	where $\tilde{B}$ is the preimage of exponential map for $\{y\in\Mcal:\|x-y\|\le r\}$. By Lemma 4.3 in \cite{belkin2008towards}, we have 
	$$
	0\le \|s\|^2-\|\exp_x(s)-x\|^2=w(s)\le C\|s\|^4.
	$$
	and 
	$$
	\exp\left(-{\|\exp_x(s)-x\|^2\over t}\right)=\exp\left(-{\|s\|^2-w(s)\over t}\right)=\exp\left(-{\|s\|^2\over t}\right)\left(1+O\left({w(s)\over t}e^{w(s)/t}\right)\right).
	$$
	Therefore, we have 
	\begin{align*}
		F_t\tilde{g}(x)&={1\over t^{d/2}} \int_{\tilde{B}} \exp\left(-{\|s\|^2\over t}\right)\left(1+O\left({w(s)\over t}e^{w(s)/t}\right)\right)\tilde{g}^\ast(s)(1+O(\|s\|^2))ds+O(t^2)\\
		&=A_t+B_t+C_t+O(t^2),
	\end{align*}
	where 
	$$
	A_t={1\over t^{d/2}} \int_{\tilde{B}} \exp\left(-{\|s\|^2\over t}\right)\tilde{g}^\ast(s)ds,
	$$
	$$
	B_t={1\over t^{d/2}} \int_{\tilde{B}} \exp\left(-{\|s\|^2\over t}\right)\tilde{g}^\ast(s)O(\|s\|^2)ds,
	$$
	and 
	$$
	C_t={1\over t^{d/2}} \int_{\tilde{B}} \exp\left(-{\|s\|^2\over t}\right)O\left({w(s)\over t}e^{w(s)/t}\right)\tilde{g}^\ast(s)(1+O(\|s\|^2))ds.
	$$
	We can apply the similar arguments in \cite{belkin2008towards} to show that $B_t=O(t^2)$ and $C_t=O(t^2)$. We now work on $A_t$. The Taylor expansion for $\tilde{g}^\ast$ suggests 
	$$
	\tilde{g}^\ast(s_1,\ldots,s_{d_s})=\tilde{g}^\ast(0)+\sum_{k=1}^{d_s}s_k{\partial \tilde{g}^\ast\over \partial s_k}(0)+{1\over 2}\sum_{k_1=1}^{d_s}\sum_{k_2=1}^{d_s}s_{k_1}s_{k_2}{\partial^2 \tilde{g}^\ast\over \partial s_{k_1}\partial s_{k_2}}(0)+O\left(\|s\|^3\right).
	$$
	Therefore, we have
	\begin{align*}
		A_t=&{1\over t^{d/2}} \int_{\tilde{B}} \exp\left(-{\|s\|^2\over t}\right)\tilde{g}^\ast(0)ds + \sum_{k=1}^{d_s}{1\over t^{d/2}} \int_{\tilde{B}} \exp\left(-{\|s\|^2\over t}\right)s_k{\partial \tilde{g}^\ast\over \partial s_k}(0)ds\\
		&+{1\over 2}\sum_{k_1=1}^{d_s}\sum_{k_2=1}^{d_s}{1\over t^{d/2}} \int_{\tilde{B}} \exp\left(-{\|s\|^2\over t}\right)s_{k_1}s_{k_2}{\partial^2 \tilde{g}^\ast\over \partial s_{k_1}\partial s_{k_2}}(0)ds+{1\over t^{d/2}} \int_{\tilde{B}} \exp\left(-{\|s\|^2\over t}\right)O\left(\|s\|^3\right)ds\\
		=&m_0\tilde{g}^\ast(0)+{t\over 2}m_2\sum_{k=1}^{d_s}{\partial^2 \tilde{g}^\ast\over \partial s_{k}^2}(0)+O(t^2)
	\end{align*}
	Putting $A_t$, $B_t$, and $C_t$ together leads to 
	$$
	F_t\tilde{g}(x)=m_0\tilde{g}^\ast(0)+{t\over 2}m_2\sum_{k=1}^{d_s}{\partial^2 \tilde{g}^\ast\over \partial s_{k}^2}(0)+O(t^2).
	$$
	Since $\tilde{g}^\ast(s)=\tilde{g}(y)=g(\phi_y)$ and $\Mcal=T(\Ncal_s\times\Ncal_v)$, we can know that
	$$
	\sum_{k=1}^{d_s}{\partial^2 \tilde{g}^\ast\over \partial s_{k}^2}(0)=-\Lcal_{\Ncal_s}g(\phi_x)\qquad{\rm and}\qquad  \tilde{g}^\ast(0)=g(\phi_x),
	$$
	where $x\in \Mcal(\phi_x)$. Plugging back to $G_t g(\phi)$ yields 
	$$
	G_t g(\phi)=m_0g(\phi)-{t\over 2}m_2\Lcal_{\Ncal_s}g(\phi)+O(t^2).
	$$
\end{proof}

\end{document}